\documentclass{article}

\PassOptionsToPackage{numbers, compress}{natbib}

\usepackage{subcaption}
\usepackage{multirow}
\usepackage{tabularx}
\usepackage[preprint]{neurips_2026}

\usepackage{amsmath,amsfonts,bm}

\def\eqref#1{equation~\ref{#1}}
\def\Eqref#1{Equation~\ref{#1}}

\def\1{\bm{1}}

\DeclareMathAlphabet{\mathsfit}{\encodingdefault}{\sfdefault}{m}{sl}
\SetMathAlphabet{\mathsfit}{bold}{\encodingdefault}{\sfdefault}{bx}{n}

\usepackage{amsmath,amsfonts,bm}

\def\eqref#1{equation~\ref{#1}}
\def\Eqref#1{Equation~\ref{#1}}

\def\1{\bm{1}}

\DeclareMathAlphabet{\mathsfit}{\encodingdefault}{\sfdefault}{m}{sl}
\SetMathAlphabet{\mathsfit}{bold}{\encodingdefault}{\sfdefault}{bx}{n}

\newcommand{\bfA}{{\bf A}}

\newcommand{\bfF}{{\bf F}}

\newcommand{\bfK}{{\bf K}}

\newcommand{\bfR}{{\bf R}}
\newcommand{\bfS}{{\bf S}}

\newcommand{\bfb}{{\bf b}}

\newcommand{\bfk}{{\bf k}}

\newcommand{\bfx}{{\bf x}}

\newcommand{\bfu}{{\bf u}}

\newcommand{\bfv}{{\bf v}}

\newcommand{\grad}{{\boldsymbol \nabla}}

\newcommand{\bftheta}{{\boldsymbol \theta}}

\newcommand{\bfxi}{{\boldsymbol \xi}}

\newcommand{\bfgamma}{{\boldsymbol \gamma}}

\newcommand{\bfalpha}{{\boldsymbol \alpha}}

\renewcommand{\grad}{{\boldsymbol \nabla\, }}

\usepackage[utf8]{inputenc} 
\usepackage[T1]{fontenc}    
\usepackage{url}            
\usepackage{booktabs}       
\usepackage{amsfonts}       
\usepackage{nicefrac}       
\usepackage{microtype}      
\usepackage{xcolor}         
\usepackage{array}

\usepackage[english]{babel}
\usepackage[controls]{animate}
\usepackage{tikz}
\usetikzlibrary{arrows}
\usepackage{algorithm}
\usepackage{algpseudocode}
\usepackage{amsmath}
\usepackage{amsfonts}
\usepackage{graphicx}
\usepackage{adjustbox}
\usepackage{float}
\usepackage{placeins}
\usepackage{wrapfig}

\usepackage[colorlinks=true, allcolors=blue]{hyperref}
\usepackage{booktabs}
\usepackage{cleveref}

\usepackage{amsthm}
\usepackage{multirow}       
\usepackage{siunitx}        

\newtheorem{theorem}{Theorem}[section]
\newtheorem{lemma}[theorem]{Lemma}

\title{Mesh-Free Convolution: Learned Spectral Attenuation and Transport}

\author{%
  Niloufar Zakariaei \\
  University of British Columbia \\
  \texttt{niloufar.zakariaei@ubc.ca}
  \And
  Manav Manav \\
  Computational Geoscience Inc. \\
  \texttt{manav@compgeoinc.com}
  \And
  Eldad Haber \\
  University of British Columbia \\
  \texttt{ehaber@eoas.ubc.ca}
}

\begin{document}

\maketitle

\begin{abstract}
Learning solution operators for partial differential equations (PDEs) requires models that can represent spatial structure across grids and scales. This becomes especially important in time-dependent problems, where the learned operator is often used autoregressively: each predicted state is fed back as input to predict future states. In this setting, small spatial errors can be repeatedly propagated through the learned dynamics, leading to instability and drift, vitiating long-term solution fidelity.

We propose Mesh-Free Convolutions (MFCs), filters defined in continuous space and discretized only at evaluation time. This allows the same learned filters to be applied across resolutions. Rather than imposing a grid-dependent stencil or a fixed spectral cutoff, MFCs learn smooth scale-dependent filtering and include directional terms for transport-like behavior.

Across fluid-dynamics and weather-forecasting benchmarks, MFC-based models produce stable long-horizon rollouts, transfer reliably across resolutions, and better preserve long-time statistics. These results suggest that continuous, mesh-free filtering is a useful inductive bias for PDE operator learning, especially in autoregressive prediction.
\end{abstract}

\section{Introduction and Motivation}
Many problems in science and engineering are governed by partial differential equations (PDEs). In these settings, one is often interested in learning a mapping from an input field, such as an initial condition, a forcing term, or a material coefficient, to an output field, such as a future state or system response. While classical numerical solvers can achieve high accuracy, repeatedly computing high-resolution solutions is computationally expensive. Neural operators seek to mitigate this cost by learning solution operators directly from data, enabling efficient inference across a range of PDE instances and, when possible, generalization across discretizations~\cite{kovachki2023neural,lu2021learning,li2020neural,li2020multipole,li2020fourier}.
When the PDE, coefficients, and boundary conditions are fully specified and a single certified solve is the goal, a converged classical solver remains the right tool~\cite{mcgreivy2024weak}. Learned operators target amortized many-query settings and systems whose physics is incompletely specified; the fully specified benchmarks below are controlled tests of rollout stability, long-time statistics, and resolution transfer.

The Fourier Neural Operator (FNO) is one of the standard architectures for this setting~\cite{li2020fourier, kossaifi2025librarylearningneuraloperators, duruisseaux2025guide}. Instead of applying only local convolutions in physical space, FNO applies learned transformations in Fourier space for efficient global mixing. This has made FNO effective on several operator-learning benchmarks~\cite{li2020fourier,tran2021factorized,rahman2022u,wen2022u}. 
However, standard FNO layers retain only a prescribed set of Fourier modes. 
This makes the layer efficient, but it also imposes a hard bandwidth limit: modes outside the retained block are not directly represented by the spectral convolution.

Such a design is often effective when the relevant dynamics is  smooth or dominated by low-frequency content. It can become restrictive when fine-scale, localized, anisotropic, or transport-driven structures affect the future evolution. In autoregressive prediction, this issue is amplified because the learned operator is repeatedly composed with itself, and small phase, amplitude, or spectral errors can be fed back into later inputs \cite{mccabe2023multiple}. Hard spectral truncation is therefore not merely an implementation detail; it is one mechanism that can contribute to oversmoothing, phase drift, and loss of dynamically relevant structures during long rollouts~\cite{tran2021factorized,gao2025dynamic}.

Several FNO variants improve efficiency, expressivity, and robustness: FFNO uses low-rank spectral factorization~\cite{tran2021factorized}, Geo-FNO handles complex geometries through learned deformations~\cite{li2023fourier}, and AFNO employs adaptive spectral mixing~\cite{guibas2021adaptive}. U-NO and U-FNO introduce multiscale architectures~\cite{rahman2022u,wen2022u,raonic2023convolutional}, while transformer-based operators~\cite{cao2021choose} and DSFNO~\cite{gao2025dynamic} provide alternative global and cross-frequency interactions. In contrast, MFC parameterizes convolution through the time-$\tau$ solution operator of a continuous PDE followed by a directional readout, enabling smooth, resolution-aware attenuation and transport-like responses without a fixed spectral cutoff.

\paragraph{Motivating example.}
We first illustrate why mode truncation can matter in autoregressive prediction. We train FNO one-step predictors on the PDEBench 1D advection benchmark~\cite{takamoto2022pdebench},
\[
u(\cdot,t_n)\mapsto u(\cdot,t_{n+1}),
\]
and then roll them out by feeding each prediction back as the next input. This tests the learned operator under repeated composition, not only under one-step error.

Figure~\ref{fig:fno_rollout_failure} compares two FNOs that differ only in the number of retained Fourier modes. With a small mode budget, the rollout develops clear long-horizon error, including phase drift, amplitude distortion, and oscillatory artifacts. Increasing the number of modes substantially improves the final state. This example does not imply that spectral truncation is the only source of autoregressive error, nor that simply increasing the number of modes is always the best solution. Rather, it shows that the spectral representation used inside the learned time-stepper can directly affect which structures are preserved under repeated composition.

\begin{figure}[h]
    \centering
    \includegraphics[width=0.9\linewidth]{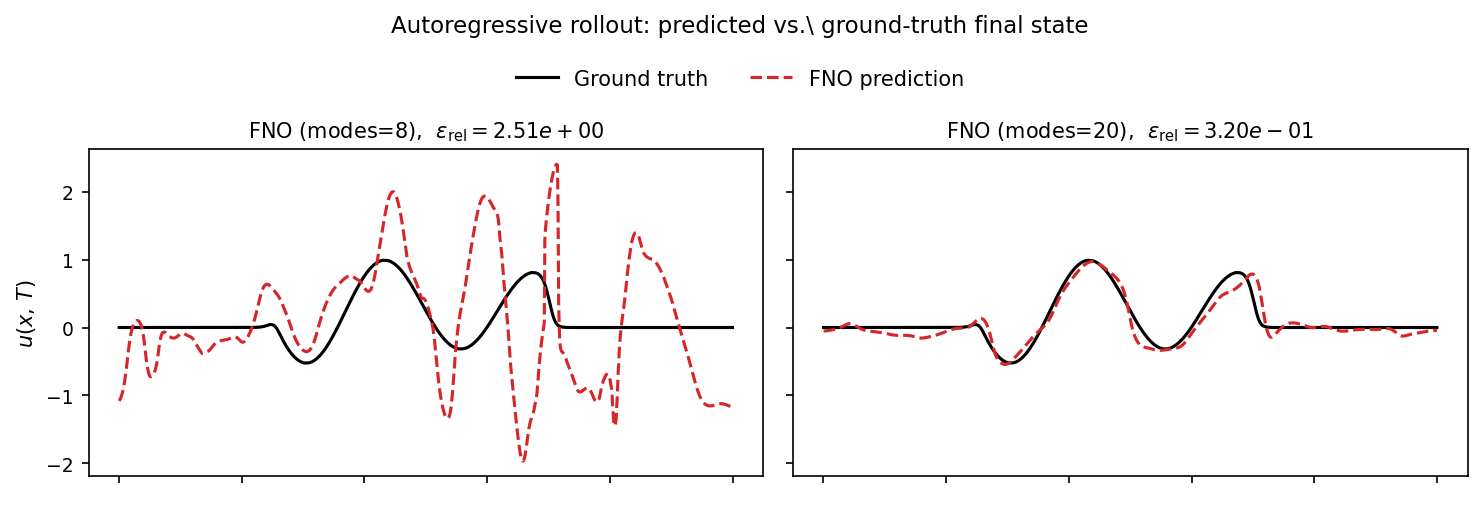}
    \caption{
    Autoregressive rollout on the PDEBench 1D advection benchmark~\cite{takamoto2022pdebench}. 
    Both FNOs are trained as one-step predictors and differ only in the number of retained Fourier modes. 
    A small mode budget leads to phase drift, amplitude distortion, and oscillatory artifacts, while increasing the number of modes improves the final-state prediction.
    }
    \label{fig:fno_rollout_failure}
\end{figure}

\paragraph{Research questions.}
The motivating example above suggests that one important challenge is the preservation of dynamically relevant spatial content under repeated application of the learned operator. This leads to the following research questions:
\begin{enumerate}
    \item Can hard spectral truncation in standard FNO layers be replaced by a learned smooth attenuation mechanism that retains useful fine-scale structure while damping spurious high-frequency components?
    \item Can a continuous-space convolution parameterization improve cross-resolution behavior by defining the filter before discretization and evaluating it on the mesh only at application time?
    \item Can directional first-order terms complement spectral attenuation by representing asymmetric, transport-like responses rather than purely diffusive smoothing?
\end{enumerate}

\paragraph{Our contributions.}
To address these questions, we propose Mesh-Free Convolutions (MFCs): continuous-space filters that are discretized only when evaluated on a computational mesh. \textbf{First}, MFCs replace hard Fourier-mode truncation with smooth learned spectral attenuation, obtained from an anisotropic parabolic parameterization, so useful fine-scale content can be retained while unstable high-frequency components are damped. \textbf{Second}, the filter is defined before discretization, allowing the same learned convolution to be evaluated on different meshes and resolutions. \textbf{Third}, we include a directional readout with first-order terms, enabling asymmetric, transport-like responses beyond purely diffusive smoothing. We evaluate MFCs in settings where spatial inductive bias is repeatedly exposed: autoregressive rollout, cross-resolution evaluation, and long-time statistics.

\section{Mesh-Free Convolutions}
\label{sec:MFC_method}

We introduce Mesh-Free Convolutions (MFCs), a family of convolutional operators defined first in continuous space and discretized only at evaluation time. In contrast to standard convolution kernels, whose weights are tied to a particular grid stencil, MFCs are parameterized as continuous operators acting on functions. Once the operator has been learned, it can be evaluated on any compatible computational mesh by discretizing its Fourier symbol on that mesh.

We define MFCs as  the continuous convolution
\begin{eqnarray}
    \label{eq:conv-int}
    v(\bfx)
    =
    \int_{\Omega}
    {\cal K}_{\bfxi}(\bfx-\bfx')\,u(\bfx')\,d\bfx',
\end{eqnarray}
where the kernel ${\cal K}_{\bfxi}$ is parameterized by learnable parameters $\bfxi$. The parametrization of MFCs is the key component which distinguishes it from the other operator learning approaches such as FNOs and allows it to be expressive, stable, and well matched to the structure of PDE solution operators. This formulation also separates the definition of the filter from the mesh on which it is evaluated.

Our choice is motivated by the connection between convolution, Green's functions, and linear PDEs \cite{EvansPDE,smets2023pde,long2018pde,folland1995introduction}. Many physically relevant evolution equations combine diffusion, which damps small-scale oscillations, with transport, which moves coherent structures through space. MFCs parameterize the convolution through the solution operator of a continuous differential equation, in contrast to FNOs which learn Fourier multipliers on a \emph{prescribed finite} set of modes not explicitly biased towards the symbol of a PDE evolution operator. This endows MFCs with an inherent bias toward smoothing, anisotropy, and transport-like motion, all of which are common in time-dependent physical systems.

Let $u$ be the input function, possibly with multiple channels, and let
$v={\cal C}(\bfxi)u$ denote the output of the mesh-free convolution. We construct ${\cal C}(\bfxi)$ by composing three operations. First, a trainable channel mixing map, implemented as a pointwise convolution, maps the input channels into latent channels. Second, each latent field is evolved by an anisotropic parabolic operator. Third, the evolved field is read out through a zeroth- and first-order differential operator, together with a learned spatial shift. In continuous notation, the basic parabolic part is
\begin{equation}
    \label{convdiff}    
    \tilde v_s
    =
    {\cal L}(\bfalpha)\tilde v,
    \qquad
    \tilde v(\bfx,0)
    =
    \bfR u(\bfx),
    \qquad
    s\in[0,\tau],
\end{equation}
where $\bfR$ is a learned channel mixing operator, and $\tau\geq 0$ is the learned evolution time. The parabolic operator is defined by
\begin{eqnarray}
    \label{eq:heat}
    {\cal L}(\bfalpha) \tilde  v
    =
    \alpha_{xx}\frac{\partial^2 \tilde v}{\partial x^2}
    +
    2\alpha_{xy}\frac{\partial^2 \tilde v}{\partial x \partial y}
    +
    \alpha_{yy}\frac{\partial^2 \tilde v}{\partial y^2},
\end{eqnarray}
where we enforce
\[
    \alpha_{xx}>0,
    \qquad
    \alpha_{yy}>0,
    \qquad
    \alpha_{xx}\alpha_{yy}-\alpha_{xy}^2>0,
\]
so that the diffusion operator is positive definite, guaranteeing anisotropic parabolic smoothing and avoiding unstable high-frequency amplification \cite{tos}.


The evolution of the latent channels under the parabolic operator is followed by the directional readout
\begin{equation}
    \label{eq:readout}
    v(\bfx)
    =
    \left(
    \gamma_0
    +
    \gamma_x \partial_x
    +
    \gamma_y \partial_y
    \right)
    \tilde v(\bfx-\boldsymbol{\delta},\tau),
    \qquad
    \boldsymbol{\delta}=(\delta_x,\delta_y).
\end{equation}
Here , $\bfgamma=(\gamma_0,\gamma_x,\gamma_y)$ are the learned parameters controlling the zeroth and first-order readout, and $\boldsymbol{\delta}$ is a learned continuous shift. Note that the set of learnable parameters is $\bfxi := \{\alpha_{xx}, \alpha_{yy}, \alpha_{xy}, \tau, \gamma_0,\gamma_x,\gamma_y,\delta_x,\delta_y, \bfR,\bfS\}$, where $\bfS$ is the projection network that projects $v$ to the output space.

\paragraph{Advection and learned shifts.}
Diffusion alone is not sufficient for many physical systems. In fluid flow, weather, waves, and transport-dominated PDEs, coherent structures are not only smoothed; they are also displaced. For example, vortices in Navier--Stokes, fronts in weather fields, and sharp layers in Burgers-type equations move through space over time. A purely symmetric smoothing kernel can damp these structures but cannot naturally represent their translation or directional skew.

To incorporate this behavior, MFCs include a learned continuous shift $\boldsymbol{\delta}$ and a first-order directional readout. The shift acts as a simple advection surrogate:
\[
    \tilde v(\bfx,\tau)
    \mapsto
    \tilde v(\bfx-\boldsymbol{\delta},\tau).
\]
This is the exact solution of the constant-velocity advection equation
\[
    w_t + \bfb\cdot \grad w = 0,
    \qquad
    w(\bfx,0)=w_0(\bfx),
\]
with displacement $\boldsymbol{\delta}=\tau\bfb$. Thus, the learned phase shift in Fourier space represents a translation of the kernel center. The first-order terms $(\gamma_x\partial_x+\gamma_y\partial_y)$ then allow the filter to form oriented, asymmetric responses. Together, the shift and directional derivative terms produce kernels that are not merely diffusive, but can also represent transport-like motion and upwind/downwind asymmetry.

This is useful in operator learning problems. In advection-dominated regimes, errors are often phase errors rather than amplitude errors: the predicted structure has the right shape but is slightly displaced. A filter family that explicitly includes translations and directional responses is therefore better aligned with the geometry of the problem. It can represent moving fronts, advected vortices, and anisotropic filaments using a small number of physically meaningful parameters. This is more natural than a generic FNO multiplier, which must learn such behavior indirectly from data on a finite set of modes.

\paragraph{Fourier symbol.}
For periodic domains, the continuous operator has a simple Fourier representation. Let
$\bfk=(k_x,k_y)$ be the spatial frequency. Then, the MFC Fourier symbol is given as follows:
\begin{equation}
    \label{eq:mfc_symbol_full}
    G_{\bfxi}(\bfk)
    =
    \underbrace{
    \exp\left(
    -\tau \bfk^\top \bfA \bfk
    \right)
    }_{\text{parabolic attenuation}}
    \underbrace{
    \left(
    \gamma_0+i\gamma_x k_x+i\gamma_y k_y
    \right)
    }_{\text{zeroth/first-order readout}}
    \underbrace{
    \exp\left(-2\pi i\,\bfk\cdot\boldsymbol{\delta}\right)
    }_{\text{learned advection shift}},
\end{equation}
where the parabolic evolution contributes the factor $\exp\left(
    -\tau \bfk^\top \bfA \bfk
    \right)$
with
\[
    \bfA
    =
    \begin{pmatrix}
        \alpha_{xx} & \alpha_{xy}\\
        \alpha_{xy} & \alpha_{yy}
    \end{pmatrix}.
\]
This exponential term smoothly attenuates high frequencies, mediated by the learned $\tau$. The directional readout contributes $(\gamma_0+i\gamma_x k_x+i\gamma_y k_y)$ and the learned shift contributes the phase factor $\exp\left(-2\pi i\,\bfk\cdot\boldsymbol{\delta}\right)$.

The convolution is then applied by
\begin{equation}
    \label{eq:mfc_channel_symbol}
    \widehat v_j(\bfk)
    =
    \sum_{i=1}^{c_{\rm in}}
    G_{ij}(\bfk)\,\widehat u_i(\bfk),
    \qquad
    j=1,\ldots,c_{\rm out}.
\end{equation}
(The proof of the stability of the MFC symbol is provided in Appendix~\ref{app:mfc-stability}).
In the implementation, this spectral mixing is followed by a pointwise convolution over the output channels.

The above formulation highlights a key distinction from FNOs: whereas FNOs parameterize free complex weights over a truncated set of Fourier modes and discard the remainder, MFCs instead learn a continuous symbol defined across all frequencies on the evaluation grid. By grounding the framework in the Fourier symbol, MFCs impose an inductive bias that more closely reflects the structure of solution operators for linearized PDEs, while retaining sufficient flexibility in the surrounding neural architecture to capture nonlinear dynamics. This structural property is particularly important in autoregressive PDE learning, where small amplitude and phase errors have qualitatively different effects over long rollouts. Excessive damping leads to energy loss and oversmoothing, while phase inaccuracies result in the misalignment of coherent structures. In contrast to truncated FNO layers, which must implicitly learn both effects through unconstrained complex weights on a limited set of modes, MFCs provide a more interpretable, PDE-inspired decomposition into attenuation, directional response, and advection-like shifts. Our formulation incorporates the solution operator of a linear parabolic PDE, effectively biasing the learned operator toward transformations arising from diffusion–advection dynamics and contributing to improved stability and structure preservation in long-horizon predictions.horizon predictions.

\begin{wrapfigure}{r}{0.5\textwidth}
    \centering
    \includegraphics[width=0.99\linewidth]{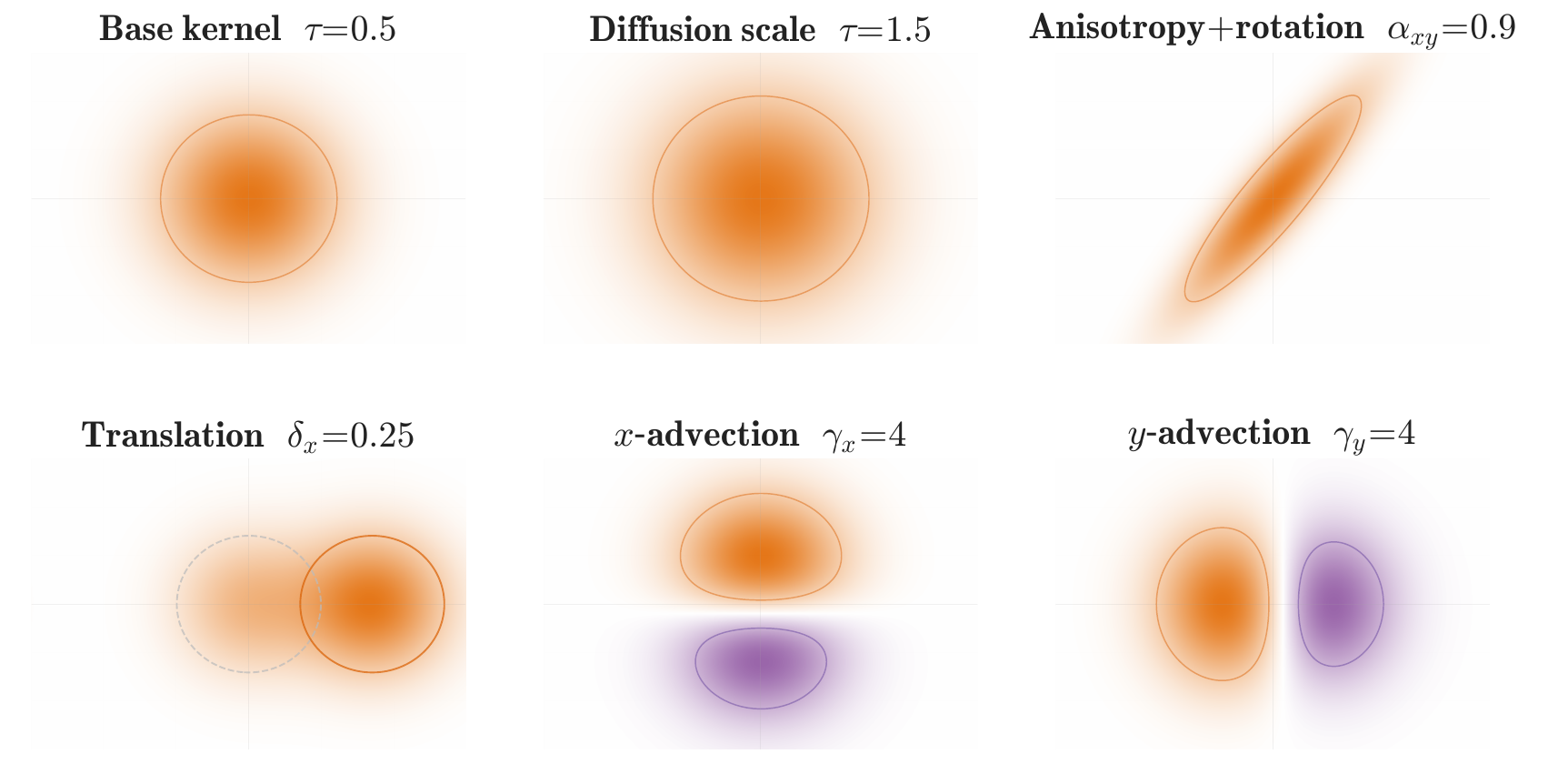} 
    \caption{
    Kernel family produced by the MFC parameterization.
    Varying $\tau$ controls scale (localized to diffuse); the diffusion-tensor entries set anisotropy and orientation; $\gamma_x,\gamma_y$ introduce directional, derivative-like lobes; and the shift $\delta$ translates the kernel center.
    A compact parameter vector $\bfxi$ thus spans smooth, anisotropic, and transport-like filters without hand-crafted stencils.
    }
    \label{fig:convs}
\end{wrapfigure}

\paragraph{Kernel interpretation.}
The features of the convolution kernel are controlled by the set of learnable parameters $\bfxi$. The parameter $\tau$ controls the spatial scale of the kernel. As $\tau\to 0$, the parabolic factor approaches one, and the kernel becomes increasingly localized. As $\tau$ increases, high frequencies are damped more strongly, producing wider and smoother kernels. The diffusion tensor $\bfA$ controls the anisotropy and orientation of the smoothing. The coefficients $\gamma_x$ and $\gamma_y$ create odd, derivative-like components, producing oriented lobes and asymmetric responses. The shift $\boldsymbol{\delta}$ moves the kernel center continuously, allowing off-centered filters and transport-like displacement.

To illustrate this behavior, we plot the MFC filters for a few choices of $\bfxi$ in Figure~\ref{fig:convs}. By varying $\bfxi$, we obtain a family of convolutional kernels ranging from local to global, isotropic to anisotropic, centered to shifted, and symmetric to directional. These are the types of geometric structures commonly observed in learned convolutional filters, but here they arise from a compact PDE-based parameterization.

\paragraph{Computing MFCs on a discretized mesh.}
\label{sec:computingMFConGrid}
Although MFCs are defined in continuous space, they are evaluated on a discrete mesh in computation. For a rectangular domain with periodic boundary conditions, we use the Fast Fourier Transform (FFT). Let $\bfu^h$ be a discrete input tensor. We compute its Fourier transform
\[
    \widehat{\bfu}^h = \bfF \bfu^h,
\]
construct the discrete frequency grid $\bfk^h=(k_x^h,k_y^h)$, evaluate the symbol
$G_{\bfxi}(\bfk^h)$ on that grid, multiply in Fourier space, and then apply the inverse FFT:
\begin{equation}
    \label{eq:mfc_discrete}
    \bfv^h
    =
    \bfF^{-1}
    \left[
    G_{\bfxi}(\bfk^h)\,\widehat{\bfu}^h
    \right].
\end{equation}
Thus, the spatial convolution matrix is never assembled. The operator is applied matrix-free by evaluating and multiplying its Fourier symbol.

This construction also explains the cross-resolution behavior of the layer. Since the learned parameters define a continuous symbol, the same parameters can be evaluated on a different frequency grid when the spatial resolution changes. This is different from a standard convolutional stencil, whose weights are tied to fixed pixel offsets, and from a standard truncated FNO layer, whose learned weights are tied to a fixed set of retained modes. In MFC, the parameters define the functional form of the filter, while the mesh only determines where that symbol is sampled.

An important consequence is that, on any fixed mesh, an MFC layer can approximate standard convolutional behavior while retaining a continuous interpretation (proof in Appendix~\ref{app:mfc_cnn_equivalence}). The purpose of MFC is therefore not to reduce expressivity relative to CNNs, but to impose a structured PDE-inspired bias on the space of filters.

\paragraph{Computational complexity.}
The dominant cost of MFC stems from two operations: the FFT and the $1\times 1$ convolution (pointwise channel mixing). For an input with $n$ spatial points and $c$ active channels, the FFT cost scales as ${\cal O}(c n\log n)$. The $1\times 1$ convolution cost is ${\cal O}(c^2n)$. The remaining operations consist of evaluating the symbol, multiplying Fourier coefficients,  all of which scale linearly in the number of spatial points up to channel factors. In this sense, MFC is matrix-free with respect to the spatial convolution matrix: it avoids assembling a dense or stencil-based convolution operator and instead applies the operator through its Fourier symbol.

\section{Experiments}
We evaluate MFC on three benchmarks, summarized in Table~\ref{tab:datasets}, each chosen to test a specific aspect of our contributions: \emph{smooth spectral attenuation} for autoregressive stability and spectral fidelity (\cref{sec:exp-stability}), \emph{mesh-free parameterization} for zero-shot cross-resolution transfer (\cref{sec:exp-crossres}), and \emph{transport-like terms} for preserving localized, advection-driven structures under distributional shift (\cref{sec:exp-localized}). Dataset details, training protocols are provided in Appendix \ref{app:training-details}.


\begin{table}[h]
\centering
\caption{Summary of experimental benchmarks.}
\label{tab:datasets}
\small
\begin{tabular}{l l l l l}
\toprule
\textbf{Dataset} & \textbf{PDE / Data} & \textbf{Resolution} & \textbf{Baselines} & \textbf{Evaluation focus} \\
\midrule
Navier--Stokes \cite{kovachki2023data} & 2D incomp\  & $128^2$ & FNO , F-FNO , UNet & Rollout stability \\
WeatherBench2 \cite{rasp2024weatherbench} & ERA5  & $64{\times}32 \!\to\! 240{\times}121$ & AFNO, FNO & Long-time statistics \\
Burgers 2D \citep{zhou2024masked} & 2D Burgers & $64^2 \!\to\! 128^2$ & FNO, F-FNO & Localized structures\\
\bottomrule
\end{tabular}
\end{table}

\subsection{Autoregressive stability and spectral fidelity}
\label{sec:exp-stability}
This subsection addresses our first contribution: replacing hard Fourier truncation with smooth learned spectral attenuation. We test whether this improves autoregressive stability and scale-wise fidelity under repeated rollout.



\paragraph{Navier--Stokes.}
To compare MFC with Fourier-based operators, we replace the convolutional layers in a standard UNet~\cite{ronneberger2015u} with MFC layers, yielding MFC-UNet. 
We compare against the same UNet backbone with standard convolutions~\cite{kossaifi2025librarylearningneuraloperators}, an FNO baseline~\cite{kossaifi2025librarylearningneuraloperators}, and F-FNO~\cite{tran2021factorized}, a factorized Fourier variant that provides a stronger spectral-operator baseline. 
We evaluate all models on the \(Re=5000\) Navier--Stokes Kolmogorov-flow split from~\cite{kovachki2023data}, consisting of 100 trajectories with 500 time steps on a \(128\times128\) grid. 
Following common practice for this dataset, all models predict scalar vorticity rather than velocity. 
We report per-step relative \(\ell_2\) error, spatial correlation with the ground truth, and the radially averaged energy spectrum to assess rollout stability, trajectory alignment, and scale-wise behavior.

All models use \(q=10\) history frames to predict the next state. In training, ground-truth input frames are used to predict only the next frame, while testing is fully autoregressive after the initial 10 frames. No model is trained with a rollout loss unavailable to the others. Figure~\ref{fig:ns_rollout_error}a shows that MFC-UNet has the slowest error growth over the rollout. FNO becomes unstable after roughly 15--20 steps. F-FNO is more stable than FNO, but its error still grows faster than MFC-UNet. The standard UNet remains bounded, but drifts from the ground truth. This is also reflected in Figure~\ref{fig:ns_rollout_error}b: MFC-UNet maintains the highest spatial correlation over the rollout, while FNO and UNet lose correlation rapidly and F-FNO decays more gradually but still below MFC-UNet.

Figure~\ref{fig:ns_rollout_error}c shows the radially averaged energy spectrum at \(\Delta t=40\). MFC-UNet follows the ground-truth spectrum more closely across a broad range of wavenumbers, indicating better preservation of fine spatial detail. UNet is overly dissipative and loses energy broadly. FNO develops unstable high-wavenumber artifacts near and beyond the retained Fourier modes, while F-FNO and MFC-UNet exhibit superior scale-wise fidelity. Additional visual rollout results are provided in Appendix~\ref{sec:appendix-ns}.

Autoregressive evaluation is stricter than one-step prediction because the model must generate a self-consistent rollout under repeated application of its own time-stepper. 
For chaotic Navier--Stokes dynamics, long-horizon success is less about pointwise trajectory matching and more about stability and preservation of statistical, spatial, and spectral structure. 
In this experiment, MFC-UNet better satisfies this criterion, suggesting that smooth learned attenuation provides a more suitable bias for repeated time-stepping while retaining fine-scale structure.




\begin{figure}[h]
    \centering

    \begin{minipage}{0.32\textwidth}
        \centering
        \includegraphics[width=\linewidth]{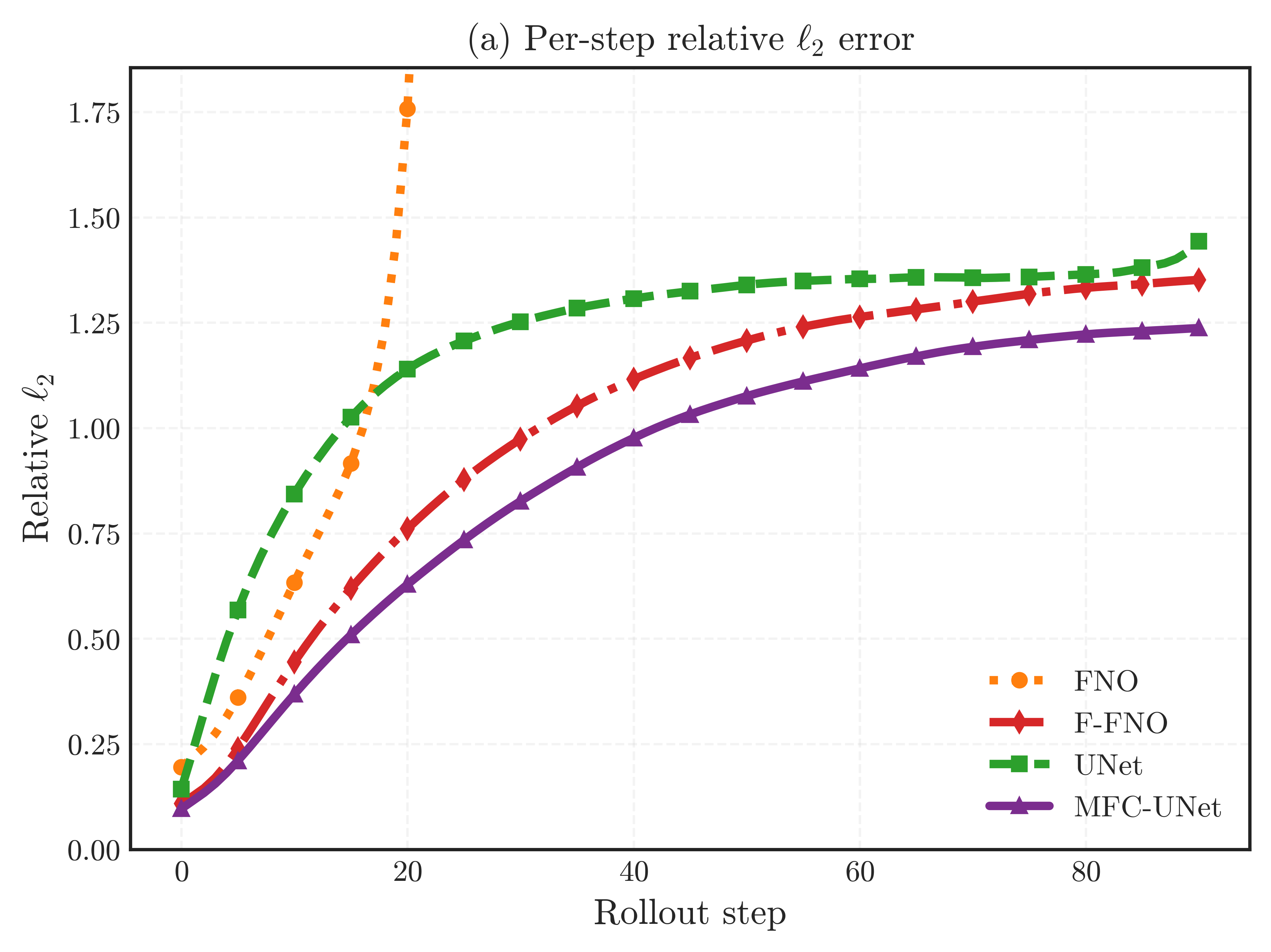}
        \vspace{-2mm}
        \subcaption{Per-step relative \(\ell_2\) error.}
        \label{fig:ns_rollout_rel_l2}
    \end{minipage}
    \hfill
    \begin{minipage}{0.32\textwidth}
        \centering
        \includegraphics[width=\linewidth]{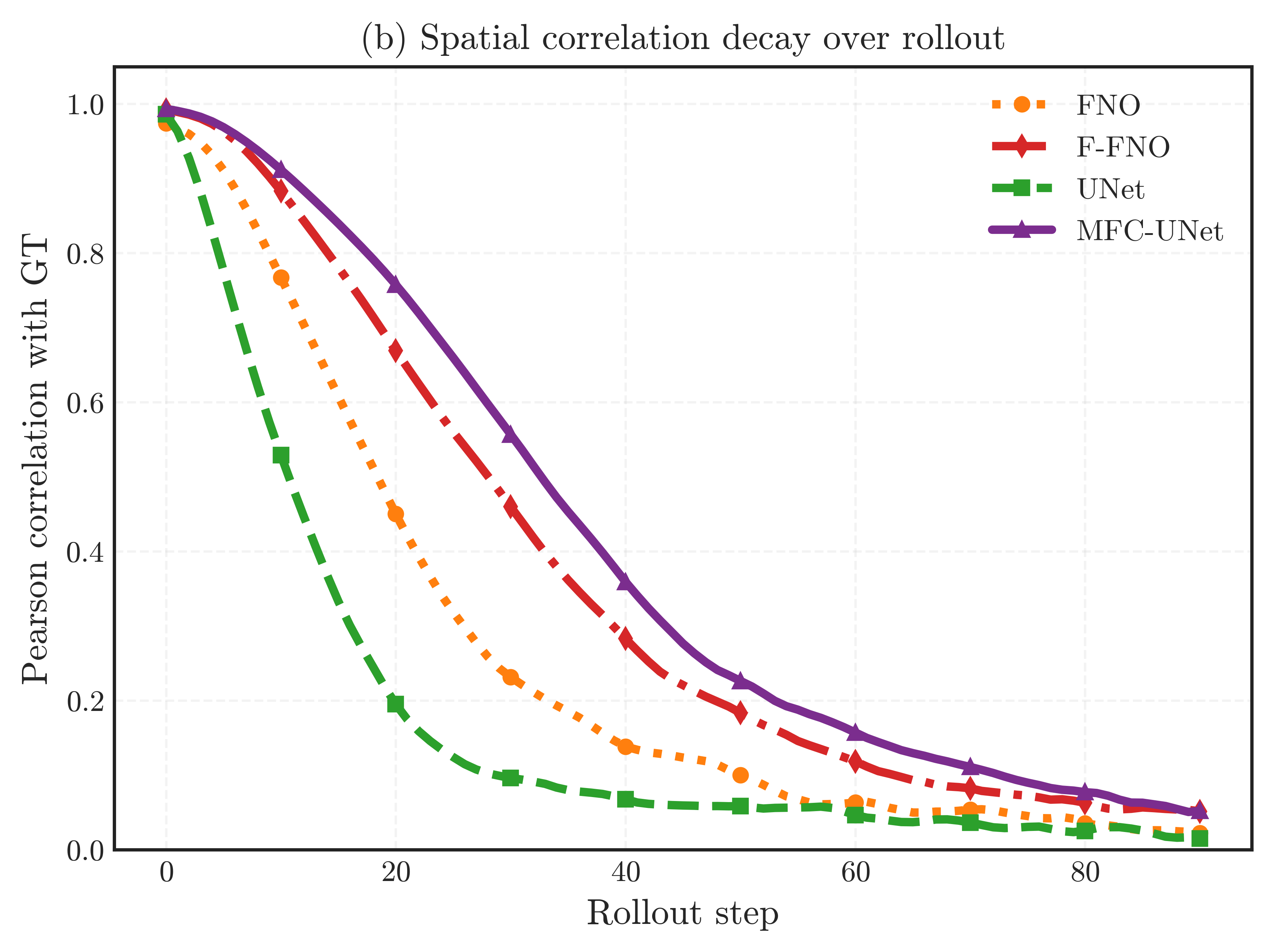}
        \vspace{-2mm}
        \subcaption{Spatial correlation with GT.}
        \label{fig:ns_correlation_decay}
    \end{minipage}
    \hfill
    \begin{minipage}{0.32\textwidth}
        \centering
        \includegraphics[width=\linewidth]{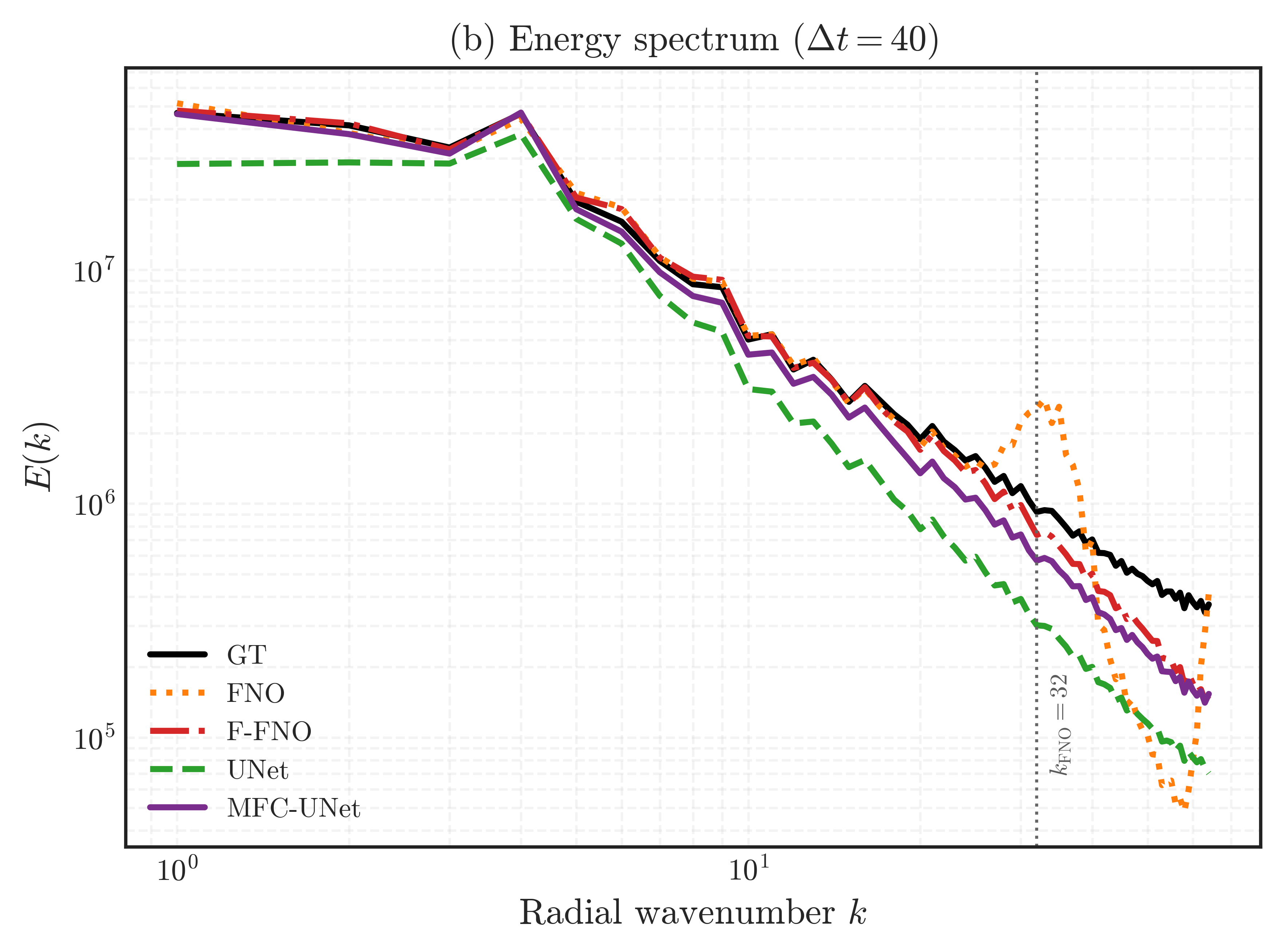}
        \vspace{-2mm}
        \subcaption{Energy spectrum at \(\Delta t=40\).}
        \label{fig:ns_energy_spectrum}
    \end{minipage}

    \caption{
    Autoregressive rollout on the 2D Navier--Stokes benchmark. 
    MFC-UNet shows the slowest error growth and highest spatial correlation, while FNO becomes unstable and F-FNO loses correlation faster. 
    At \(\Delta t=40\), the energy spectrum shows that MFC-UNet better preserves the ground-truth scale structure, whereas FNO develops high-wavenumber artifacts.
    }
    \label{fig:ns_rollout_error}
\end{figure}

\paragraph{WeatherBench2.}
We next test whether learned time-steppers remain stable and preserve long-time statistics on WeatherBench2, a benchmark built from global ERA5 reanalysis fields with standardized protocols for data-driven weather forecasting~\cite{rasp2024weatherbench}. 
We use the 2018--2020 test period and evaluate three variables: near-surface zonal wind \(U_{10}\), near-surface meridional wind \(V_{10}\), and temperature at 850 hPa \(T_{850}\). 
The data are sampled every 6 hours on a latitude--longitude grid, and all spatial errors are latitude-weighted to account for the changing grid-cell area with latitude.
All three models apply the 2D FFT on the lat--lon grid and therefore periodize latitude; there is no pole-specific treatment, and no model receives a boundary handling unavailable to the others. Per-latitude unweighted RMSE (Appendix~\ref{app:latitude}) shows that elevated polar error is shared, with MFC least affected. 
We compare against FNO and AFNO~\cite{guibas2021adaptive}, the latter serving as a strong adaptive Fourier baseline motivated by its use in FourCastNet for global weather forecasting~\cite{pathak2022fourcastnet}.

We first train the network to predict the next time step, and then run 500-step autoregressive forecasts, corresponding to 125 days, from 35 deterministically spaced initial conditions in the test period.
Rather than comparing a pointwise prediction, we compare the temporal mean of the first \(T\) predicted fields for the horizons \(T\in\{100,300,500\}\), with the corresponding ground-truth mean.
 This corresponds to trend or climate prediction \cite{petri2015impacts}.
A rollout is classified as divergent if the prediction exceeds \(10\times\) the ground-truth absolute maximum at any step up to 500. 
By this criterion, 11 of 35 FNO rollouts diverge, whereas MFC-resnet (resnet~\cite{he2016deep} with MFCs) and AFNO complete all 35 rollouts. 
Table~\ref{tab:climatology} reports latitude-weighted RMSE normalized by the latitude-weighted spatial standard deviation of the ground-truth temporal mean, \(\sigma_{\mathrm{GT}}\). 
MFC-resnet and AFNO statistics are computed over all 35 initial conditions, while FNO statistics are reported only on its 24 non-divergent rollouts.

Table~\ref{tab:climatology} shows that MFC-resnet gives the lowest climatology error across all variables and horizons. 
Its advantage is strongest for the wind fields, where AFNO’s long-time mean error grows more rapidly with horizon, especially for \(V_{10}\). 
On \(T_{850}\), MFC-resnet and AFNO remain close, with MFC-resnet slightly lower throughout. 
FNO behaves differently: even after excluding divergent rollouts, its climatology error is much larger and exceeds the ground-truth climatological variability by \(T=300\) for all variables. 
Overall, MFC-resnet matches AFNO’s rollout stability while better preserving the long-time mean state, particularly for winds.

\begin{table}[h]
\centering
\caption{Latitude-weighted climatology error on WeatherBench2.
NRMSE is RMSE normalized by the ground-truth spatial standard deviation,
\(\sigma_{\mathrm{GT}}\), and reported as a percentage.}
\label{tab:climatology}
\sisetup{detect-weight, mode=text}
\small
\begin{tabular}{l S[table-format=3.0] S[table-format=2.3] S[table-format=2.3] S[table-format=3.1] S[table-format=2.3] S[table-format=3.1] S[table-format=2.3] S[table-format=3.1]}
\toprule
& & & \multicolumn{2}{c}{\textbf{MFC-resnet}} & \multicolumn{2}{c}{\textbf{AFNO}} & \multicolumn{2}{c}{\textbf{FNO}} \\
\cmidrule(lr){4-5} \cmidrule(lr){6-7} \cmidrule(lr){8-9}
{Variable} & {Horizon \(T\)} & {$\sigma_{\mathrm{GT}}$} & {RMSE} & {NRMSE (\%)} & {RMSE} & {NRMSE (\%)} & {RMSE} & {NRMSE (\%)} \\
\midrule
  \multirow{3}{*}{$U_{10}$ (\si{m/s})} & 100 & 3.509 & \bfseries 0.569 & 16.2 & 0.572 & 16.3 & 8.227 & 234.5 \\
   & 300 & 3.514 & \bfseries 0.836 & 23.8 & 1.134 & 32.3 & 10.159 & 289.1 \\
   & 500 & 3.525 & \bfseries 1.029 & 29.2 & 1.422 & 40.3 & 10.075 & 285.8 \\
\midrule
  \multirow{3}{*}{$V_{10}$ (\si{m/s})} & 100 & 1.703 & \bfseries 0.388 & 22.8 & 0.412 & 24.2 & 3.679 & 216.0 \\
   & 300 & 1.708 & \bfseries 0.512 & 30.0 & 0.856 & 50.1 & 5.292 & 309.9 \\
   & 500 & 1.714 & \bfseries 0.663 & 38.7 & 1.083 & 63.2 & 6.020 & 351.1 \\
\midrule
  \multirow{3}{*}{$T_{850}$ (\si{K})} & 100 & 10.894 & \bfseries 1.161 & 10.7 & 1.336 & 12.3 & 7.797 & 71.6 \\
   & 300 & 10.909 & \bfseries 2.014 & 18.5 & 2.283 & 20.9 & 11.805 & 108.2 \\
   & 500 & 10.901 & \bfseries 2.750 & 25.2 & 2.825 & 25.9 & 12.950 & 118.8 \\
\bottomrule
\end{tabular}

\vspace{2pt}
{\footnotesize MFC/AFNO statistics are computed over all 35/35 ICs.; FNO statistics use only non-divergent rollouts: 24/35}
\end{table}

\begin{figure}[h]
    \centering
    \setlength{\tabcolsep}{2pt}
    \renewcommand{\arraystretch}{0.95}

    \newcommand{\imgw}{0.13
    
    \textwidth}
    \newcommand{\cbarw}{0.135\textwidth}
    \newcommand{\figdir}{Figures/wb2/ic_stability_500/globes}

    \begin{tabular}{c c c c c c}
        & Ground truth & MFC-resnet & AFNO & MFC $-$ GT & AFNO $-$ GT \\[0.2em]

        \(U10\) &
        \includegraphics[width=\imgw]{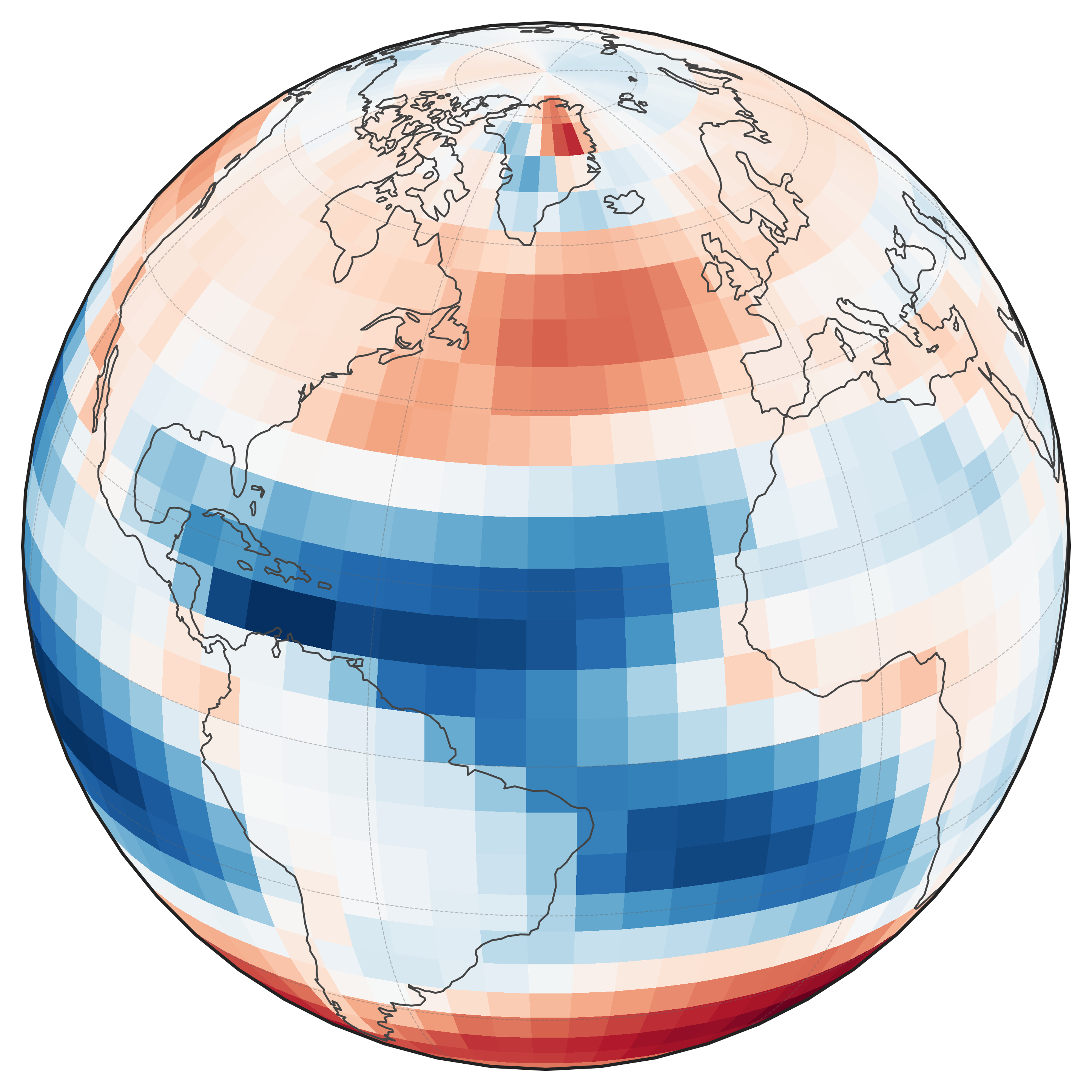} &
        \includegraphics[width=\imgw]{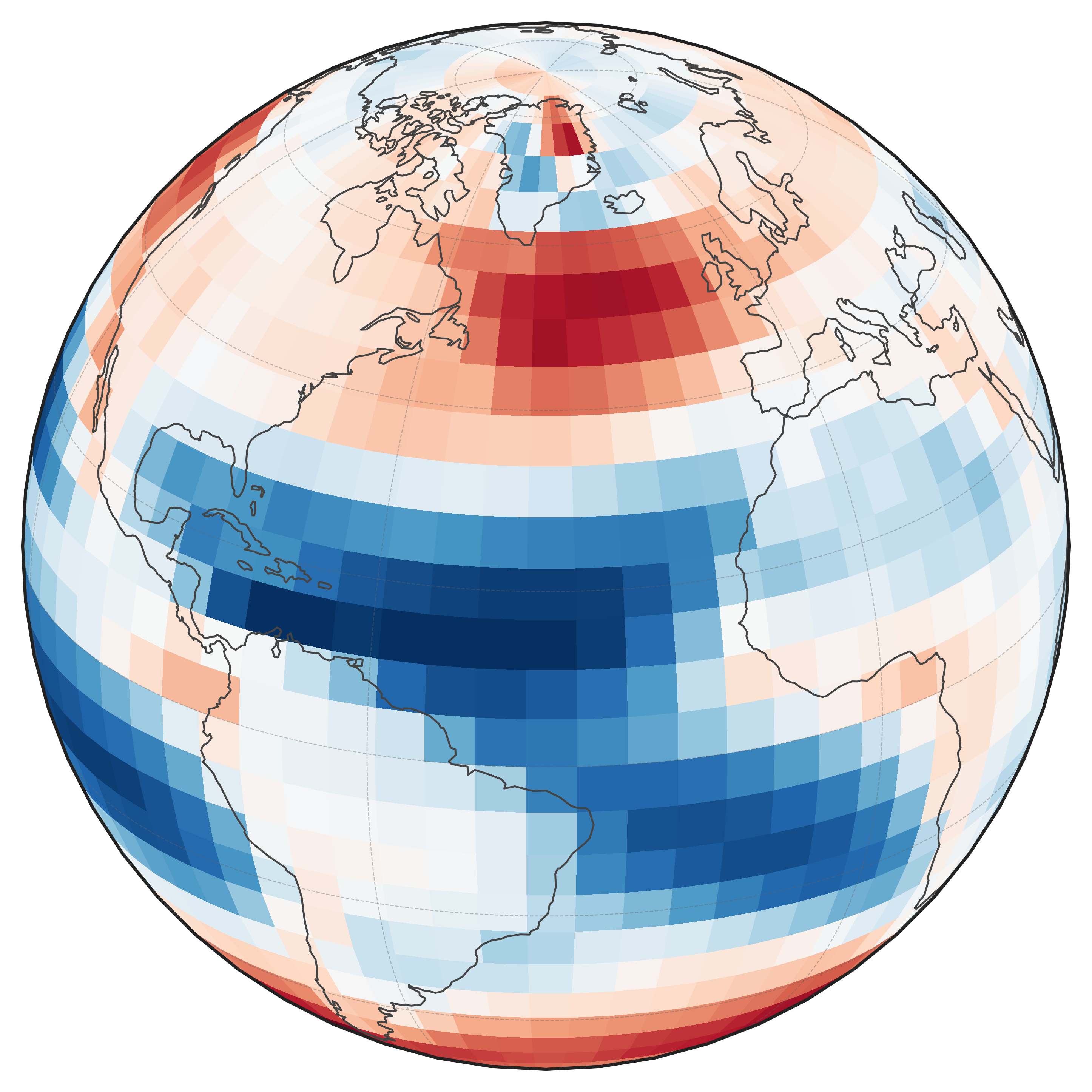} &
        \includegraphics[width=\imgw]{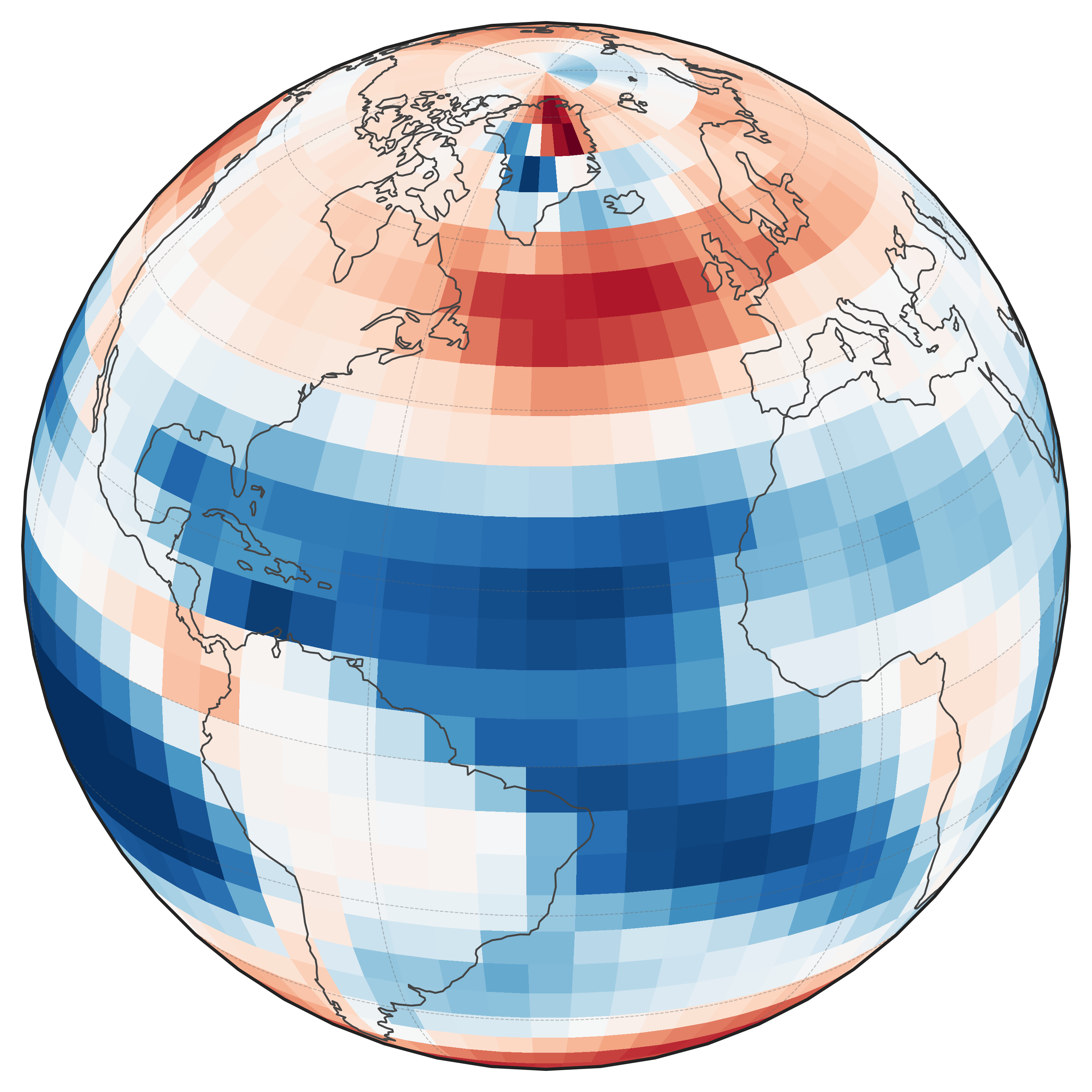} &
        \includegraphics[width=\imgw]{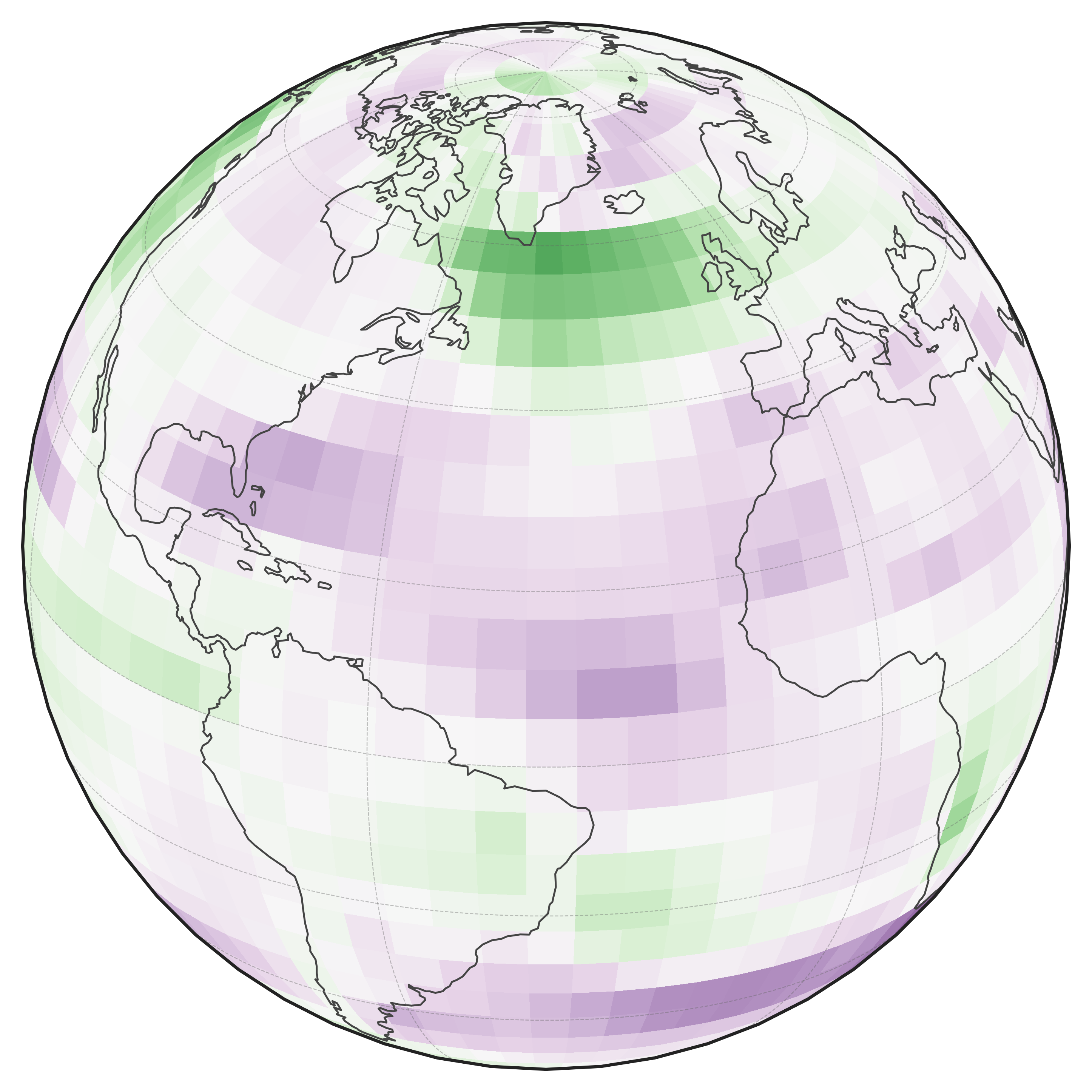} &
        \includegraphics[width=\imgw]{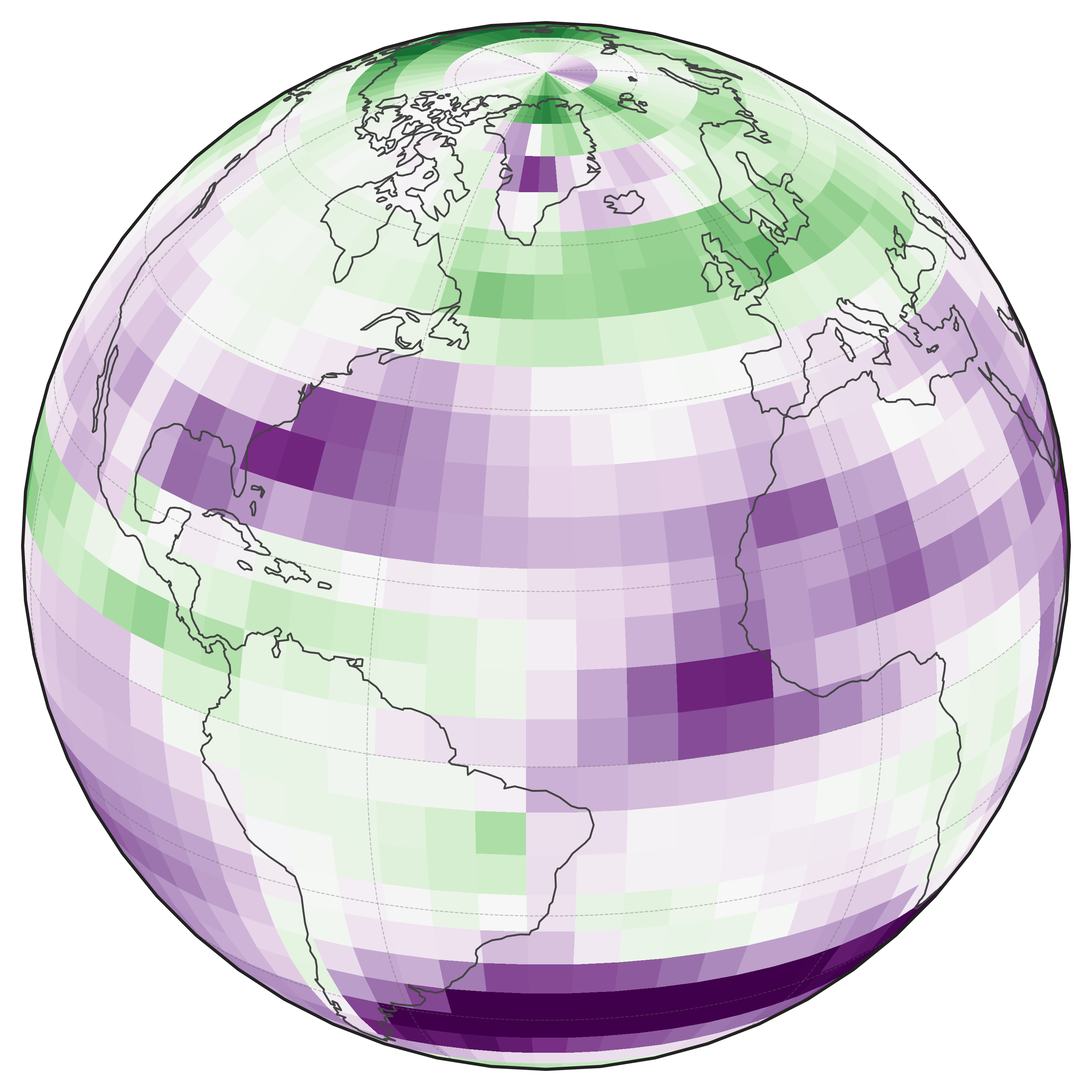}
        \\[-0.35em]
        &
        \multicolumn{3}{c}{
            \includegraphics[width=0.22\textwidth]{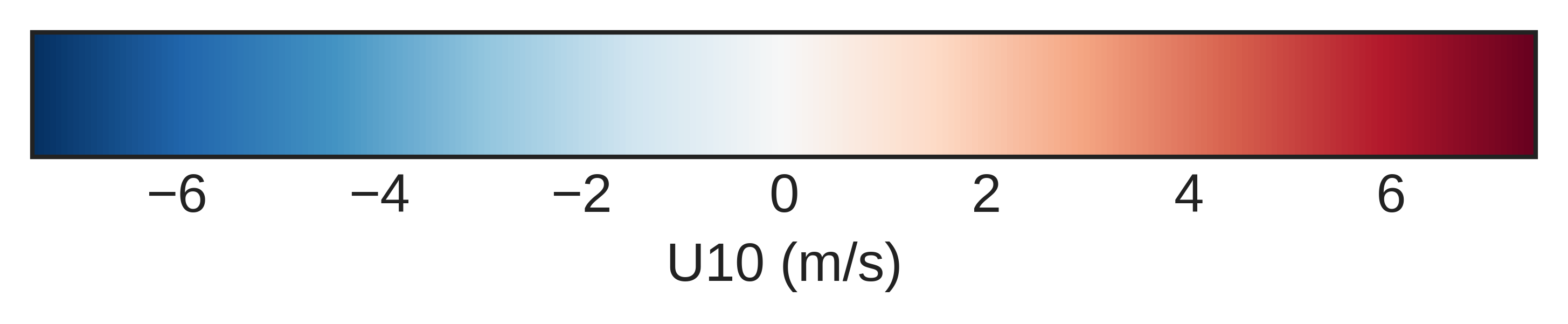}
        }
        &
        \multicolumn{2}{c}{
            \includegraphics[width=0.22\textwidth]{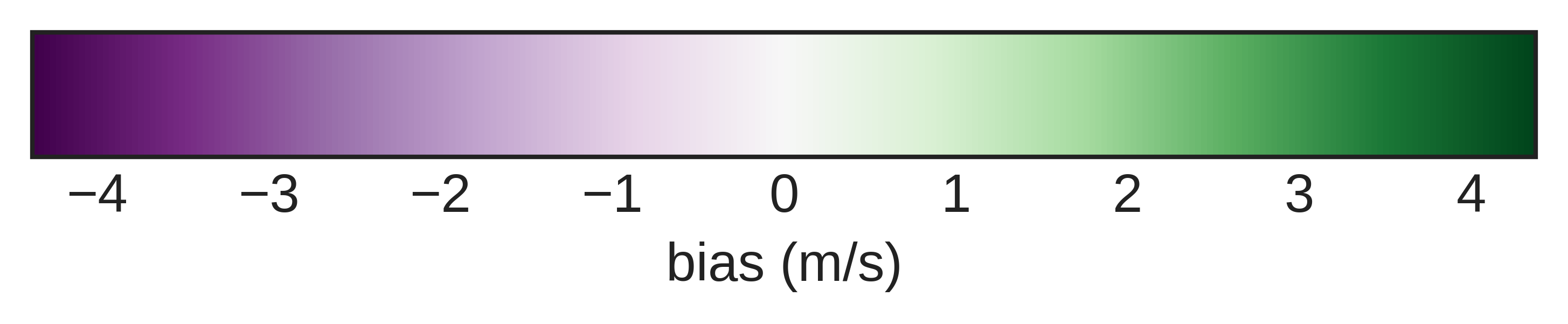}
        }
        \\[0.5em]

        \(V10\) &
        \includegraphics[width=\imgw]{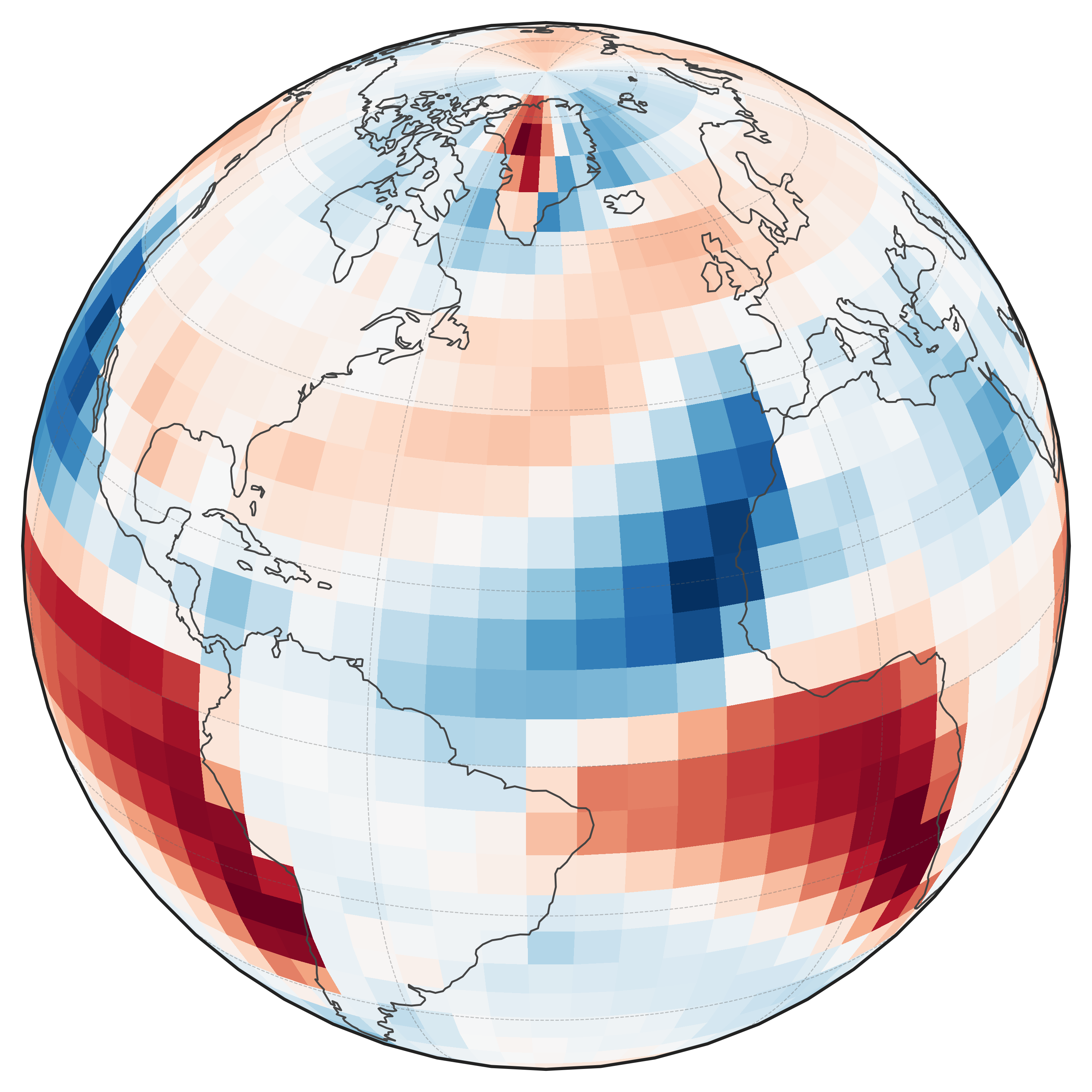} &
        \includegraphics[width=\imgw]{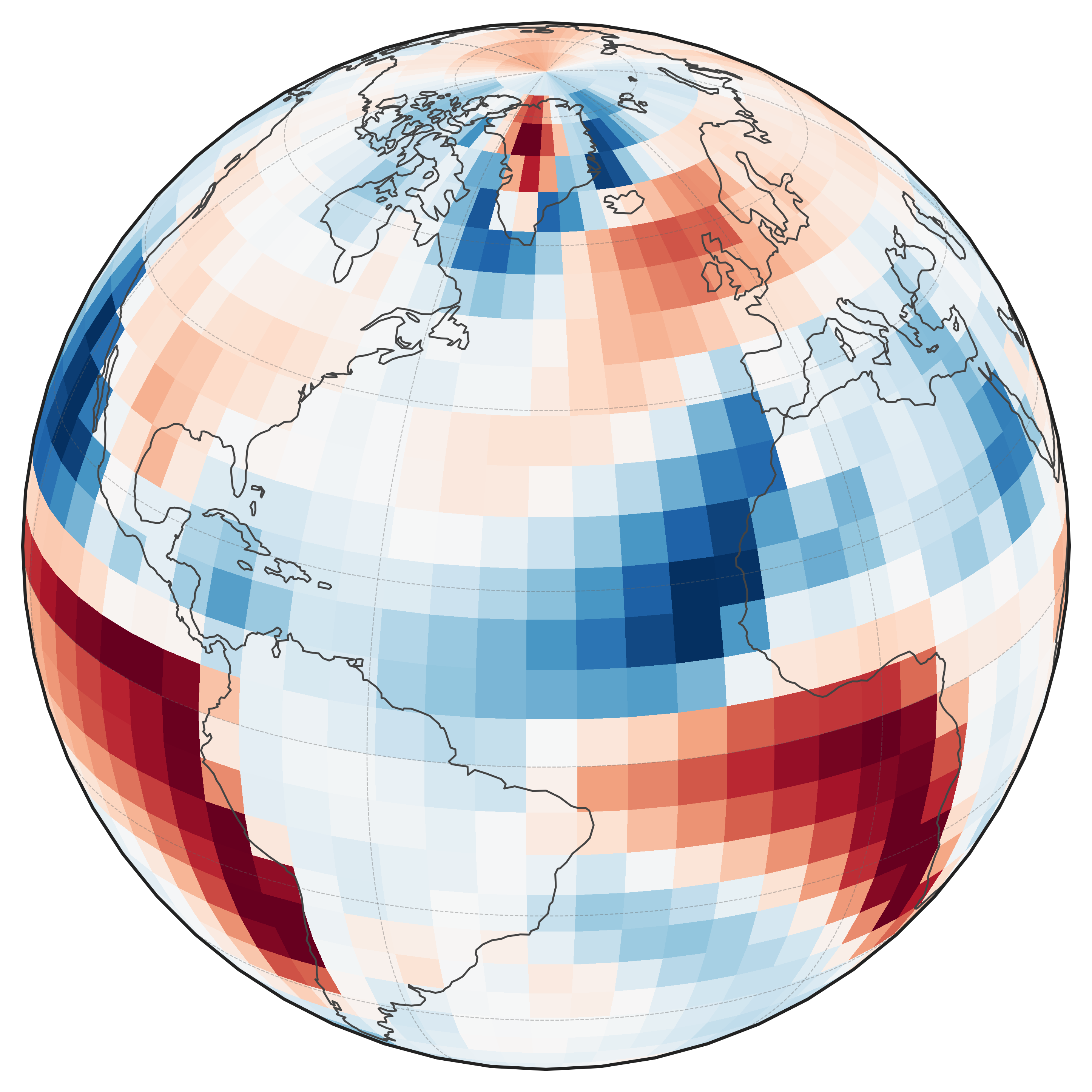} &
        \includegraphics[width=\imgw]{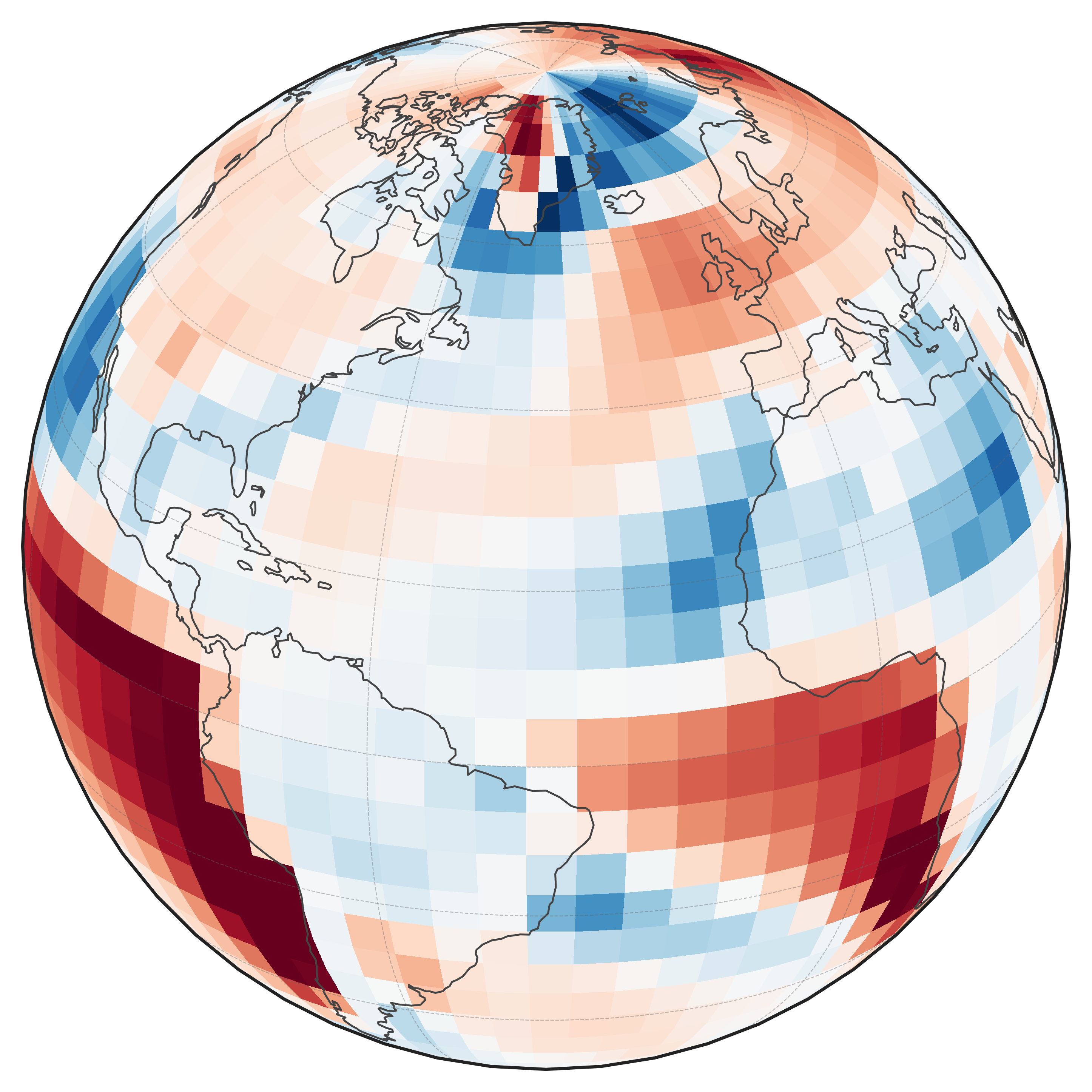} &
        \includegraphics[width=\imgw]{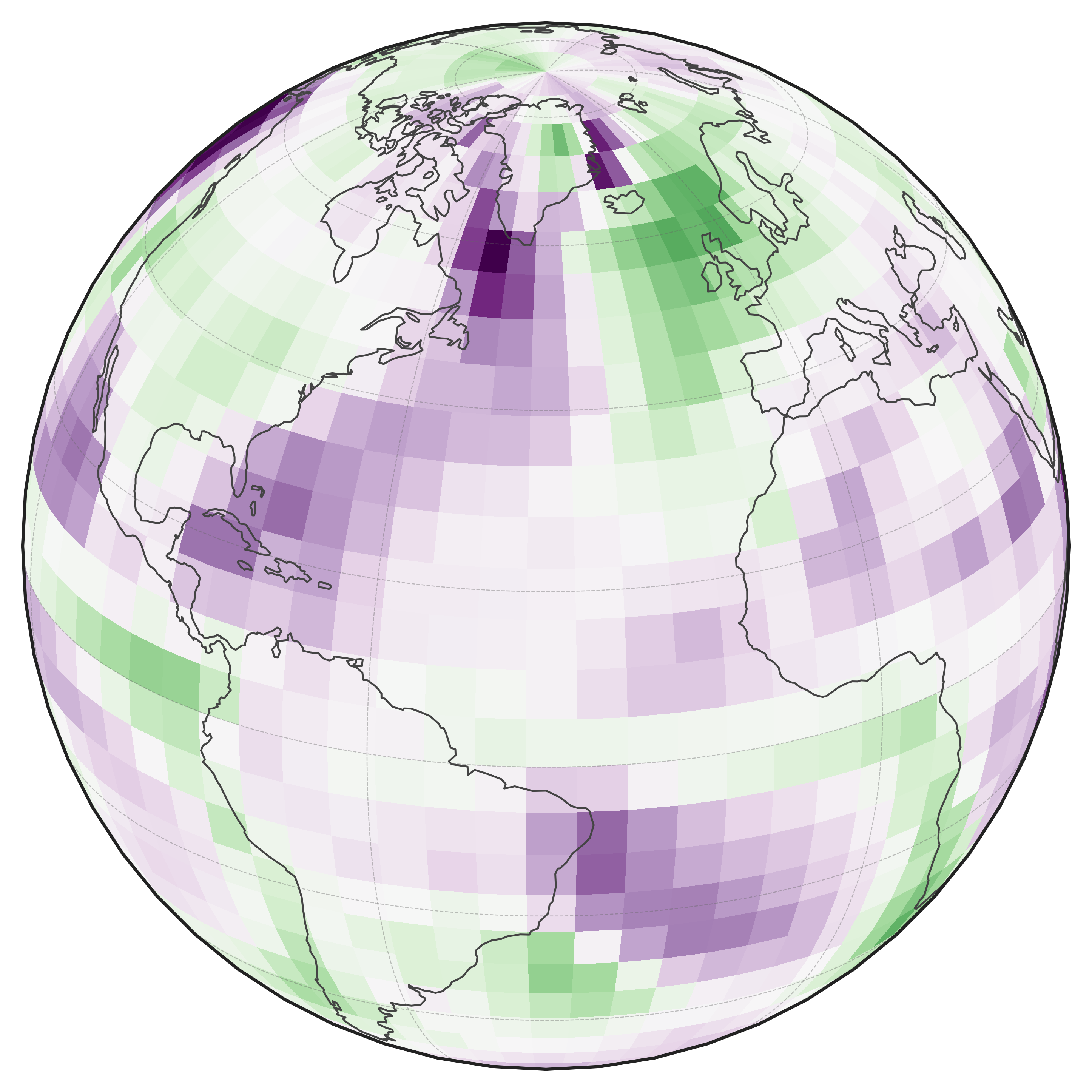} &
        \includegraphics[width=\imgw]{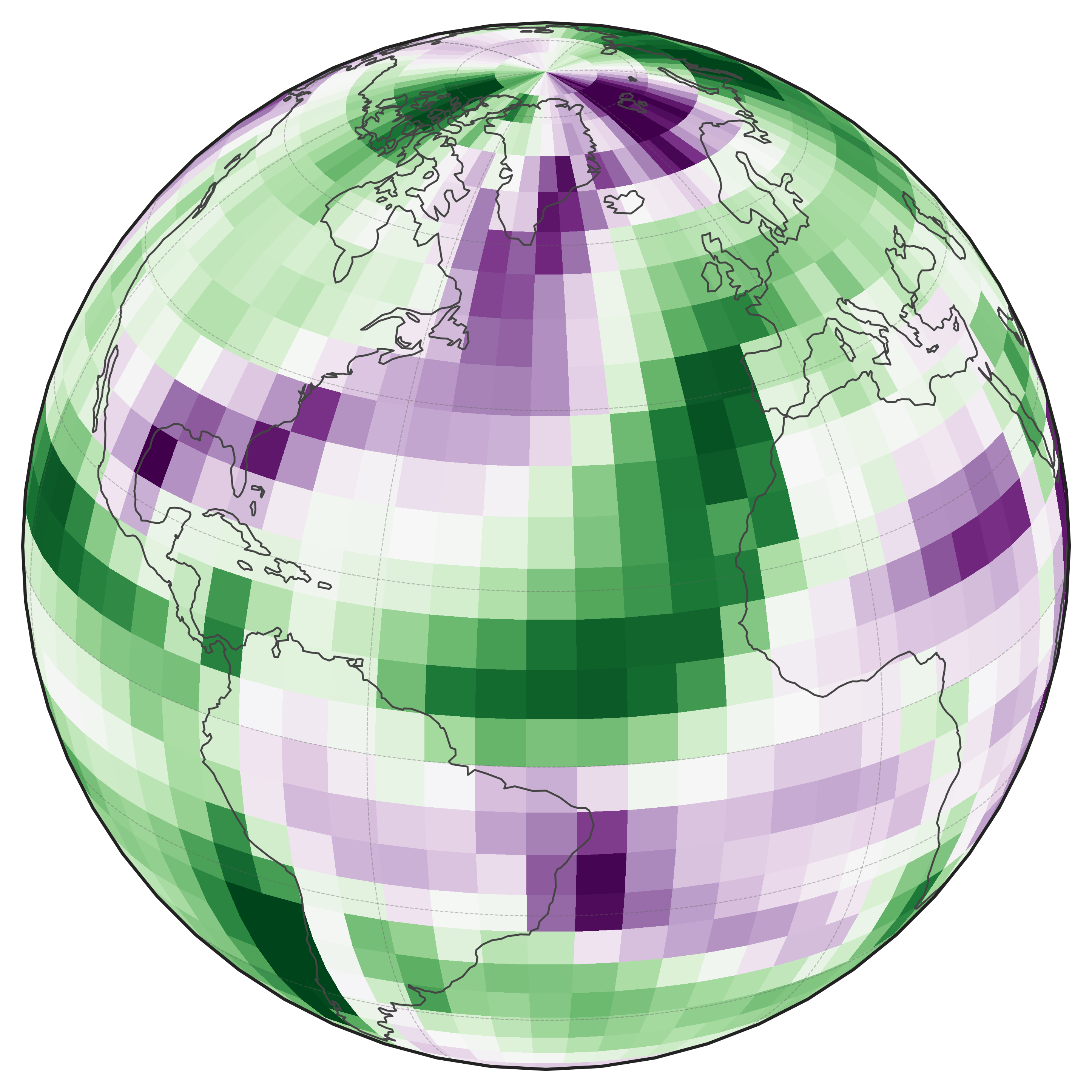}
        \\[-0.35em]
        &
        \multicolumn{3}{c}{
            \includegraphics[width=0.22\textwidth]{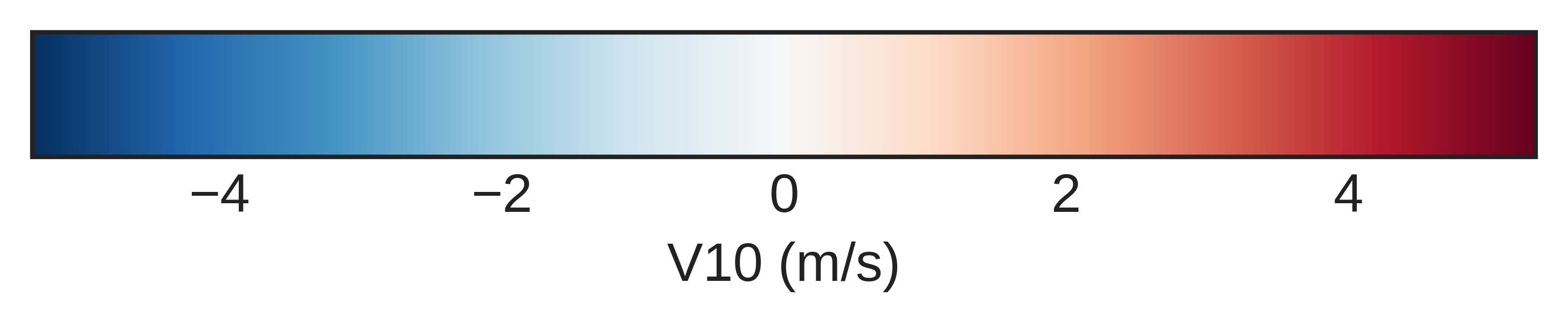}
        }
        &
        \multicolumn{2}{c}{
            \includegraphics[width=0.22\textwidth]{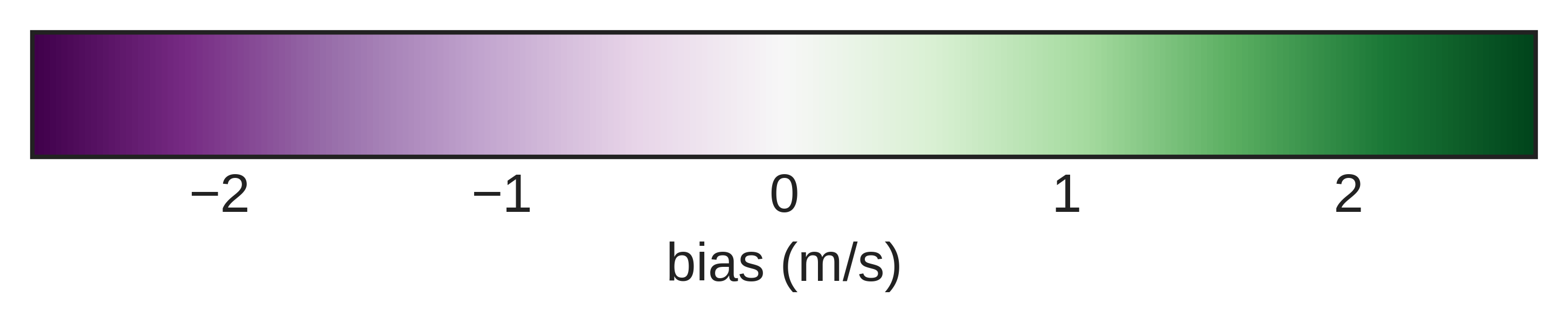}
        }
        \\[0.5em]

        \(T850\) &
        \includegraphics[width=\imgw]{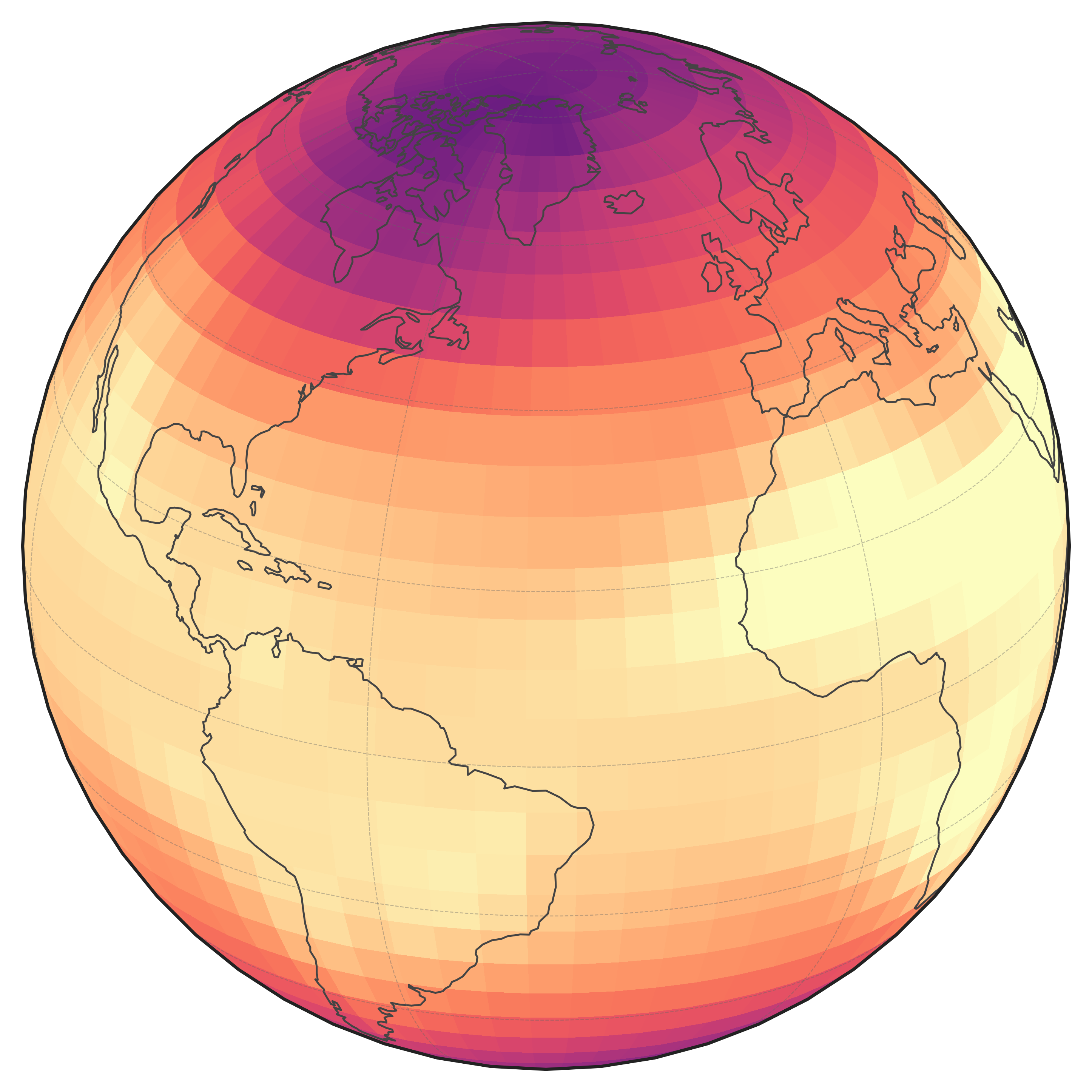} &
        \includegraphics[width=\imgw]{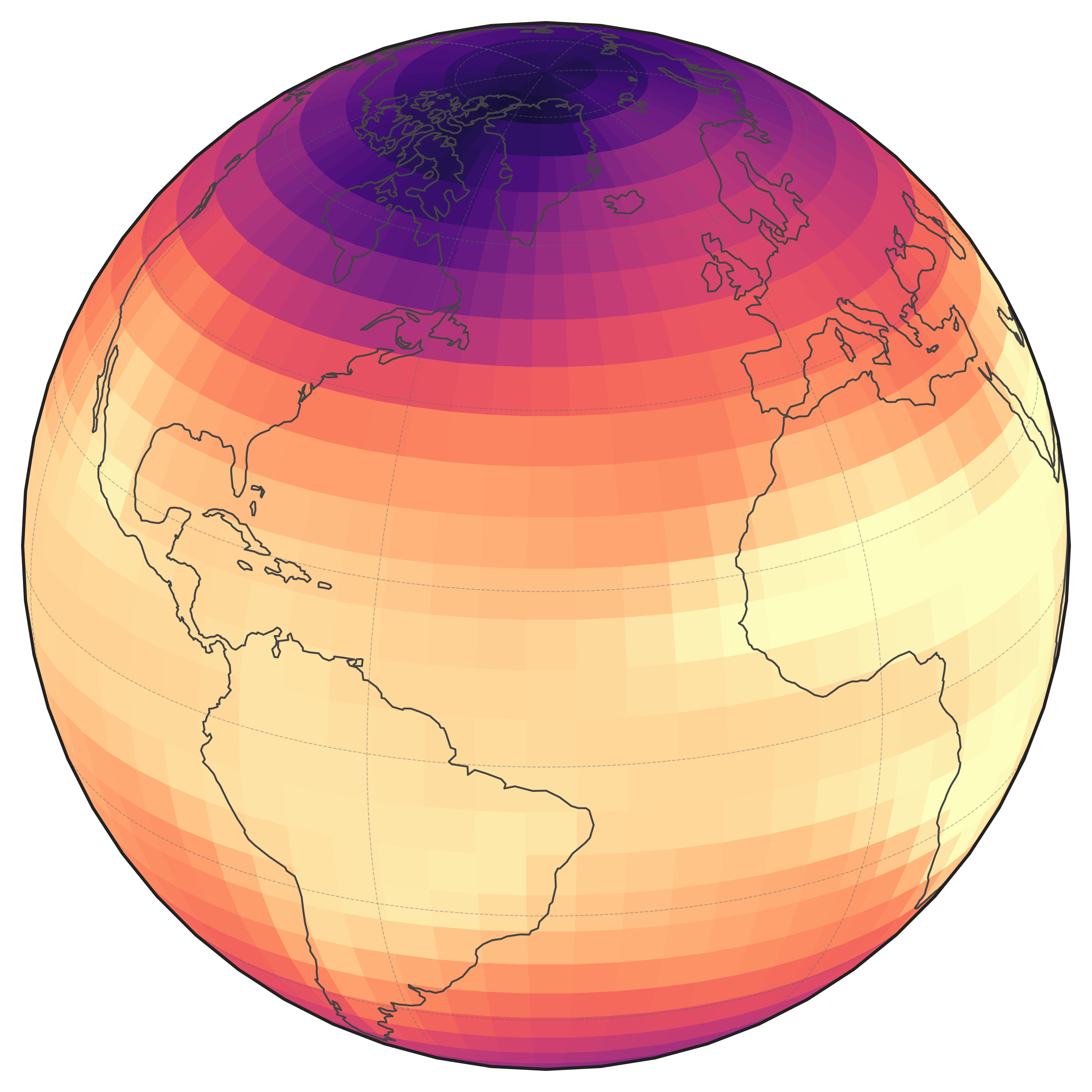} &
        \includegraphics[width=\imgw]{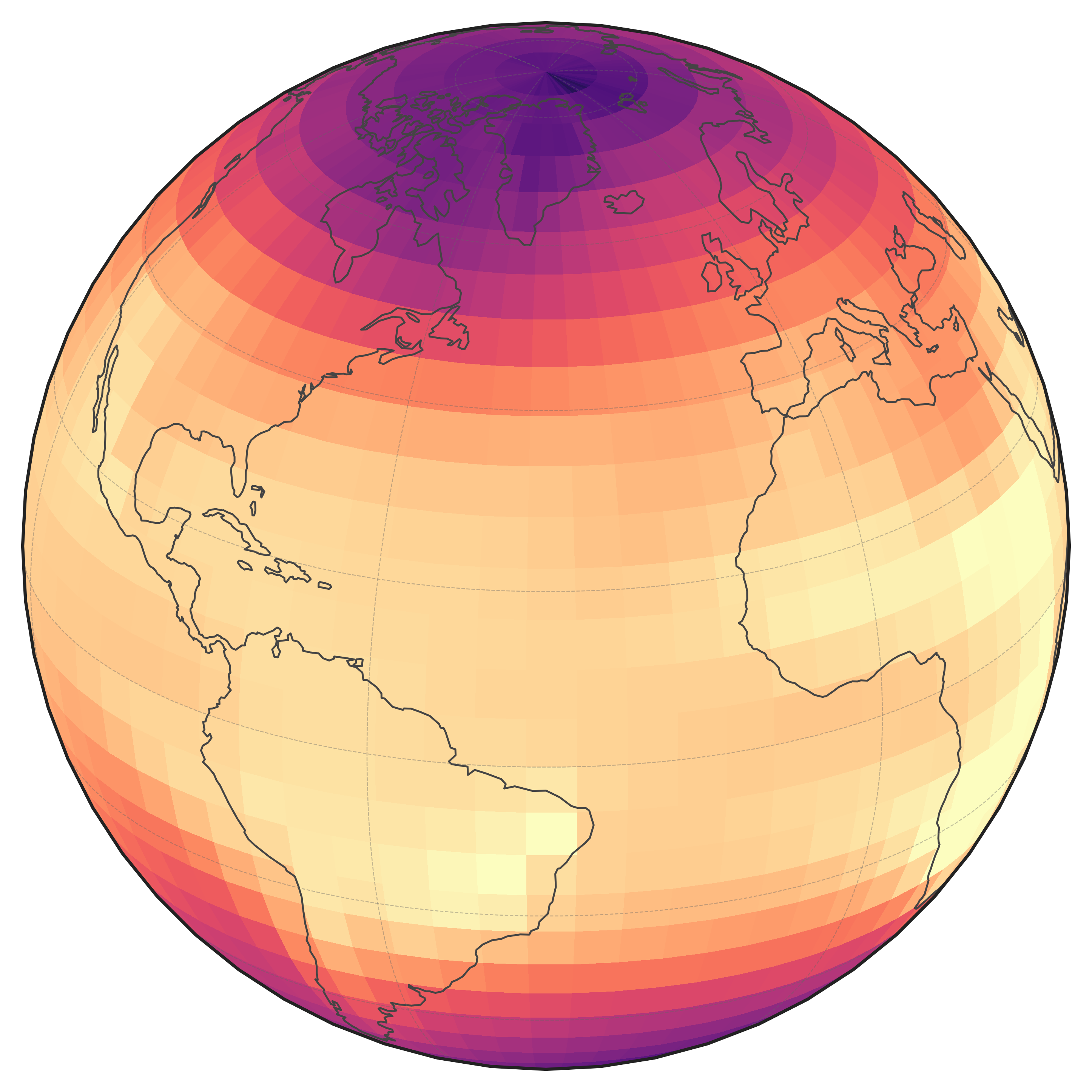} &
        \includegraphics[width=\imgw]{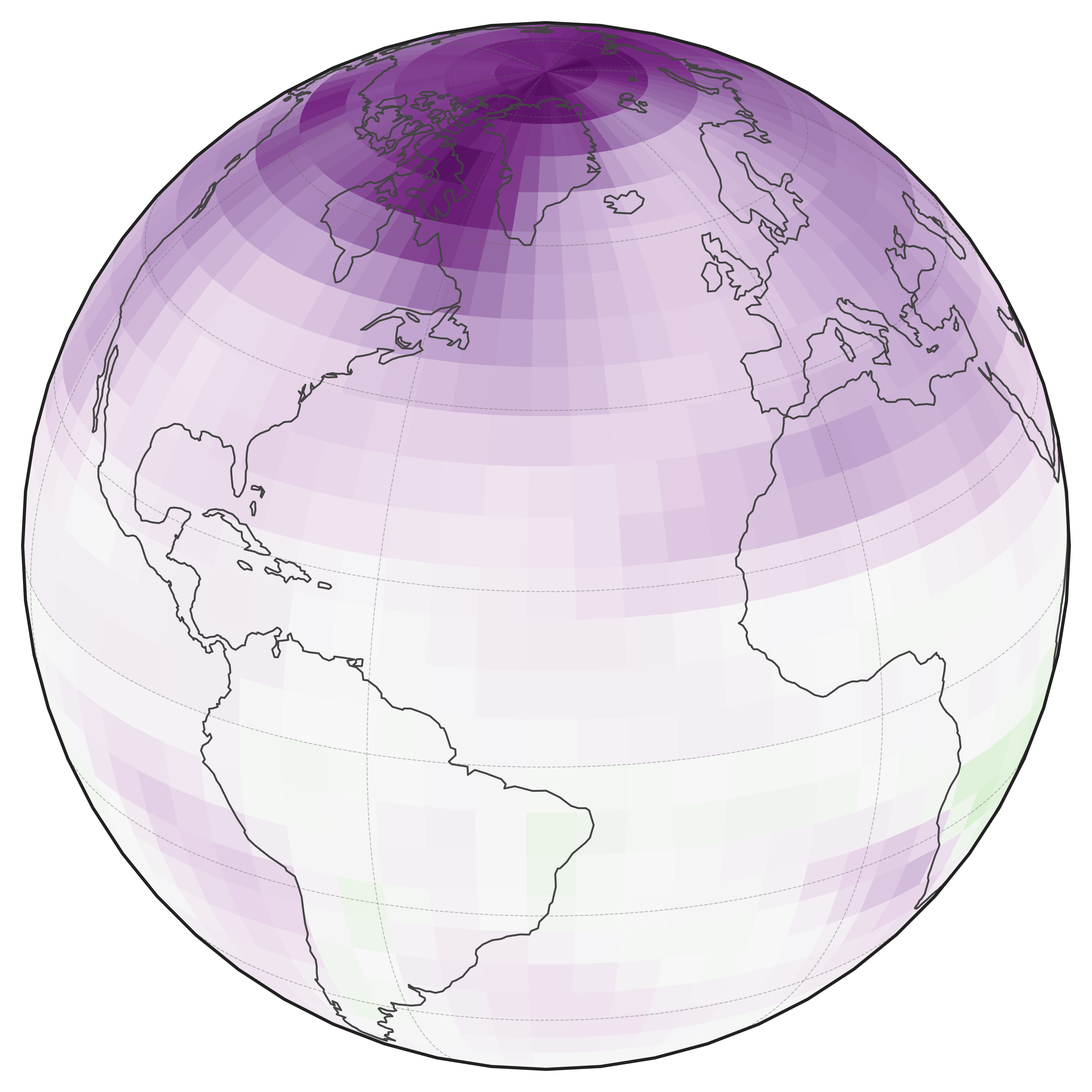} &
        \includegraphics[width=\imgw]{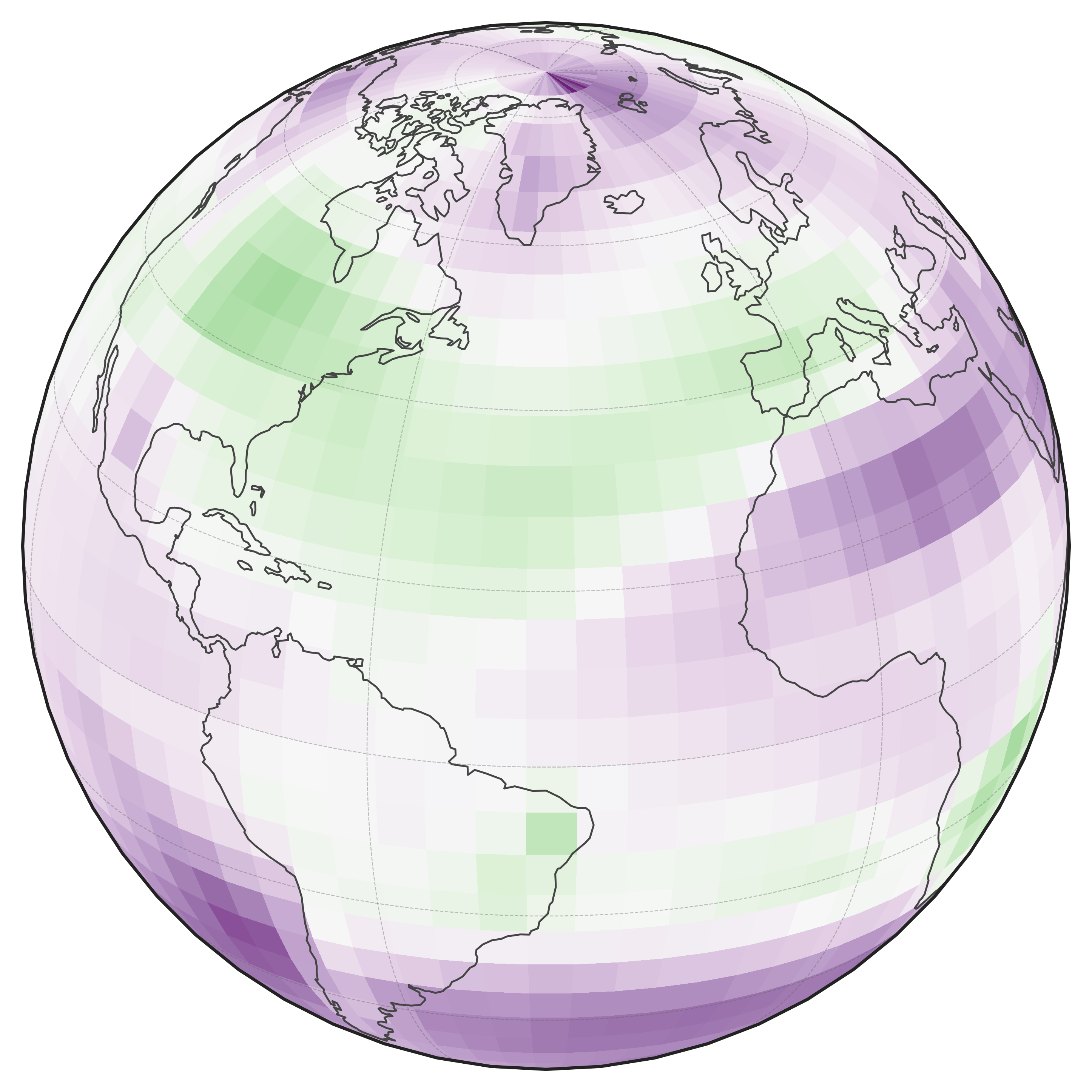}
        \\[-0.35em]
        &
        \multicolumn{3}{c}{
            \includegraphics[width=0.22\textwidth]{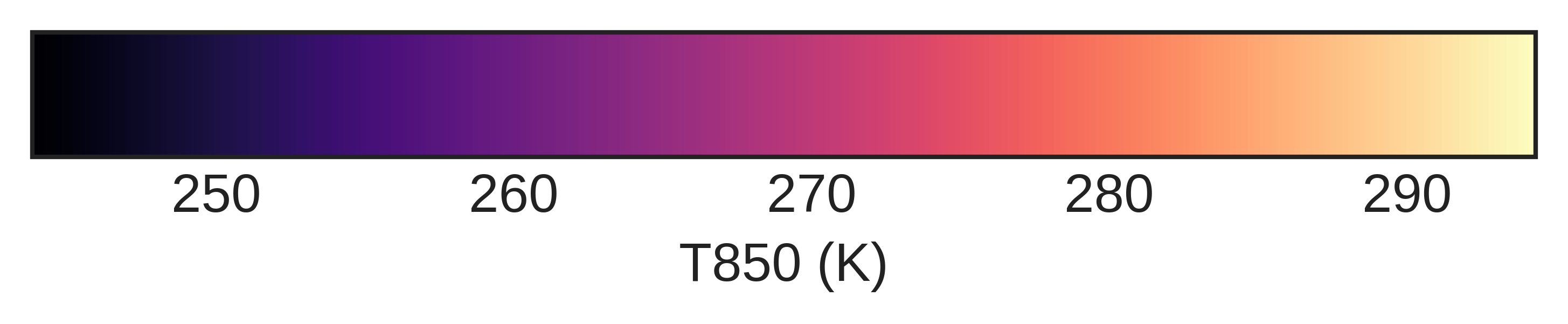}
        }
        &
        \multicolumn{2}{c}{
            \includegraphics[width=0.22\textwidth]{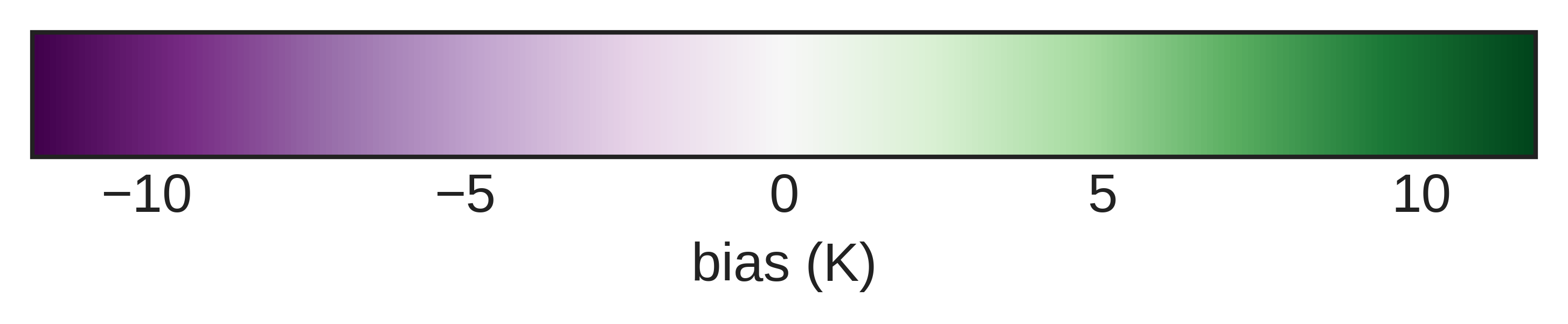}
        }
        \\
    \end{tabular}

    \caption{
    WeatherBench2 500-step time-mean fields. Rows show \(U10\), \(V10\), and \(T850\).
    The first three columns show the ground truth, MFC-resnet, and AFNO time means with a shared
    color scale. The last two columns show the corresponding climatological biases,
    \( \text{MFC} - \text{GT} \) and \( \text{AFNO} - \text{GT} \), with a separate shared bias scale.
    }
    \label{fig:wb2_climatology_globes}
\end{figure}

\subsection{Zero-shot cross-resolution transfer}
\label{sec:exp-crossres}
In this subsection, we test the generalization of an operator trained on one resolution to a finer grid without retraining.


\paragraph{WeatherBench2.} We train three models on 64x32 and evaluate them via 50-step autoregressive rollouts at the training resolution and under zero-shot transfer to 240x121. For MFC at the higher resolution, we apply tapered spectral windowing near the training Nyquist limit. All errors are latitude-weighted; we report the median.

Table~\ref{tab:wb2_crossres} shows that MFC and AFNO are comparable at the training resolution and both improve over FNO. 
Under zero-shot transfer to \(240{\times}121\), however, MFC remains close to its native-resolution accuracy, whereas AFNO degrades substantially, especially on \(T_{850}\). 
This suggests that although AFNO is effective for scalable high-resolution weather modeling, its adaptive Fourier mixing does not by itself guarantee robustness to changes in evaluation resolution. 
MFC gives the lowest high-resolution RMSE for all variables, with degradation ratios close to one. 
Qualitative comparisons are provided in Appendix~\ref{app:wb2_qualitative}.

\begin{table}[h]
\centering
\caption{Cross-resolution zero-shot transfer at step 50, reported as median RMSE over initial conditions.
Models are trained at \(64{\times}32\) and evaluated at both native and \(240{\times}121\) resolution.
The degradation ratio \(\mathrm{RMSE}_{240{\times}121}/\mathrm{RMSE}_{64{\times}32}\) measures sensitivity to resolution change, while the absolute \(240{\times}121\) RMSE is the main high-resolution accuracy metric.}
\label{tab:wb2_crossres}
\small
\begin{tabular}{l rrr rrr rrr}
\toprule
 & \multicolumn{3}{c}{\textbf{MFC-resnet}} & \multicolumn{3}{c}{\textbf{AFNO}} & \multicolumn{3}{c}{\textbf{FNO}} \\
\cmidrule(lr){2-4} \cmidrule(lr){5-7} \cmidrule(lr){8-10}
Variable & \(64{\times}32\) & \(240{\times}121\) & Ratio 
         & \(64{\times}32\) & \(240{\times}121\) & Ratio 
         & \(64{\times}32\) & \(240{\times}121\) & Ratio \\
\midrule
\(U_{10} \) (m/s)  & 4.47 & 4.80 & \textbf{1.07} & 4.50 & 9.40  & 2.09 & 5.54 & 6.06  & 1.09 \\
\(V_{10} \) (m/s)  & 4.66 & 4.93 & 1.06          & 4.71 & 6.06  & 1.29 & 5.53 & 5.60  & \textbf{1.01} \\
\(T_{850} \) (K) & 4.25 & 4.33 & \textbf{1.02} & 4.43 & 13.14 & 2.97 & 7.01 & 10.53 & 1.50 \\
\bottomrule
\end{tabular}
\end{table}

\subsection{Localized structure preservation under distributional shift}
\label{sec:exp-localized}

This subsection addresses our third contribution: directional first-order terms that allow asymmetric, transport-like responses. 
We test whether these terms, together with smooth attenuation, help preserve localized advection-driven structures under two distribution shifts: lower viscosity, which produces sharper gradients, and higher spatial resolution. 
The latter also provides a second test of cross-resolution transfer, corresponding to contribution \label{sec:exp-crossres}.

\paragraph{Burgers' equation.}
We study the 2D viscous Burgers equation
\(\partial_t u + (u \cdot \nabla)u = \nu\Delta u\)
on a periodic domain. 
The training set contains 200 trajectories at \(64{\times}64\) with
\(\nu \in [10^{-2},\,2{\cdot}10^{-2}]\). 
To probe generalization along orthogonal axes, we construct four test sets of 50 trajectories each:
\textbf{(A)} in-distribution;
\textbf{(B)} out-of-distribution (OOD) viscosity, with \(\nu \in [10^{-3},\,7{\cdot}10^{-3}]\), producing sharper structures;
\textbf{(C)} OOD resolution at \(128{\times}128\);
and \textbf{(D)} both shifts combined.

We compare parameter-matched MFC-UNet, FNO, and F-FNO models with approximately \(10^6\) parameters. The stored training arguments of the three evaluated checkpoints agree in every field (optimizer, learning rate and schedule, epochs, batch size, gradient clipping, 5-step unrolled relative-\(\ell_2\) loss, and seed; Table~\ref{tab:burgers-training}), so no rollout stabilization was applied to MFC that was not applied to the baselines. F-FNO serves as a stronger Fourier baseline with factorized spectral mixing.
At test time, we roll out for 100 steps, a \(20{\times}\) extrapolation beyond the training horizon, and report the geometric mean of per-step normalized RMSE following APEBench~\citep{koehler2024apebench}. F-FNO attains the lowest in-distribution error at short horizons (GMean@20 on sets A and B); MFC's advantage is long-horizon stability and OOD/cross-resolution robustness.

The baselines' long-horizon divergence is a spectral instability rather than a tuning artifact (Appendix~\ref{app:divergence-mechanism}): FNO develops a \({\sim}10^4\) amplification hump just inside its retained-mode boundary by step 20, flooding all wavenumbers by step 50; F-FNO fails the same way, delayed to step \({\geq}50\); MFC stays within a factor of \({\approx}40\) of the true spectrum through 100 steps. An FNO retrained at modes 32 (effectively untruncated in the resolved band, \(3{\times}\) the parameters) still diverges on every trajectory of every test set by step 50, while MFC under an FNO-style hard cutoff remains stable on all 200 trajectories. The driver is unconstrained spectral weights, not truncation~\citep{mccabe2306towards,lippe2023pde,koehler2024apebench}.

\begin{table}[h]
\centering
\caption{Aggregated rollout error on 2D Burgers across the four test sets.
We report the geometric mean of normalized RMSE at 20, 50, and 100 autoregressive rollout steps~\citep{koehler2024apebench}; lower is better.}
\label{tab:gmean_ffno}
\setlength{\tabcolsep}{5pt}
\begin{tabular}{c c c c c c c c c c c}
\toprule
& & \multicolumn{3}{c}{GMean@20} 
& \multicolumn{3}{c}{GMean@50} 
& \multicolumn{3}{c}{GMean@100} \\
\cmidrule(lr){3-5}
\cmidrule(lr){6-8}
\cmidrule(lr){9-11}
Test & Res. 
& FNO & FFNO & MFC 
& FNO & FFNO & MFC 
& FNO & FFNO & MFC \\
\midrule
A & $64^2$  
& 0.4153 & \textbf{0.0135} & 0.0841
& 1381.5232 & \textbf{0.1969} & 0.6660
& $5.75{\times}10^9$ & 8.7762 & \textbf{1.7009} \\

B & $64^2$  
& 0.5353 & \textbf{0.0694} & 0.2183
& 1414.6782 & \textbf{0.3758} & 0.8718
& $4.97{\times}10^9$ & 9.7672 & \textbf{1.5673} \\

C & $128^2$ 
& 0.7712 & \textbf{0.2592} & 0.2681
& 2384.3485 & 1.2758 & \textbf{1.1350}
& $9.20{\times}10^9$ & 44.6557 & \textbf{2.3423} \\

D & $128^2$ 
& 0.9035 & \textbf{0.3454} & 0.4090
& 2261.6899 & 1.2914 & \textbf{1.2314}
& $7.15{\times}10^9$ & 39.7726 & \textbf{2.1275} \\
\bottomrule
\end{tabular}
\end{table}

\subsection{Component ablations}
\label{sec:exp-ablations}

We ablate each component of the MFC symbol on 2D Burgers, retraining every variant under the identical recipe, seed, and best-validation rule of Table~\ref{tab:burgers-training}: \textbf{$-$\,shift} freezes \(\delta \equiv 0\); \textbf{$-$\,first-order} freezes \(\gamma_x = \gamma_y \equiv 0\); \textbf{$-$\,anisotropy} freezes \(\alpha_{xy} \equiv 0\); and \textbf{hard cutoff} replaces \(\exp(-\tau\,\bfk^\top\bfA\bfk)\) with \(\mathbf{1}[|\bfk| \leq 17]\), matching FNO's modes-18 truncation at equal bandwidth.

Retraining the full model under an identical seed shifts GMean@20 by \(+52\%\) and GMean@100 by \(-43\%\). We therefore report two full-model draws as a retraining-noise band and treat only \({\geq}2\times\) degradation relative to the band, or divergence, as signal (four-seed study in Appendix~\ref{app:seeds}). No ablated variant diverges.

\begin{table}[h]
\centering
\caption{Component ablations on 2D Burgers. Validation relative \(\ell_2\) is the 5-step unrolled average used for checkpoint selection. GMean@50/@100 and checkpoint provenance: Appendix~\ref{app:ablation-details}.}
\label{tab:ablations}
\small
\setlength{\tabcolsep}{4.5pt}
\begin{tabular}{l l c cccc}
\toprule
& & & \multicolumn{4}{c}{GMean@20} \\
\cmidrule(lr){4-7}
Variant & Modification & Val.\ rel.\ \(\ell_2\) & A & B & C & D \\
\midrule
MFC-full (released) & --- & 0.0113 & 0.084 & 0.218 & 0.268 & 0.409 \\
MFC-full (retrain)  & --- & 0.0131 & 0.128 & 0.279 & 0.332 & 0.486 \\
\midrule
$-$ shift           & \(\delta \equiv 0\) & 0.0144 & 0.214 & 0.304 & 0.359 & 0.461 \\
$-$ first-order     & \(\gamma_x = \gamma_y \equiv 0\) & 0.0127 & 0.070 & 0.173 & 1.012 & 0.998 \\
$-$ anisotropy      & \(\alpha_{xy} \equiv 0\) & 0.0166 & 0.115 & 0.253 & 1.017 & 0.970 \\
hard cutoff         & \(e^{-\tau \bfk^\top \bfA \bfk} \to \mathbf{1}[|\bfk| \leq 17]\) & 0.0720 & 0.195 & 0.270 & 0.522 & 0.535 \\
\bottomrule
\end{tabular}
\end{table}

Smooth attenuation dominates one-step accuracy (hard cutoff: \(5.5\)--\(6.4\times\) validation degradation) but does not cause divergence. The shift is systematically learned (96\% of channels have \(|\delta|\geq 0.1\) cells) and contributes modestly at GMean@20. First-order and anisotropy terms sit inside the noise band in-distribution but degrade zero-shot \(128^2\) transfer by \(2\)--\(3\times\) (GMean@20 on C/D). Learned-parameter inspections: Appendix~\ref{app:learned-params}.

\section{Discussion}
We introduced MFCs, a novel operator learning framework in which convolutional filters are defined in continuous space, drawing on the connection between convolution, Green’s functions, and linear PDEs. By construction, the approach avoids reliance on grid-dependent stencils or spectral truncation, enabling flexibility across discretizations while better preserving structure over long autoregressive roll outs. When augmented with directional components to capture transport-dominated phenomena, MFCs encode a strong inductive bias toward the structure of PDE solutions.

Beyond pointwise improvements in predictions, empirical evaluation indicates that MFCs remain stable over long autoregressive rollouts without resorting to excessive spatial smoothing. This distinction is critical: a model may remain numerically bounded over extended horizons by becoming overly dissipative, yet such behavior often erodes the fine-scale structures that govern future dynamics. In contrast, MFCs faithfully preserve these structures while also robustly capturing sharp localized features. Moreover, they demonstrate reliable zero-shot generalization to higher-resolution grids, consistently outperforming baseline models.

\paragraph{Computational cost.}
Runtime and peak memory versus FNO, F-FNO, and AFNO are in Appendix~\ref{app:runtime}. MFC-UNet is slower than F-FNO at \(64^2\) but faster at \(128^2\), \(256^2\), and \(240{\times}121\); only vanilla FNO is consistently cheaper. Cost grows \(16.9\times\) from \(64^2\) to \(256^2\) (\(16\times\) grid points), consistent with \({\cal O}(cn\log n + c^2 n)\). Symbol evaluation is resolution-independent (\({\approx}2.2\)\,ms).

\paragraph{Beyond periodic rectangular domains.}
The FFT assumption is shared by spectral operator learning. Padding~\citep{li2021fno} and domain extension~\citep{liu2023domain} apply directly. More fundamentally, MFC is the symbol of a parabolic operator: the same construction applies in any eigenbasis of the underlying elliptic operator (spherical harmonics~\citep{bonev2306spherical}, DCT/DST, FEM eigenbases). The spatial kernel is a closed-form shifted anisotropic Gaussian, so it can also be evaluated on scattered points.

\paragraph{Limitations.}
The current implementation targets periodic rectangular grids; the extensions above, and 3D evaluation, are future work. The stability theory is layer-level: Lemma~\ref{app:mfc-stability} bounds the symbol, and Appendix~\ref{app:system-stability} gives a conditional rollout statement whose hypotheses (\(L_\Phi\leq 1\), one-step error bound) are not automatic for the released checkpoints. Observed long-horizon stability is empirical. Pointwise error in chaotic regimes is governed by Lyapunov decorrelation, which is why we report correlation, spectra, and long-time statistics.

\newpage
{
\FloatBarrier
\bibliographystyle{plain}
\bibliography{biblio}

@article{kovachki2023neural,
  title   = {Neural Operator: Learning Maps Between Function Spaces With Applications to {PDEs}},
  author  = {Kovachki, Nikola and Li, Zongyi and Liu, Burigede and Azizzadenesheli, Kamyar and Bhattacharya, Kaushik and Stuart, Andrew and Anandkumar, Anima},
  journal = {Journal of Machine Learning Research},
  volume  = {24},
  number  = {89},
  pages   = {1--97},
  year    = {2023}
}

@article{lu2021learning,
  title   = {Learning nonlinear operators via {DeepONet} based on the universal approximation theorem of operators},
  author  = {Lu, Lu and Jin, Pengzhan and Pang, Guofei and Zhang, Zhongqiang and Karniadakis, George Em},
  journal = {Nature Machine Intelligence},
  volume  = {3},
  number  = {3},
  pages   = {218--229},
  year    = {2021}
}

@inproceedings{li2020fourier,
  title     = {Fourier Neural Operator for Parametric Partial Differential Equations},
  author    = {Li, Zongyi and Kovachki, Nikola and Azizzadenesheli, Kamyar and Liu, Burigede and Bhattacharya, Kaushik and Stuart, Andrew and Anandkumar, Anima},
  booktitle = {International Conference on Learning Representations},
  year      = {2021}
}

@inproceedings{tran2021factorized,
  title     = {Factorized Fourier Neural Operators},
  author    = {Tran, Alasdair and Mathews, Alexander and Xie, Lexing and Ong, Cheng Soon},
  booktitle = {International Conference on Learning Representations},
  year      = {2023}
}

@article{li2023fourier,
  title   = {Fourier Neural Operator with Learned Deformations for {PDEs} on General Geometries},
  author  = {Li, Zongyi and Huang, Daniel Zhengyu and Liu, Burigede and Anandkumar, Anima},
  journal = {Journal of Machine Learning Research},
  volume  = {24},
  number  = {388},
  pages   = {1--26},
  year    = {2023}
}

@article{guibas2021adaptive,
  title   = {Adaptive Fourier Neural Operators: Efficient Token Mixers for Transformers},
  author  = {Guibas, John and Mardani, Morteza and Li, Zongyi and Tao, Andrew and Anandkumar, Anima and Catanzaro, Bryan},
  journal = {arXiv preprint arXiv:2111.13587},
  year    = {2021}
}

@article{gao2025dynamic,
  title   = {Dynamic Schwartz-Fourier Neural Operator for Enhanced Expressive Power},
  author  = {Gao, Wenhan and Luo, Jian and Xu, Ruichen and Liu, Yi},
  journal = {Transactions on Machine Learning Research},
  year    = {2025}
}

@inproceedings{takamoto2022pdebench,
  title     = {{PDEBench}: An Extensive Benchmark for Scientific Machine Learning},
  author    = {Takamoto, Makoto and Praditia, Timothy and Leiteritz, Raphael and MacKinlay, Daniel and Alesiani, Francesco and Pfl{\"u}ger, Dirk and Niepert, Mathias},
  booktitle = {Advances in Neural Information Processing Systems},
  volume    = {35},
  pages     = {1596--1611},
  year      = {2022}
}

@article{petri2015impacts,
  title={Impacts of global warming on residential heating and cooling degree-days in the United States},
  author={Petri, Yana and Caldeira, Ken},
  journal={Scientific reports},
  volume={5},
  number={1},
  pages={12427},
  year={2015},
  publisher={Nature Publishing Group UK London}
}

@inproceedings{ronneberger2015u,
  title={U-net: Convolutional networks for biomedical image segmentation},
  author={Ronneberger, Olaf and Fischer, Philipp and Brox, Thomas},
  booktitle={International Conference on Medical image computing and computer-assisted intervention},
  pages={234--241},
  year={2015},
  organization={Springer}
}

@inproceedings{he2016deep,
  title={Deep residual learning for image recognition},
  author={He, Kaiming and Zhang, Xiangyu and Ren, Shaoqing and Sun, Jian},
  booktitle={Proceedings of the IEEE conference on computer vision and pattern recognition},
  pages={770--778},
  year={2016}
}

@article{li2020multipole,
  title={Multipole graph neural operator for parametric partial differential equations},
  author={Li, Zongyi and Kovachki, Nikola and Azizzadenesheli, Kamyar and Liu, Burigede and Stuart, Andrew and Bhattacharya, Kaushik and Anandkumar, Anima},
  journal={Advances in Neural Information Processing Systems},
  volume={33},
  pages={6755--6766},
  year={2020}
}

@book{EvansPDE,
	Address = {San Francisco},
	Author = {L. C. Evans},
	Publisher = {American Mathematical Society},
	Title = {Partial Differential Equations},
	Year = {1998}}

@book{tos,
	Author = {U. Trottenberg and C. Oosterlee and A. Schuller},
	Publisher = {Academic Press},
	Title = {Multigrid},
	Year = {2001}}

@article{lr,
	Author = {A.I. Liapis and D.W.T. Rippin},
	Journal = {Chem. Eng. Science},
	Pages = {619--627},
	Title = {A general model for the simulation of multi-component adsorption from a finite bath},
	Volume = {32},
	Year = {1977}}

@article{mn,
	Author = {P.C. Meek and J. Norbury},
	Journal = {IMA J. Num. Anal.},
	Pages = {335-356},
	Title = {Two-stage, two-level finite difference schemes for nonlinear parabolic equations},
	Volume = {2},
	Year = {1982}}

@article{ll,
	Author = {Ljung, Stefan and Ljung, Lennart},
	Journal = {Automatica},
	Number = 2,
	Pages = {157--167},
	Title = {Error Propagation Properties of Recursive Least-squares Adaptation Algorithms},
	Volume = 21,
	Year = 1985}

@inproceedings{mccabe2023multiple,
  title={Multiple Physics Pretraining for Physical Surrogate Models},
  author={McCabe, Michael and R{\'e}galdo-Saint Blancard, Bruno and Parker, Liam Holden and Ohana, Ruben and Cranmer, Miles and Bietti, Alberto and Eickenberg, Michael and Golkar, Siavash and Krawezik, Geraud and Lanusse, Francois and others},
  booktitle={NeurIPS 2023 AI for Science Workshop},
  year={2023}
}

@article{li2020neural,
  title={Neural operator: Graph kernel network for partial differential equations},
  author={Li, Zongyi and Kovachki, Nikola and Azizzadenesheli, Kamyar and Liu, Burigede and Bhattacharya, Kaushik and Stuart, Andrew and Anandkumar, Anima},
  journal={arXiv preprint arXiv:2003.03485},
  year={2020}
}

@article{li2021fno,
  title={Fourier Neural Operator for Parametric Partial Differential Equations},
  author={Li, Zongyi and Kovachki, Nikola and Azizzadenesheli, Kamyar and Liu, Burigede and Bhattacharya, Kaushik and Stuart, Andrew and Anandkumar, Anima},
  journal={ICLR},
  year={2021}
}

@misc{kovachki2023data,
  title={Learning Dissipative Dynamics in Chaotic Systems Dataset},
  author={Kovachki, Nikola and others},
  year={2023},
  howpublished={Zenodo},
  note={\url{https://zenodo.org/records/7495555}}
}

@article{rasp2024weatherbench,
  title={Weatherbench 2: A benchmark for the next generation of data-driven global weather models},
  author={Rasp, Stephan and Hoyer, Stephan and Merose, Alexander and Langmore, Ian and Battaglia, Peter and Russell, Tyler and Sanchez-Gonzalez, Alvaro and Yang, Vivian and Carver, Rob and Agrawal, Shreya and others},
  journal={Journal of Advances in Modeling Earth Systems},
  volume={16},
  number={6},
  pages={e2023MS004019},
  year={2024},
  publisher={Wiley Online Library}
}

@article{koehler2024apebench,
  title={Apebench: A benchmark for autoregressive neural emulators of pdes},
  author={Koehler, Felix and Niedermayr, Simon and Westermann, R{\"u}diger and Thuerey, Nils},
  journal={Advances in Neural Information Processing Systems},
  volume={37},
  pages={120252--120310},
  year={2024}
}

@article{lippe2023pde,
  title={Pde-refiner: Achieving accurate long rollouts with neural pde solvers},
  author={Lippe, Phillip and Veeling, Bas and Perdikaris, Paris and Turner, Richard and Brandstetter, Johannes},
  journal={Advances in Neural Information Processing Systems},
  volume={36},
  pages={67398--67433},
  year={2023}
}

@article{zhou2024masked,
  title={Masked autoencoders are PDE learners},
  author={Zhou, Anthony and Farimani, Amir Barati},
  journal={arXiv preprint arXiv:2403.17728},
  year={2024}
}

@article{pathak2022fourcastnet,
  title={Fourcastnet: A global data-driven high-resolution weather model using adaptive fourier neural operators},
  author={Pathak, Jaideep and Subramanian, Shashank and Harrington, Peter and Raja, Sanjeev and Chattopadhyay, Ashesh and Mardani, Morteza and Kurth, Thorsten and Hall, David and Li, Zongyi and Azizzadenesheli, Kamyar and others},
  journal={arXiv preprint arXiv:2202.11214},
  year={2022}
}

@article{cao2021choose,
  title={Choose a transformer: Fourier or galerkin},
  author={Cao, Shuhao},
  journal={Advances in neural information processing systems},
  volume={34},
  pages={24924--24940},
  year={2021}
}

@article{rahman2022u,
  title={U-no: U-shaped neural operators},
  author={Rahman, Md Ashiqur and Ross, Zachary E and Azizzadenesheli, Kamyar},
  journal={arXiv preprint arXiv:2204.11127},
  year={2022}
}

@article{wen2022u,
  title={U-FNO—An enhanced Fourier neural operator-based deep-learning model for multiphase flow},
  author={Wen, Gege and Li, Zongyi and Azizzadenesheli, Kamyar and Anandkumar, Anima and Benson, Sally M},
  journal={Advances in Water Resources},
  volume={163},
  pages={104180},
  year={2022},
  publisher={Elsevier}
}

@article{kossaifi2025librarylearningneuraloperators,
   author    = {Jean Kossaifi and
                  Nikola Kovachki and
                  Zongyi Li and
                  David Pitt and
                  Miguel Liu-Schiaffini and
                  Robert Joseph George and
                  Boris Bonev and
                  Kamyar Azizzadenesheli and
                  Julius Berner and
                  Valentin Duruisseaux and
                  Anima Anandkumar},
   title     = {A Library for Learning Neural Operators},
   journal   = {arXiv preprint arXiv:2412.10354},
   year      = {2025},
}

@article{duruisseaux2025guide,
   author    = {Valentin Duruisseaux and
                  Jean Kossaifi and
                  Anima Anandkumar},
   title     = {Fourier Neural Operators Explained: A Practical Perspective},
   journal   = {arXiv preprint arXiv:2512.01421},
   year      = {2025},
}

@article{raonic2023convolutional,
  title={Convolutional neural operators for robust and accurate learning of pdes},
  author={Raonic, Bogdan and Molinaro, Roberto and De Ryck, Tim and Rohner, Tobias and Bartolucci, Francesca and Alaifari, Rima and Mishra, Siddhartha and De B{\'e}zenac, Emmanuel},
  journal={Advances in Neural Information Processing Systems},
  volume={36},
  pages={77187--77200},
  year={2023}
}

@article{smets2023pde,
  title={PDE-based group equivariant convolutional neural networks},
  author={Smets, Bart MN and Portegies, Jim and Bekkers, Erik J and Duits, Remco},
  journal={Journal of Mathematical Imaging and Vision},
  volume={65},
  number={1},
  pages={209--239},
  year={2023},
  publisher={Springer}
}

@inproceedings{long2018pde,
  title={Pde-net: Learning pdes from data},
  author={Long, Zichao and Lu, Yiping and Ma, Xianzhong and Dong, Bin},
  booktitle={International conference on machine learning},
  pages={3208--3216},
  year={2018},
  organization={PMLR}
}

@book{folland1995introduction,
  title={Introduction to partial differential equations},
  author={Folland, Gerald B},
  volume={102},
  year={1995},
  publisher={Princeton university press}
}

@article{mcgreivy2024weak,
  title={Weak baselines and reporting biases lead to overoptimism in machine learning for fluid-related partial differential equations},
  author={McGreivy, Nick and Hakim, Ammar},
  journal={Nature machine intelligence},
  volume={6},
  number={10},
  pages={1256--1269},
  year={2024},
  publisher={Nature Publishing Group UK London}
}

@article{mccabe2306towards,
  title={Towards stability of autoregressive neural operators, 2023},
  author={McCabe, Michael and Harrington, Peter and Subramanian, Shashank and Brown, Jed},
  journal={URL https://arxiv. org/abs/2306.10619}
}

@article{liu2023domain,
  title={Domain agnostic fourier neural operators},
  author={Liu, Ning and Jafarzadeh, Siavash and Yu, Yue},
  journal={Advances in neural information processing systems},
  volume={36},
  pages={47438--47450},
  year={2023}
}

@article{bonev2306spherical,
  title={Spherical fourier neural operators: Learning stable dynamics on the sphere, 2023},
  author={Bonev, Boris and Kurth, Thorsten and Hundt, Christian and Pathak, Jaideep and Baust, Maximilian and Kashinath, Karthik and Anandkumar, Anima},
  journal={URL https://arxiv. org/abs/2306.03838}
}
}


\newpage
\appendix

\section{Background on the solution of Parabolic Equations}
\label{app:background}

The convolution we use is defined by the solution of the parabolic PDE
\begin{eqnarray}
    \label{conv}
    v_t = {\cal L} v, \quad \quad v(0,\bfx) = \bfR u,
\end{eqnarray}
equipped with periodic boundary conditions
and on the bounded domain $[0,1]^2$.
We now review the basic principles of parabolic PDEs~\citep{EvansPDE} to shed light on this choice.

The operator ${\cal L}$ that
has the form
\begin{eqnarray}
    \label{soPDE}
{\cal L}v =     \alpha_{xx} {\frac {\partial^2 v}{\partial x^2}} + 2\alpha_{xy}{\frac {\partial^2 v}{\partial x \partial y}} + \alpha_{yy} {\frac {\partial^2 v}{\partial y^2}}.
\end{eqnarray}
We say that the operator is elliptic if
\begin{eqnarray}
\label{condel}
  \alpha_{xx}\alpha_{yy} - \alpha_{xy}^2 >0.
\end{eqnarray}
The ellipticity of the operator guarantees that it has an inverse (up to a constant) and that the parabolic  \Eqref{soPDE} is well-posed.

Using the separation of variables, the solution of \Eqref{conv} can be written in the following form
$$v(t,x,y) = \sum_{mn} v_{mn}(t) \phi_{mn}(x,y),$$
where $\phi_{mn}(x,y)$ are spatial eigenfunctions given by 
$$\phi_{mn}(x,y) = \exp(2\pi i (mx + ny)).$$
The associated temporal coefficient is given by 
$$v_{mn}(t) = \exp(-\lambda_{mn}t) v_{mn}(0),$$
where
$$\lambda_{mn} = 4\pi^2\left( \alpha_{xx}m^2 + \alpha_{yy}n^2 + 2\alpha_{xy}nm  \right)$$
controls the decay of the associated eigensolution. Note that \Eqref{condel} ensures that $\lambda_{mn} > 0$, guaranteeing the stability of the solution. Additionally, since larger $m,n$ yield eigenfunctions with higher spatial frequency as well as larger $\lambda_{mn}$, high frequency components in the solution decay faster.

\section{Equivalence between regular convolutions and discretized MFCs on a fixed mesh}
\label{app:mfc_cnn_equivalence}

\begin{lemma}
\label{equivalence_lemma}
Let the grid have \(n \times n\) points with spacing \(h = 1/n\), and let \(\bfK\) be any \(k \times k\)-stencil convolution with weights \(w_{ij}\) on the offset set \(S\). For every \(\epsilon > 0\) there exists a discretized MFC \(\bfK_{\bftheta}\) such that
\[
    \|\bfK_{\bftheta} - \bfK\|_2 \le \epsilon
\]
in the operator norm (a \(\bfu\)-independent additive bound cannot hold for all inputs).
\end{lemma}

\begin{proof}
Take MFC channels indexed by \((i,j)\in S\), each with \(\gamma_0=1\), \(\gamma_x=\gamma_y=0\), shift \(\delta_{ij}=(ih,jh)\), and any fixed SPD \(\bfA\); mix them with a \(1{\times}1\) layer carrying \(w_{ij}\).

\emph{Step 1 (exact shifts at \(\tau=0\)).} On the discrete frequency grid, \(e^{-2\pi i(k_x i+k_y j)/n}\) is the DFT symbol of the integer shift \((i,j)\). The mixed output equals \(\bfK\bfu\) exactly.

\emph{Step 2 (rate for \(\tau>0\)).} Each channel differs from the pure shift only by \(e^{-\tau\bfk^\top\bfA\bfk}\), and \(\lvert e^{-\tau\bfk^\top\bfA\bfk}-1\rvert\le\tau\lambda_{\max}(\bfA)\,\bar K^2\) uniformly on the grid. The DFT is unitary, so
\[
    \|\bfK_{\bftheta}-\bfK\|_2
    \le
    \tau\,\lambda_{\max}(\bfA)\,\bar K^2\,\|w\|_1.
\]
Choosing \(\tau\le\epsilon/(\lambda_{\max}(\bfA)\,\bar K^2\,\|w\|_1)\) gives an explicit first-order rate. Taking \(k=n\) covers every periodic discrete convolution on the grid.
\end{proof}

\section{Stability of the MFC symbol}

\begin{lemma}[Stability of the MFC symbol]
\label{app:mfc-stability}
Let \(\tau>0\) and \(\bfA\succ 0\).
Then the MFC symbol \(G_{\bfxi}\) and its parameter
derivatives are bounded by a polynomial times a Gaussian,
and hence are integrable. In particular, the Fourier-space
functional \(\ell^*(\bfxi)\) defined below and its gradient
\(\nabla_{\bfxi}\ell^*\) are finite for integrable inputs.
\end{lemma}

\begin{proof}
Recall that the MFC Fourier symbol is
\begin{equation}
    \label{eq:mfc_symbol_proof}
    G_{\bfxi}(\bfk)
    =
    \left(
    \gamma_0+i\gamma_x k_x+i\gamma_y k_y
    \right)
    \exp\left(
    -\tau \bfk^\top \bfA \bfk
    \right)
    \exp\left(-2\pi i\,\bfk\cdot\boldsymbol{\delta}\right),
\end{equation}
where
\begin{equation}
    \bfA
    =
    \begin{pmatrix}
        \alpha_{xx} & \alpha_{xy}\\
        \alpha_{xy} & \alpha_{yy}
    \end{pmatrix}.
\end{equation}

Since \(\alpha_{xx}>0\), \(\alpha_{yy}>0\), and
\(\alpha_{xx}\alpha_{yy}-\alpha_{xy}^2>0\),
the matrix \(\bfA\) is symmetric positive definite.
Hence,
\begin{equation}
    \bfk^\top \bfA \bfk
    \ge
    \lambda_{\min}(\bfA)\|\bfk\|_2^2,
\end{equation}
where \(\lambda_{\min}(\bfA)>0\) denotes the smallest
eigenvalue of \(\bfA\).

Also, by the triangle inequality and Cauchy--Schwarz,
\begin{equation}
\begin{aligned}
    \left|
    \gamma_0+i\gamma_x k_x+i\gamma_y k_y
    \right|
    &\le
    |\gamma_0|+|\gamma_x||k_x|+|\gamma_y||k_y|\\
    &\le
    |\gamma_0|+\|\bfgamma\|_2\|\bfk\|_2\\
    &\le
    C(1+\|\bfk\|_2),
\end{aligned}
\end{equation}
where \(\bfgamma=(\gamma_x,\gamma_y)^\top\)
and \(C>0\) is a constant.

Furthermore,
\begin{equation}
    \left|
    \exp\left(-2\pi i\,\bfk\cdot\boldsymbol{\delta}\right)
    \right|=1.
\end{equation}

Therefore,
\begin{equation}
    |G_{\bfxi}(\bfk)|
    \le
    C(1+\|\bfk\|_2)
    \exp\left(
    -\tau\lambda_{\min}(\bfA)\|\bfk\|_2^2
    \right).
\end{equation}

Since Gaussian decay dominates polynomial growth,
the right-hand side is integrable.

Consider the Fourier-space functional
\begin{equation}
    \ell^*(\bfxi)
    =
    \int_{\mathbb R^2}
    G_{\bfxi}(\bfk)\,
    \widehat u(\bfk)\,
    \overline{\widehat y(\bfk)}
    \,d\bfk.
\end{equation}

For \(u,y\in L^1(\mathbb R^2)\), their Fourier transforms
are bounded. Since \(G_{\bfxi}\) is integrable,
it follows that \(\ell^*(\bfxi)\) is finite.

The same argument applies to the gradient.
Differentiating \(G_{\bfxi}\) with respect to any
parameter in \(\bfxi\) introduces only polynomial
factors in \(\bfk\). In particular, differentiating
the shift parameter contributes a factor proportional
to \(k_x\) or \(k_y\), while differentiating the
attenuation term contributes quadratic factors.
Thus, for each parameter \(\xi_j\),
\begin{equation}
    |\partial_{\xi_j}G_{\bfxi}(\bfk)|
    \le
    C_j(1+\|\bfk\|_2^{m_j})
    \exp\left(
    -\tau\lambda_{\min}(\bfA)\|\bfk\|_2^2
    \right),
\end{equation}
for some constant \(C_j>0\) and integer \(m_j\ge 0\).
This is again integrable, so each component of
\(\nabla_{\bfxi}\ell^*\) is finite.

In the multi-channel case, the operator is applied
componentwise in Fourier space via
\begin{equation}
    \widehat v_j(\bfk)
    =
    \sum_{i=1}^{c_{\rm in}}
    G_{ij}(\bfk)\,\widehat u_i(\bfk),
    \qquad
    j=1,\ldots,c_{\rm out}.
\end{equation}

Since each entry \(G_{ij}(\bfk)\) satisfies the
same bound, the same conclusion holds channelwise.
\end{proof}

In closed form, for \(\tau>0\) and \(\bfA\succ0\),
we also obtain the uniform-in-frequency bound
\[
    \sup_{\bfk}|G_{\bfxi}(\bfk)|
    \le
    \left(
    \gamma_0^2+
    \frac{\|\bfgamma\|_2^2}
    {2e\tau\lambda_{\min}(\bfA)}
    \right)^{1/2},
\]
since
\[
    \sup_{r\ge0}r^2e^{-2ar^2}
    =\frac{1}{2ea},\qquad a>0.
\]
This highlights the role of \(\tau>0\):
without Gaussian attenuation, the first-order
readout can be unbounded when
\(\bfgamma\neq0\).

\paragraph{Conditional rollout stability.}
\label{app:system-stability}
With 1-Lipschitz activations, residual blocks compose to a certified one-step Lipschitz constant \(L_\Phi\) from the layer bounds above and the spectral norms of the mixing matrices.

\begin{theorem}[Conditional rollout stability]
\label{thm:rollout-stability}
If \(\Phi\) is \(L_\Phi\)-Lipschitz on a set containing the exact trajectory and the one-step error along that trajectory is at most \(\varepsilon\), then
\[
    e_n \le L_\Phi^n e_0 + \varepsilon\sum_{j<n} L_\Phi^j,
\]
which is uniformly bounded if \(L_\Phi<1\) and at most linear if \(L_\Phi=1\).
\end{theorem}

The released checkpoints do not satisfy \(L_\Phi\le 1\). For the Burgers MFC-UNet the certified product bound is \({\sim}10^{59}\) (LayerNorm worst-case \(1/\sqrt{\epsilon_{\rm LN}}\) and unconstrained \(1{\times}1\) norms); the empirical local Jacobian norm on test-A states is median \(4.7\); realized per-step error expansion is median \(1.005\times\). Observed stability is therefore on-distribution, not an architectural contraction. \(L_\Phi\le 1\) is enforceable by residual rescaling; that training variant is future work.

\section{Component ablations: full results}
\label{app:ablation-details}

\begin{table}[h]
\centering
\caption{Component ablations on 2D Burgers at all horizons (test sets A/B/C/D).}
\label{tab:ablations-full}
\footnotesize
\setlength{\tabcolsep}{3pt}
\begin{tabular}{l cccc cccc cccc}
\toprule
& \multicolumn{4}{c}{GMean@20} & \multicolumn{4}{c}{GMean@50} & \multicolumn{4}{c}{GMean@100} \\
\cmidrule(lr){2-5} \cmidrule(lr){6-9} \cmidrule(lr){10-13}
Variant & A & B & C & D & A & B & C & D & A & B & C & D \\
\midrule
MFC-full (released) & 0.084 & 0.218 & 0.268 & 0.409 & 0.666 & 0.872 & 1.135 & 1.231 & 1.701 & 1.567 & 2.342 & 2.128 \\
MFC-full (retrain)  & 0.128 & 0.279 & 0.332 & 0.486 & 0.621 & 0.806 & 1.133 & 1.195 & 0.972 & 0.998 & 1.359 & 1.287 \\
$-$ shift           & 0.214 & 0.304 & 0.359 & 0.461 & 0.834 & 0.852 & 1.075 & 1.086 & 1.512 & 1.266 & 1.833 & 1.623 \\
$-$ first-order     & 0.070 & 0.173 & 1.012 & 0.998 & 0.459 & 0.601 & 1.149 & 1.099 & 1.139 & 1.074 & 1.311 & 1.200 \\
$-$ anisotropy      & 0.115 & 0.253 & 1.017 & 0.970 & 0.764 & 0.934 & 1.618 & 1.433 & 1.933 & 1.705 & 2.228 & 1.847 \\
hard cutoff         & 0.195 & 0.270 & 0.522 & 0.535 & 0.771 & 0.785 & 1.312 & 1.196 & 1.411 & 1.189 & 2.120 & 1.788 \\
\bottomrule
\end{tabular}
\end{table}

The $-$\,first-order run was interrupted at 271/300 by a disk fault after its best-validation epoch (131); the selected checkpoint is unaffected.

\section{Learned-parameter inspection}
\label{app:learned-params}

Released Burgers MFC-UNet (14 layers, 2{,}288 channels): 96\% of channels have \(|\delta|\ge 0.1\) cells (median \(0.49\)); median \(|\gamma_{x,y}|/|\gamma_0|=1.11\); median axis ratio \(a/b=3.11\). Navier--Stokes MFC-UNet (2{,}464 channels): shift median \(0.46\) cells; first-order medians \(\approx 0.12\); median \(a/b=1.80\). Figure~\ref{fig:learned-params}.

\begin{figure}[h]
    \centering
    \begin{minipage}{0.49\textwidth}
        \includegraphics[width=\linewidth]{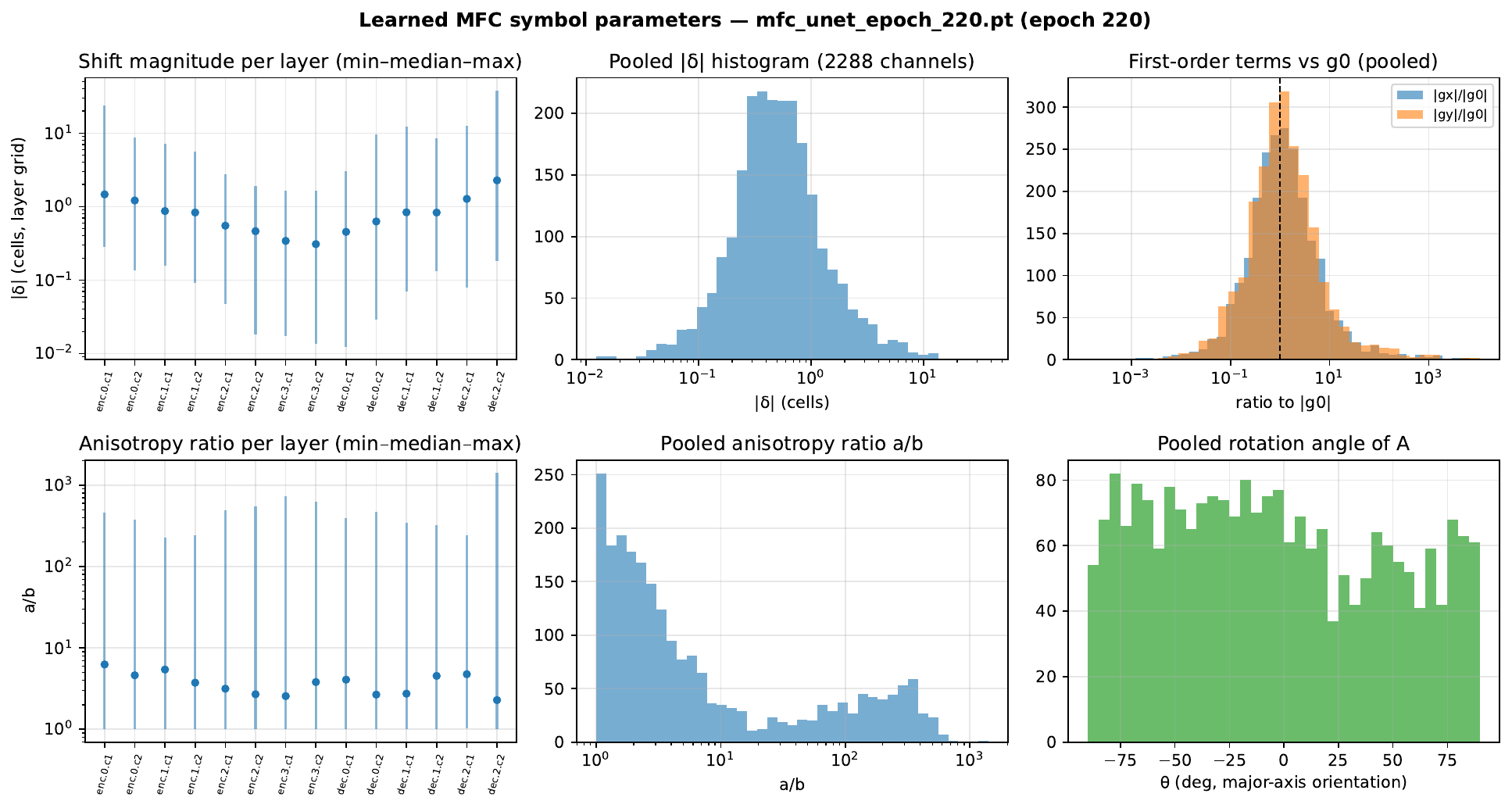}
    \end{minipage}
    \hfill
    \begin{minipage}{0.49\textwidth}
        \includegraphics[width=\linewidth]{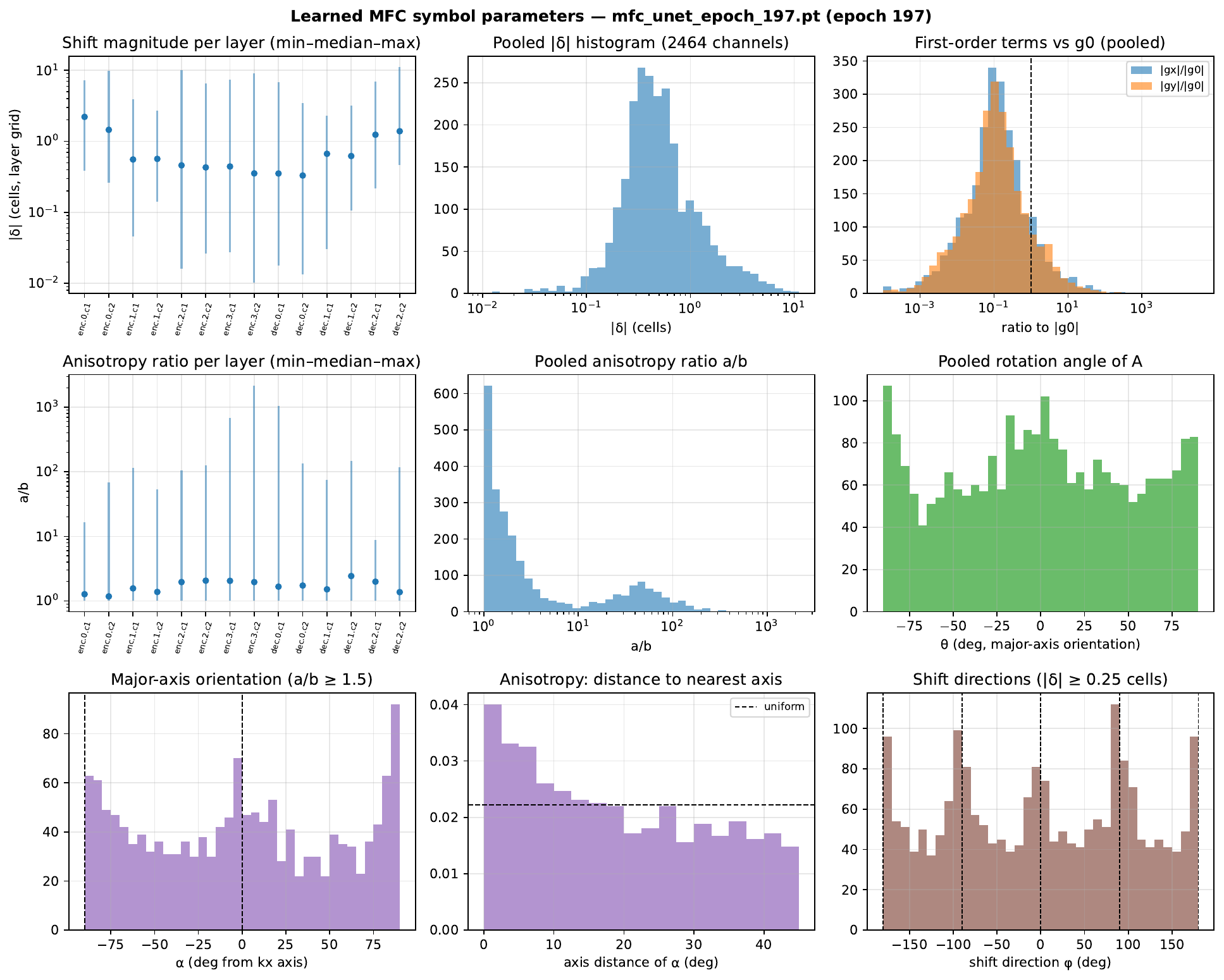}
    \end{minipage}
    \caption{Learned MFC parameters: Burgers (left) and Navier--Stokes (right).}
    \label{fig:learned-params}
\end{figure}

\section{Mechanism of the baseline divergence on Burgers}
\label{app:divergence-mechanism}

Figure~\ref{fig:spectral-diagnostic} shows \(E_{\rm model}(k,t)/E_{\rm GT}(k,t)\) on test A. FNO peaks just inside \(k=18\) by \(t=20\) and floods the band by \(t=50\) (\(2.6\times 10^{51}\) at \(t=100\)). F-FNO fails later (max \(2.0\times 10^{11}\)). MFC stays within \(\approx 40\) of the true spectrum.

\begin{figure}[h]
    \centering
    \includegraphics[width=\linewidth]{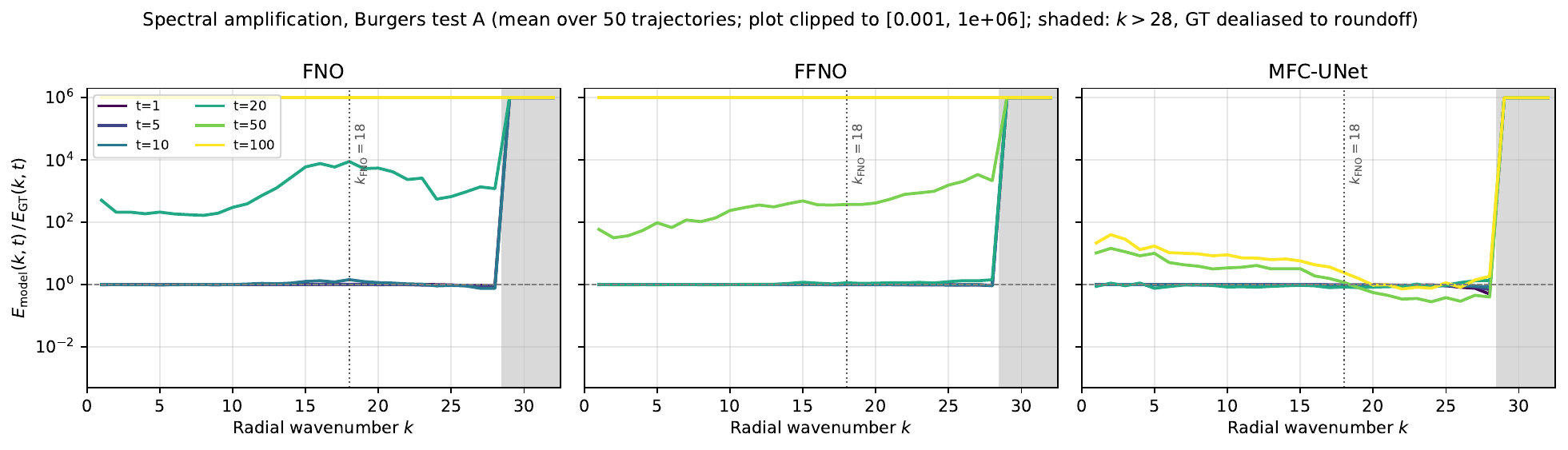}
    \caption{Spectral amplification on Burgers test A. Vertical line: FNO mode cutoff \(k=18\); shaded: unresolved \(k>28\).}
    \label{fig:spectral-diagnostic}
\end{figure}

FNO retrained at modes 32 (3.28M params; truncation beyond the dealiasing limit): validation rel.\ \(\ell_2=0.0053\), GMean@20 on A \(0.178\), yet 50/50 trajectories diverge by step 50 on every test set. Peak amplification \(3.1\times 10^{47}\) at \((k=20,t=100)\).

\section{Seed robustness}
\label{app:seeds}

Four independent MFC-full draws remain bounded on all 200 trajectories. Every draw retains a \(\ge 10\times\) GMean@100 gap to every baseline (worst case \(12.3\times\) vs.\ F-FNO on A). Long-horizon GMean varies by up to \(\approx 2\times\) across seeds.

\section{Runtime and memory}
\label{app:runtime}

One idle RTX 4090, batch 16, fp32, median of 100 iterations after warm-up.

\begin{table}[h]
\centering
\caption{Runtime (ms/iter) and peak memory (GiB).}
\label{tab:runtime}
\small
\begin{tabular}{l rrrrr rrr}
\toprule
& \multicolumn{5}{c}{ms / iteration} & \multicolumn{3}{c}{Peak memory} \\
\cmidrule(lr){2-6} \cmidrule(lr){7-9}
Model & \(64^2\) & \(128^2\) & \(256^2\) & \(64{\times}32\) & \(240{\times}121\) & \(64^2\) & \(128^2\) & \(256^2\) \\
\midrule
FNO      & 0.9  & 2.3   & 18.8  & 0.8  & 4.8   & 0.12 & 0.44 & 1.12 \\
F-FNO    & 2.8  & 19.5  & 83.0  & 1.5  & 38.2  & 1.23 & 4.81 & 2.67 \\
MFC-UNet & 4.5  & 12.9  & 75.9  & 4.2  & 29.7  & 0.42 & 1.14 & 4.51 \\
AFNO     & 28.7 & 133.3 & 547.9 & 10.7 & 207.8 & 0.53 & 2.09 & 8.29 \\
\bottomrule
\end{tabular}
\end{table}

\section{WeatherBench2 latitude-boundary analysis}
\label{app:latitude}

\begin{figure}[h]
    \centering
    \includegraphics[width=\linewidth]{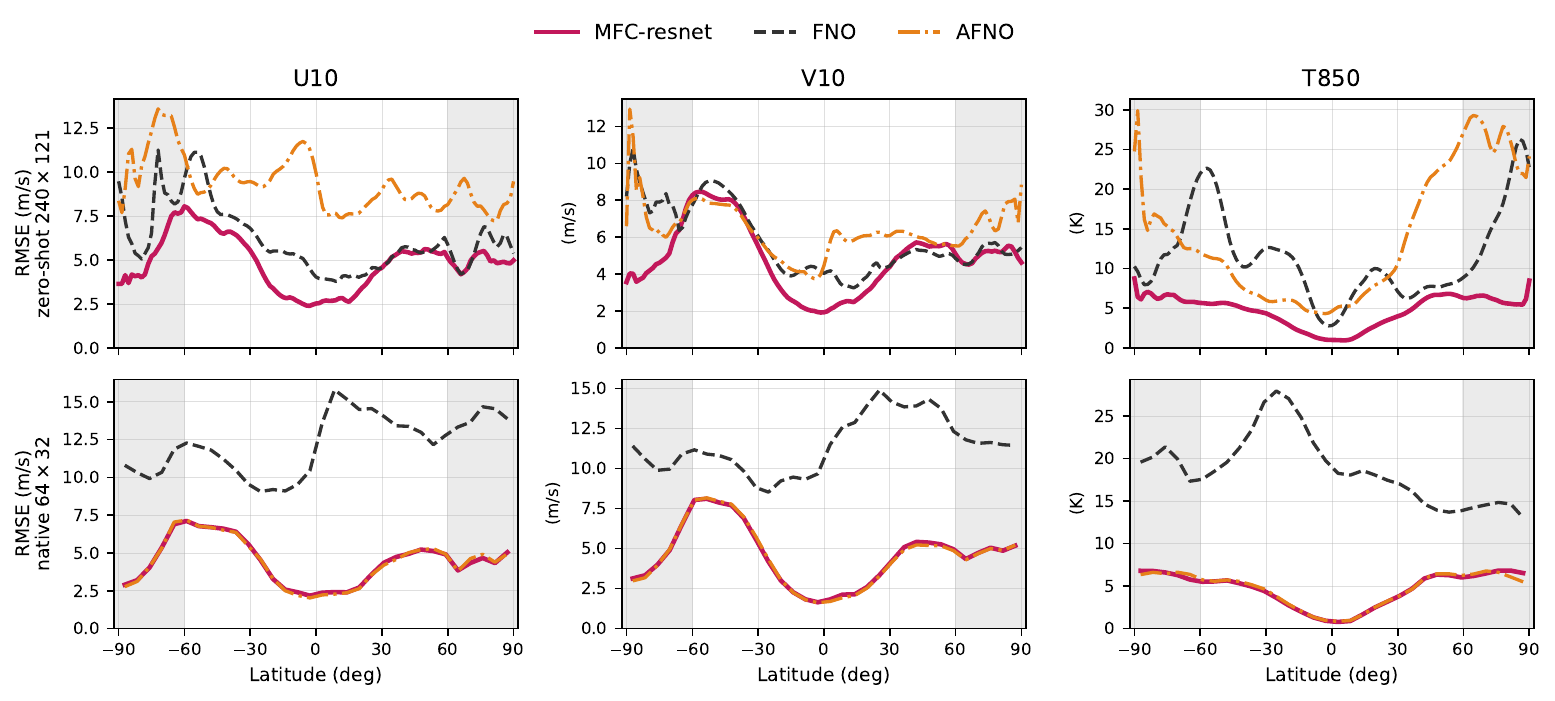}
    \caption{Per-latitude unweighted RMSE at step 50. Top: zero-shot \(240{\times}121\) (poles exist). Bottom: native \(64{\times}32\). Shaded: \(|{\rm lat}|\ge 60^\circ\).}
    \label{fig:lat-binned}
\end{figure}

At step 50 on \(240{\times}121\), pole amplification (worst pole-adjacent row / median \(|{\rm lat}|\le 60^\circ\)) is \(1.0\)--\(2.1\times\) (MFC), \(1.7\)--\(2.9\times\) (FNO), \(1.1\)--\(3.8\times\) (AFNO); \(T_{850}\) at \(-90^\circ\) is \(8.8\), \(10.3\), \(24.8\)\,K. Taper-ratio sweep at \(240{\times}121\) changes step-50 RMSE by at most \(1.6\%\); 0/39 ICs diverge.

\section{Longer rollouts}
\label{app:longer-rollouts}

Burgers to 200 steps (50 ICs): FNO 50/50 diverged by step 50, F-FNO 50/50 by step 100, MFC 0/50; MFC GMean \(1.70\to 3.41\) at step 195. WeatherBench2 to 750 steps (130 ICs): FNO 42/130 diverge, AFNO and MFC 0/130; at step 750 MFC is ahead on all three variables (\(U_{10}\) 5.06 vs.\ 5.36, \(V_{10}\) 4.91 vs.\ 5.27, \(T_{850}\) 8.95 vs.\ 9.06\,K).

\section{Complementary Quantitative and Qualitative Evaluation}

\subsection{Navier--Stokes}
\label{sec:appendix-ns}

\paragraph{Qualitative rollout comparison.}
Figure~\ref{fig:ns_appendix_rollout_panels} compares autoregressive rollouts of MFC-UNet, FNO, and UNet
on the 2D Navier--Stokes Kolmogorov-flow benchmark at horizons
$\Delta t \in \{3, 6, 9, 12\}$.
MFC-UNet faithfully preserves coherent vortical structures and their spatial locations
throughout the rollout, whereas FNO exhibits substantial degradation at longer horizons,
consistent with the rapid error growth shown in Figure~\ref{fig:ns_rollout_error}.
UNet remains numerically stable but shows increasing smoothing and structural drift.

\paragraph{Representative sample with error maps.}
To provide a fair and reproducible qualitative assessment, we select the test trajectory
whose ground-truth initial condition has \emph{median} high-frequency energy content.
Specifically, for each test trajectory we compute the fraction of energy in wavenumbers
$k > k_c$ (with $k_c = 12$) from the 2D power spectrum of the ground-truth initial
vorticity field:
\begin{equation}
  f_{\mathrm{HF}} = \frac{\sum_{k > k_c} E(k)}{\sum_{k} E(k)},
  \label{eq:hf-frac}
\end{equation}
where $E(k)$ is the angularly averaged energy spectrum.  We then select the trajectory
at the 50th percentile (P50) of the $f_{\mathrm{HF}}$ distribution.  This criterion
depends only on the input data and is independent of any model output, making it a
natural, model-free proxy for ``this trajectory has typical fine-scale content.''

Figure~\ref{fig:ns_appendix_rollout_panels} shows the selected P50 sample at rollout horizons
$\Delta t \in \{3, 6, 9, 12\}$.  Each model block contains a prediction row and a
signed pointwise error row.  Error maps display
$(u_{\mathrm{pred}} - u_{\mathrm{gt}}) / \mathrm{RMS}(u_{\mathrm{gt}})$
as a percentage, with a shared diverging colorscale across all models and timesteps.
The relative $\ell_2$ error,
$\|u_{\mathrm{pred}} - u_{\mathrm{gt}}\|_2 / \|u_{\mathrm{gt}}\|_2$,
is annotated in each error panel.

At early horizons ($\Delta t = 3$), all three models produce visually plausible
predictions.  As the rollout progresses, FNO develops large-scale structural errors:
vortex positions drift and fine-scale features are lost, yielding relative $\ell_2$
errors that grow from roughly 40\% at $\Delta t = 3$ to over 95\% at $\Delta t = 12$.
Annotated errors are relative \(\ell_2\); RMS enters only the error-map normalization.
At \(Re=5000\), long-horizon pointwise error is set by chaotic decorrelation, not scheme accuracy.
UNet retains overall stability but progressively smooths the vorticity field, resulting
in a broadband loss of energy and a steady increase in error.  MFC-UNet maintains the
sharpest vortex cores and the closest spatial alignment with the ground truth across all
horizons, with the lowest relative $\ell_2$ error at every timestep.

A black arrow marks a representative positive (cyclonic) vortex in the lower portion of
the domain.  Tracking this feature across models reveals that FNO and UNet exhibit
noticeable spatial drift and deformation of the vortex, while MFC-UNet preserves both
its position and intensity more accurately.

\begin{figure*}[!t]
    \centering

    \newcommand{\labw}{0.09\textwidth}
    \newcommand{\imgw}{0.205\textwidth}
    \newcommand{\panel}[2]{%
        \begin{minipage}{\imgw}
            \centering
            \includegraphics[width=\linewidth]{#1}\\[-0.5mm]
            \includegraphics[width=\linewidth]{#2}
        \end{minipage}%
    }

    \makebox[\labw][c]{}%
    \makebox[\imgw][c]{$\Delta t = 3$}%
    \makebox[\imgw][c]{$\Delta t = 6$}%
    \makebox[\imgw][c]{$\Delta t = 9$}%
    \makebox[\imgw][c]{$\Delta t = 12$}

    \vspace{1mm}

    \begin{minipage}{\labw}
        \centering
        \rotatebox{90}{\scalebox{0.8}{\textbf{Ground truth}}}
    \end{minipage}%
    \begin{minipage}{\imgw}
        \includegraphics[width=\linewidth]{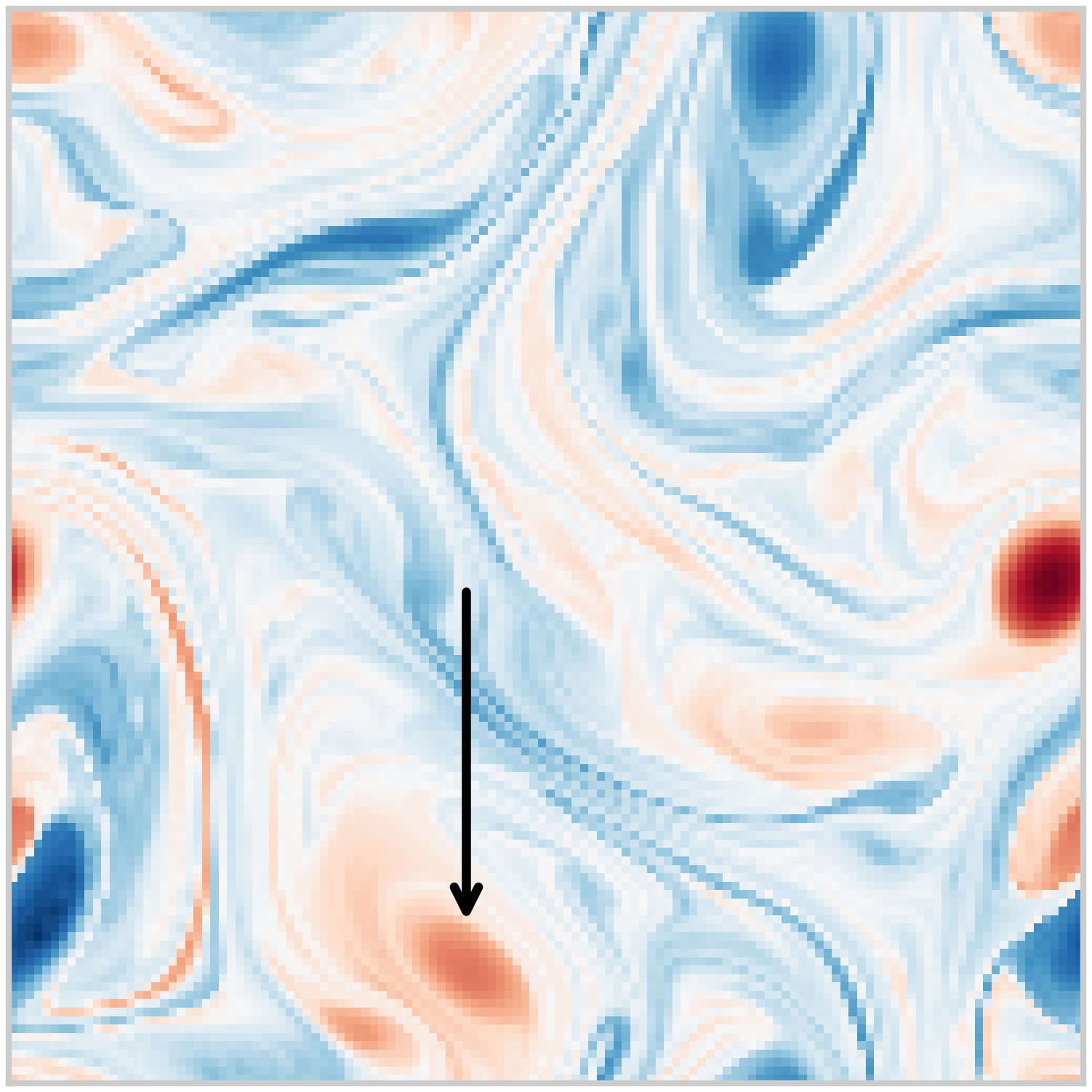}
    \end{minipage}%
    \begin{minipage}{\imgw}
        \includegraphics[width=\linewidth]{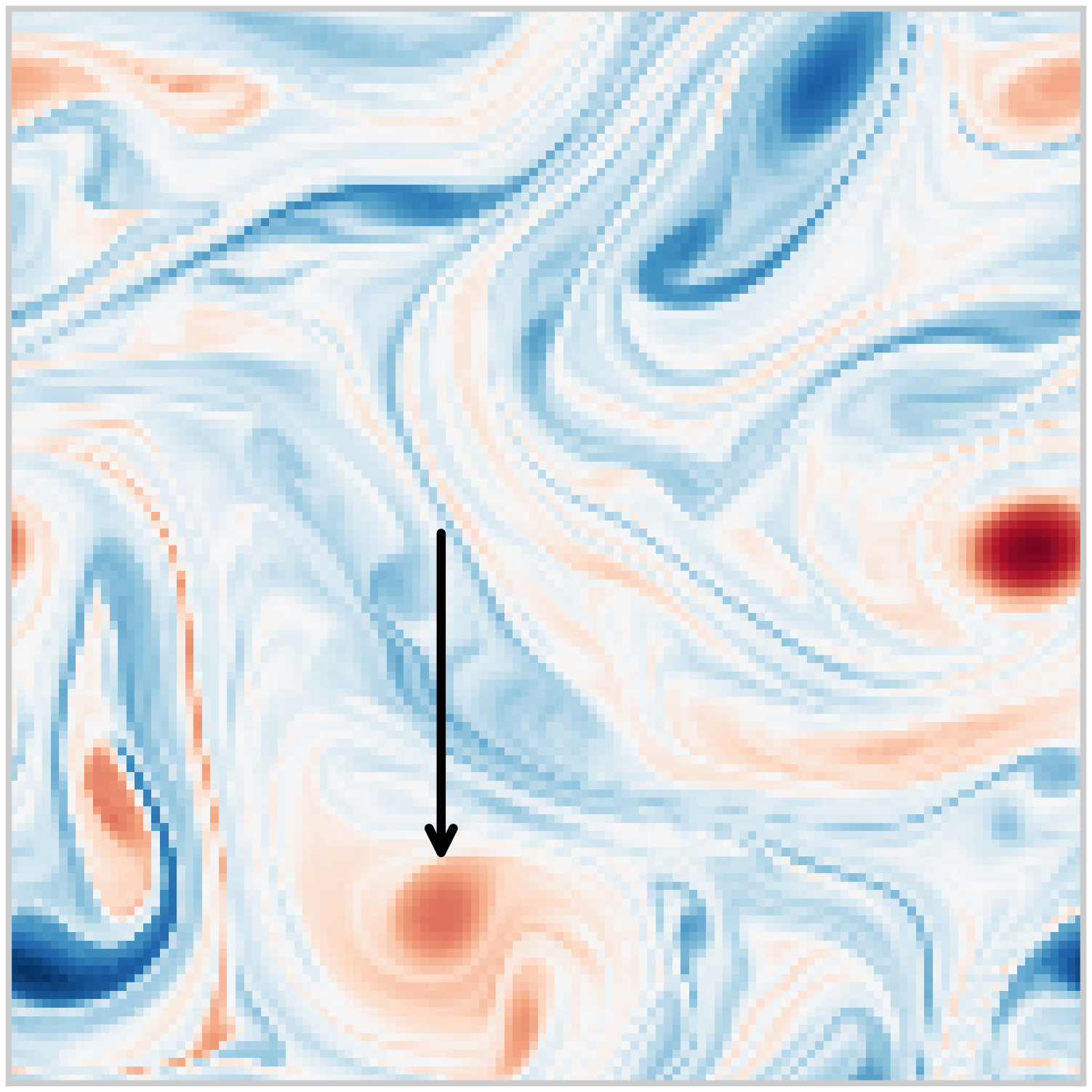}
    \end{minipage}%
    \begin{minipage}{\imgw}
        \includegraphics[width=\linewidth]{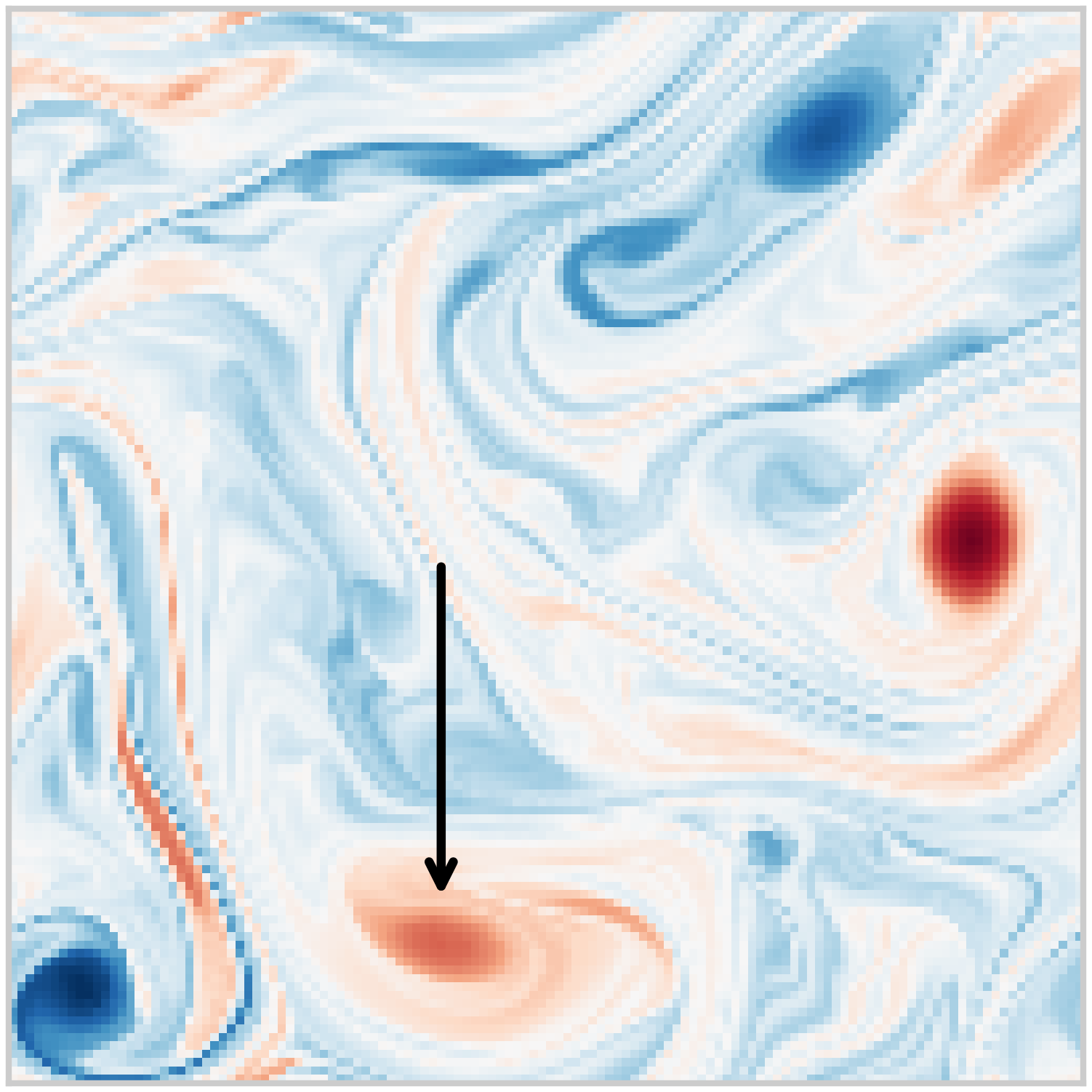}
    \end{minipage}%
    \begin{minipage}{\imgw}
        \includegraphics[width=\linewidth]{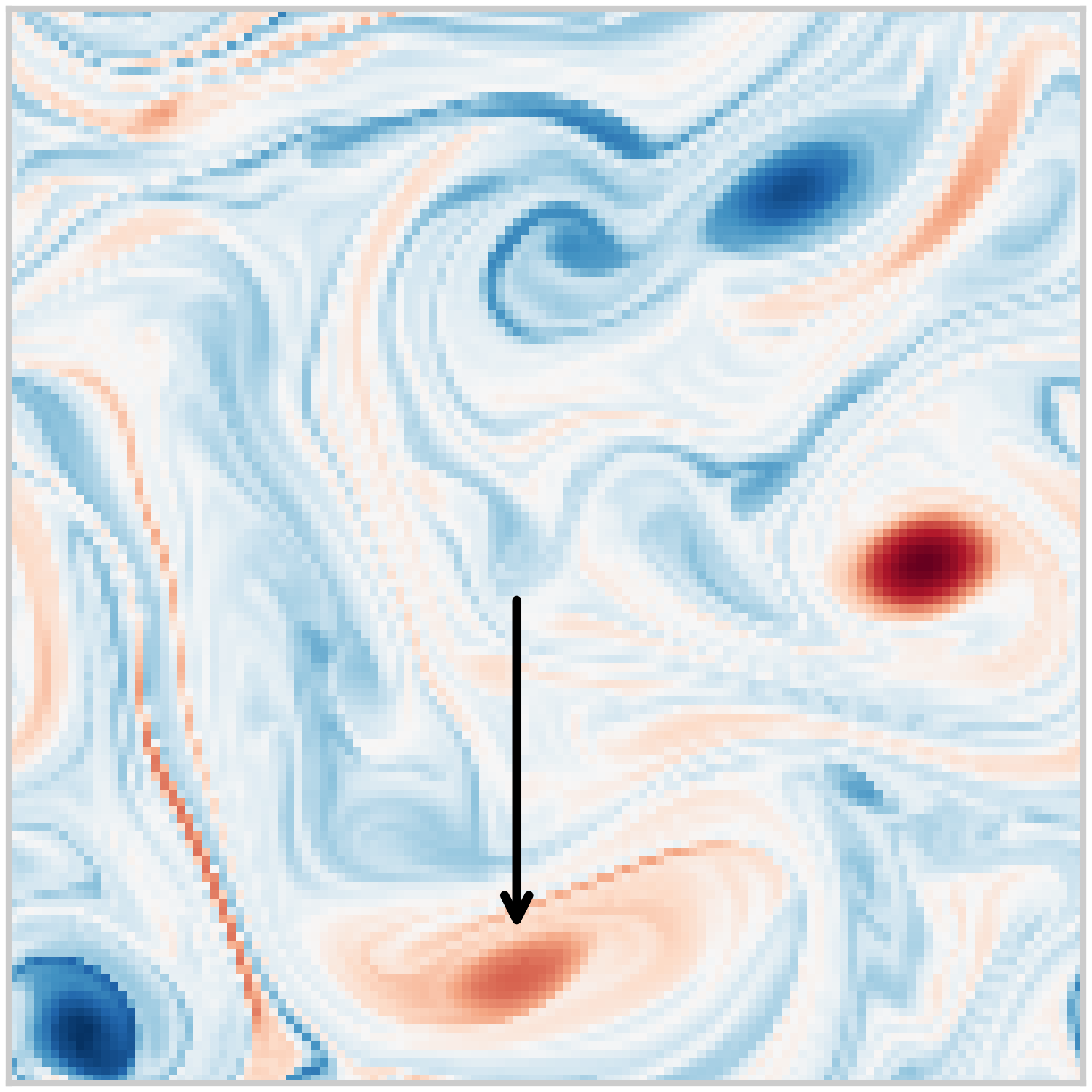}
    \end{minipage}

    \vspace{1.5mm}

    \begin{minipage}{\labw}
        \centering
        \rotatebox{90}{\scalebox{0.8}{\textbf{MFC-UNet}}}
    \end{minipage}%
    \panel
    {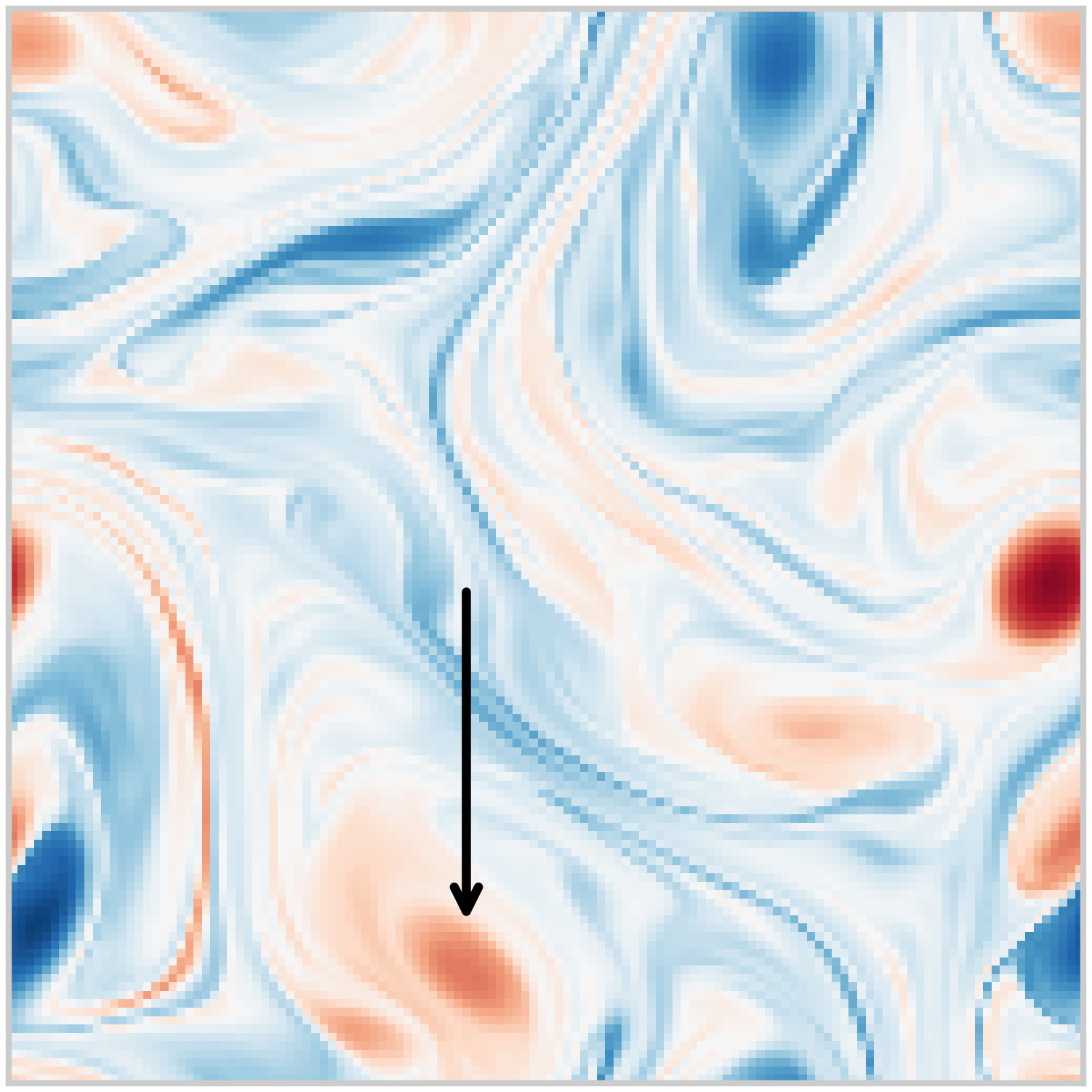}
    {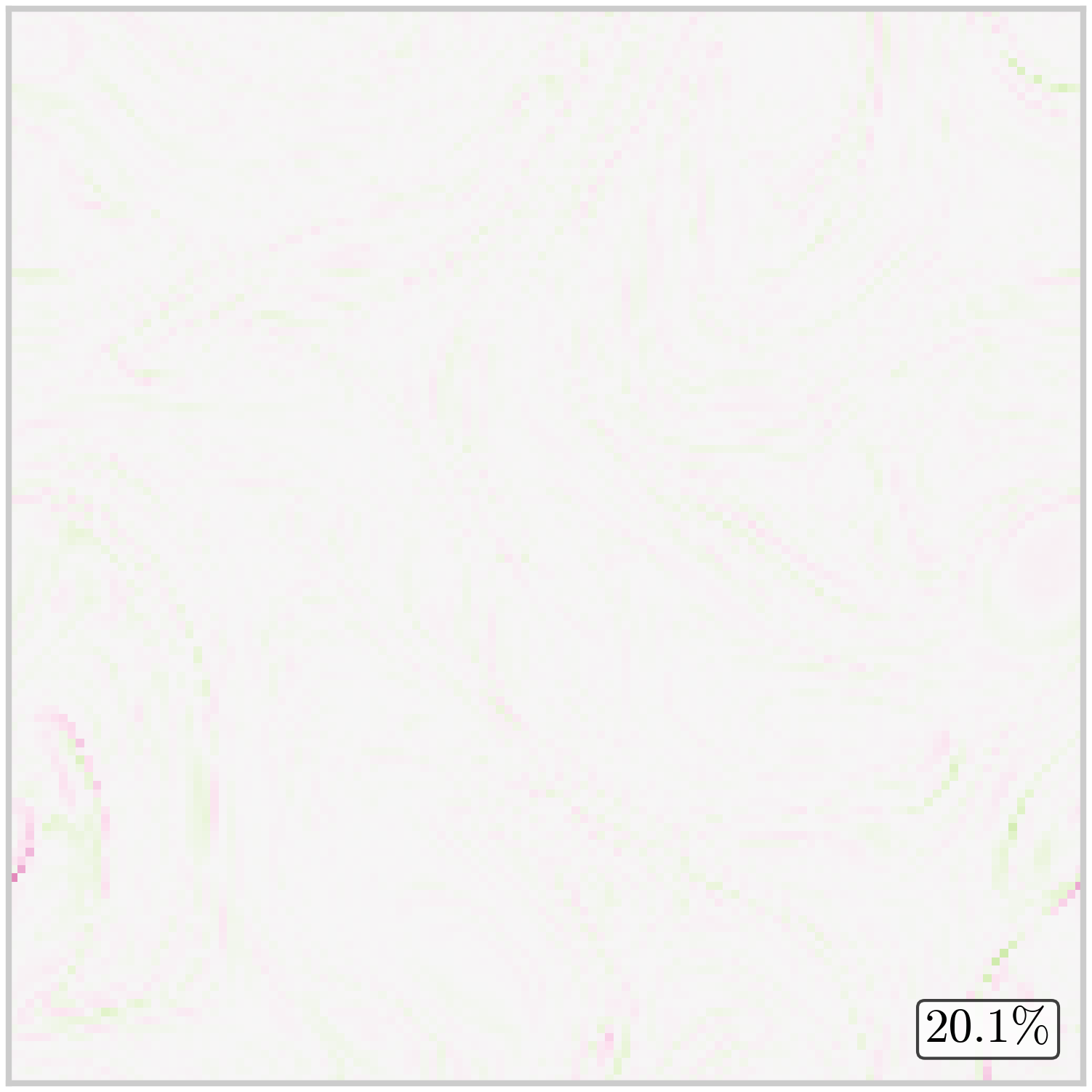}
    \panel
    {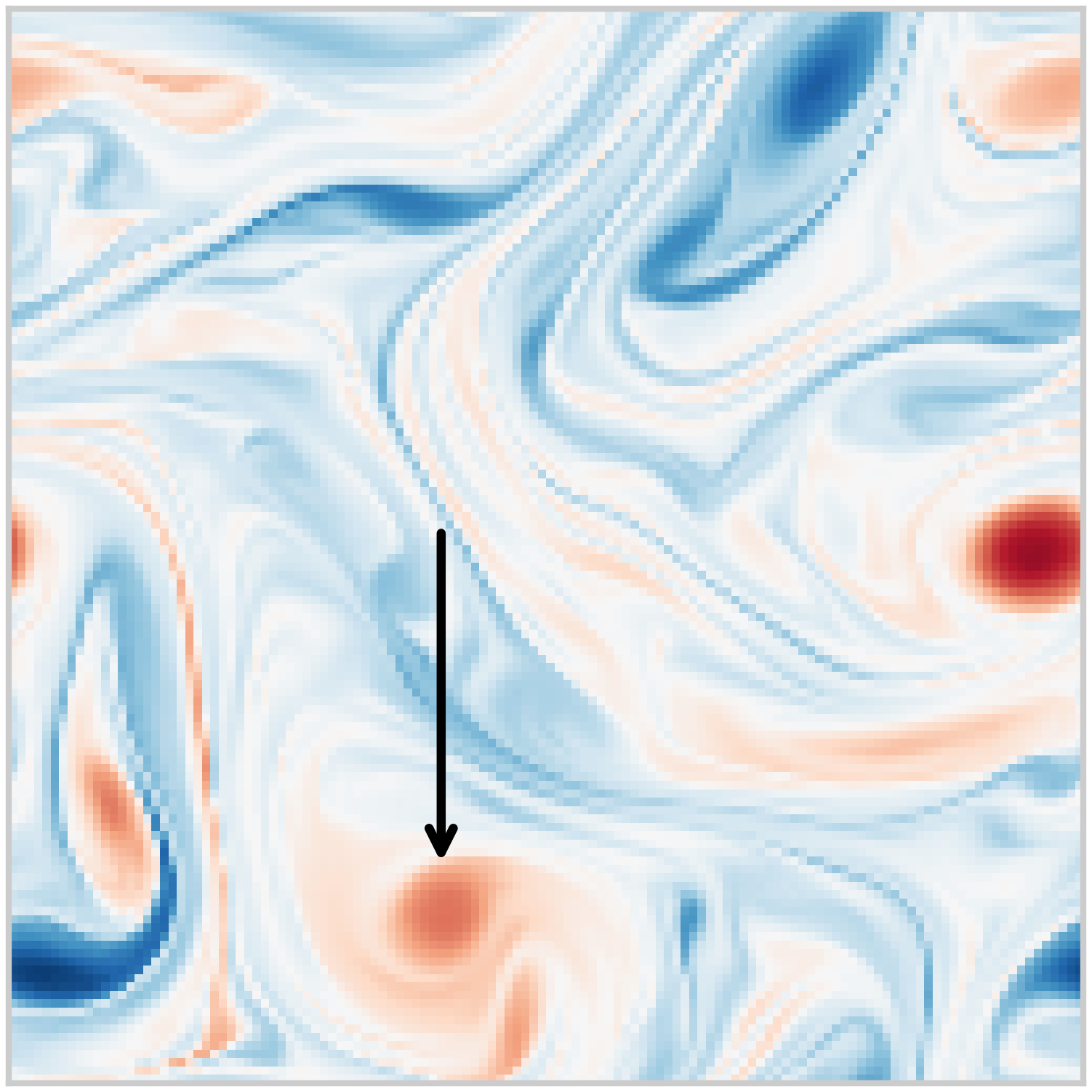}
    {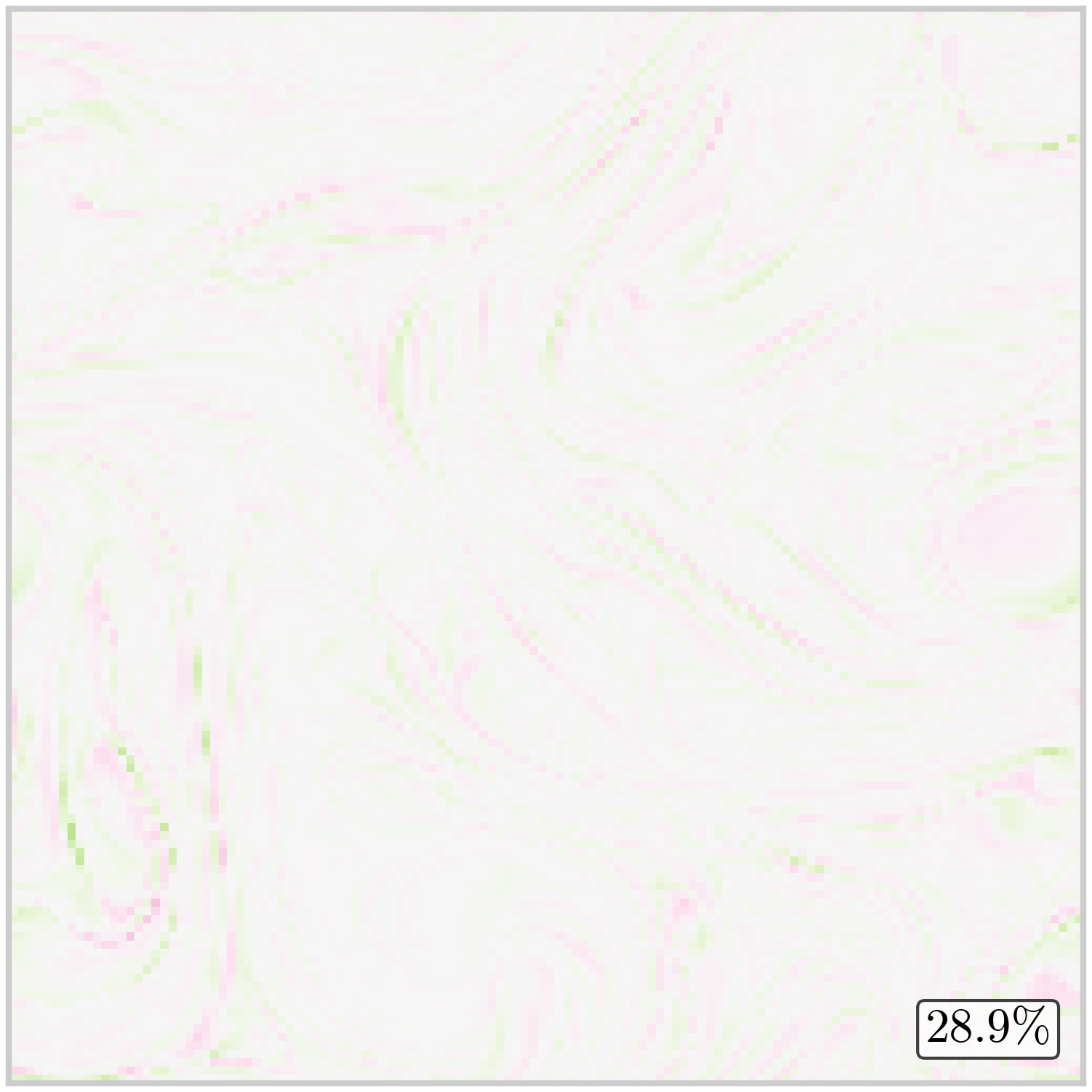}
    \panel
    {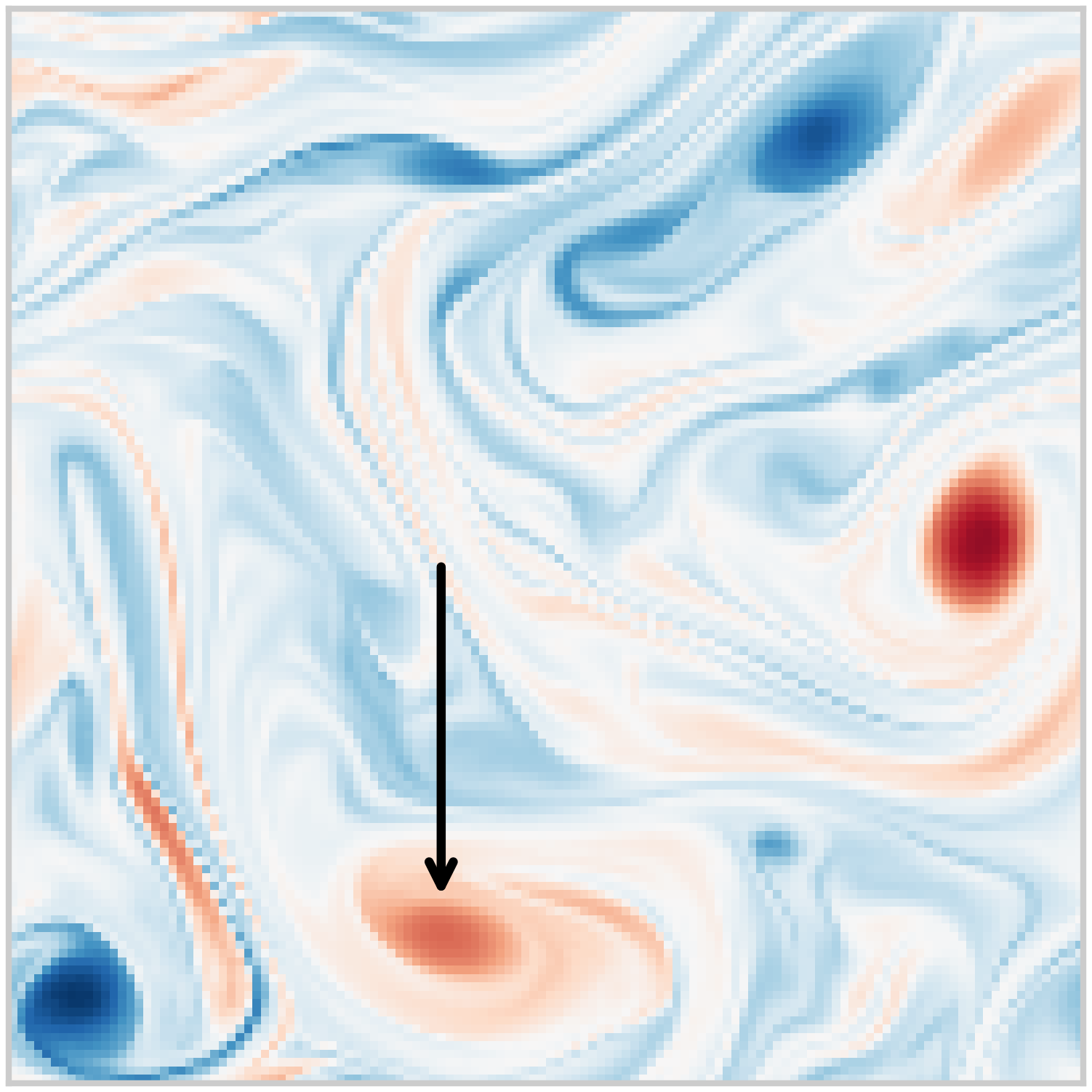}
    {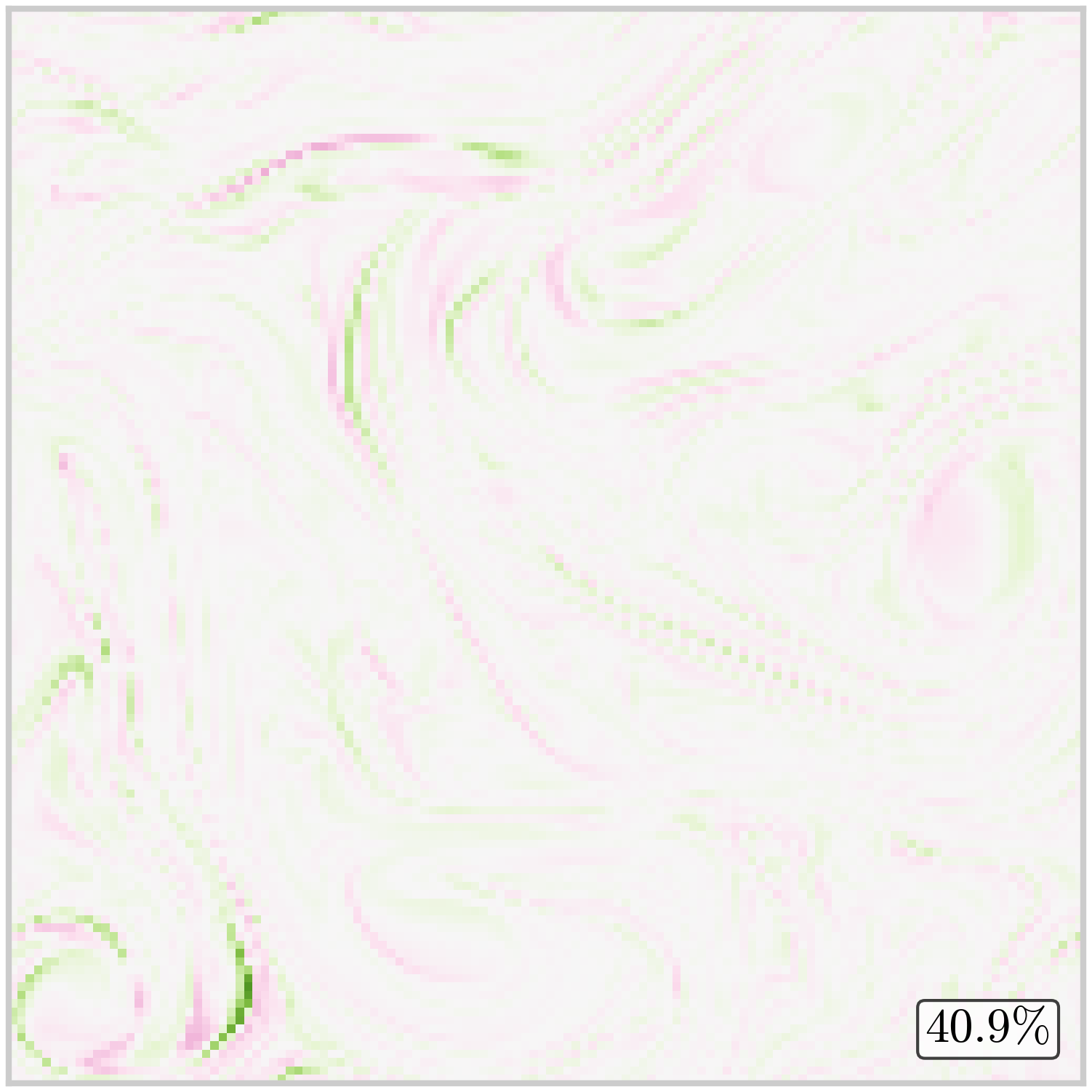}
    \panel
    {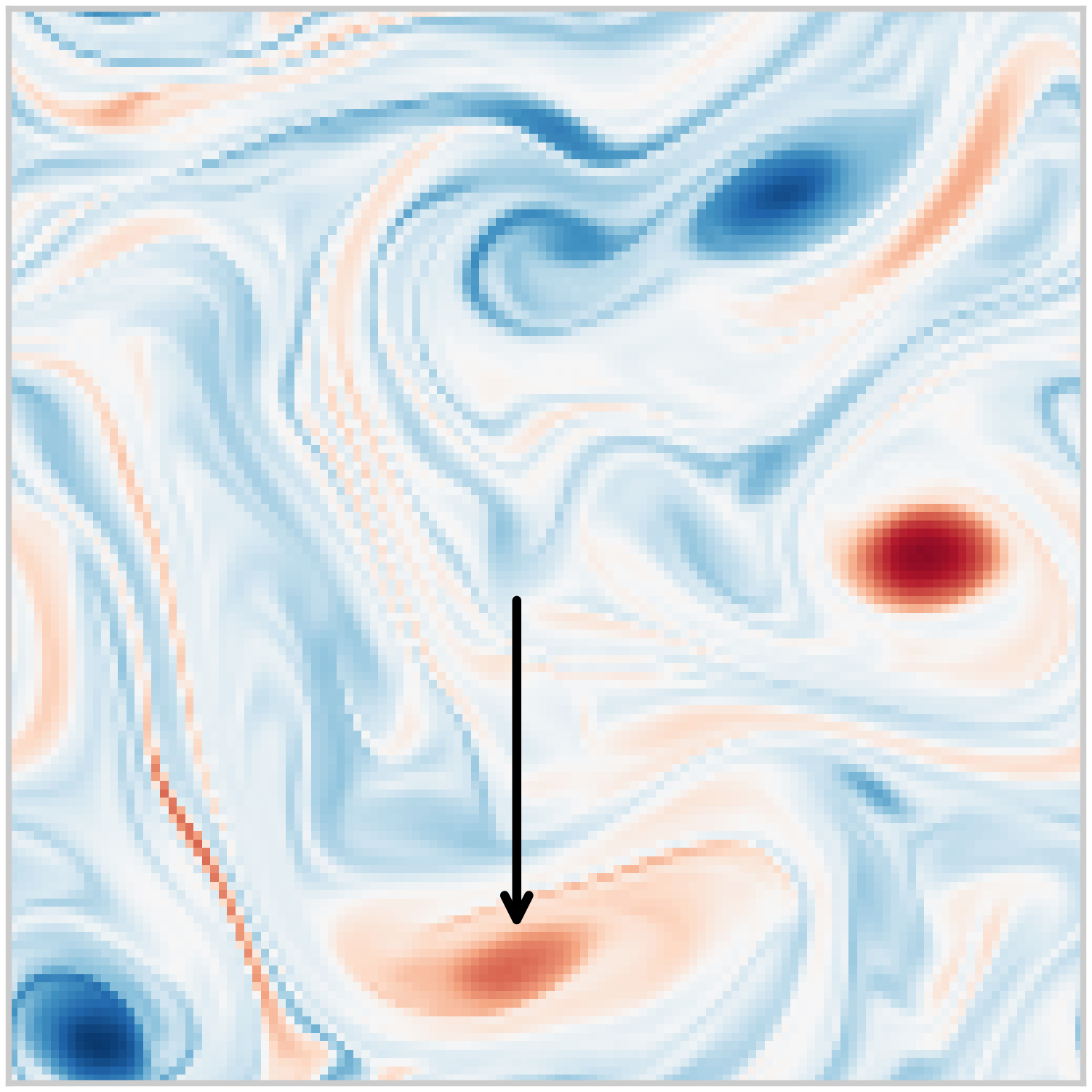}
    {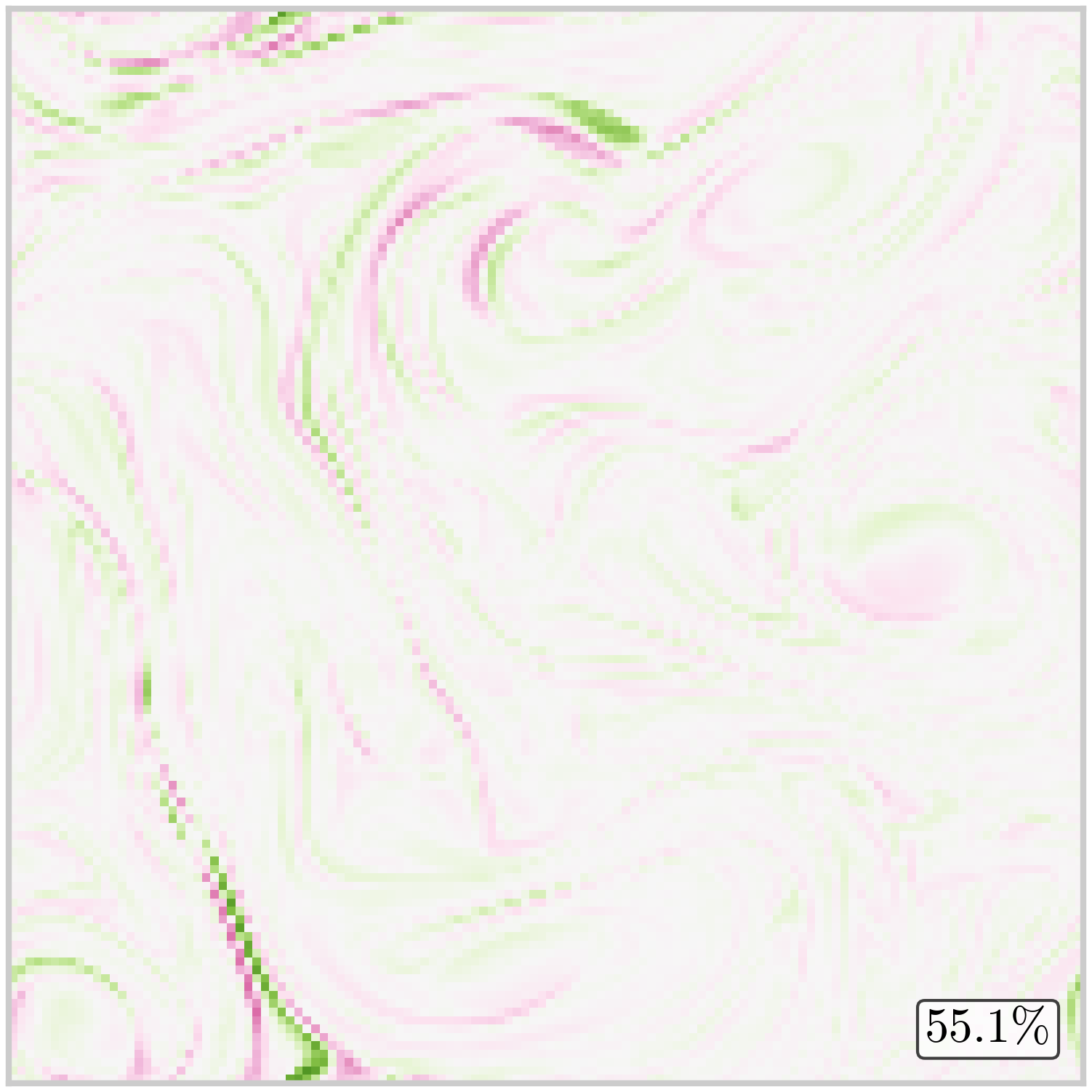}

    \vspace{1mm}

    \begin{minipage}{\labw}
        \centering
        \rotatebox{90}{\scalebox{0.8}{\textbf{FNO}}}
    \end{minipage}%
    \panel
    {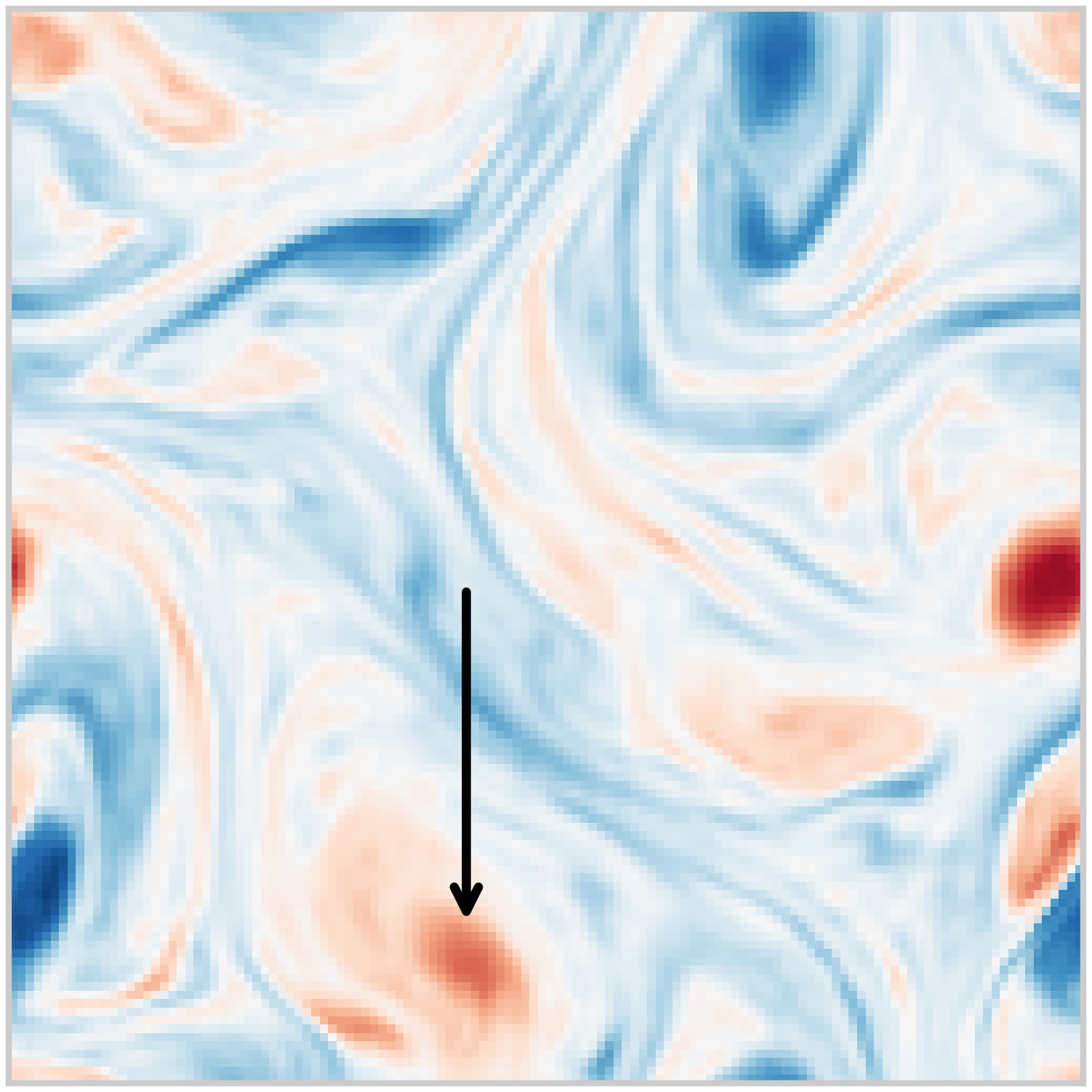}
    {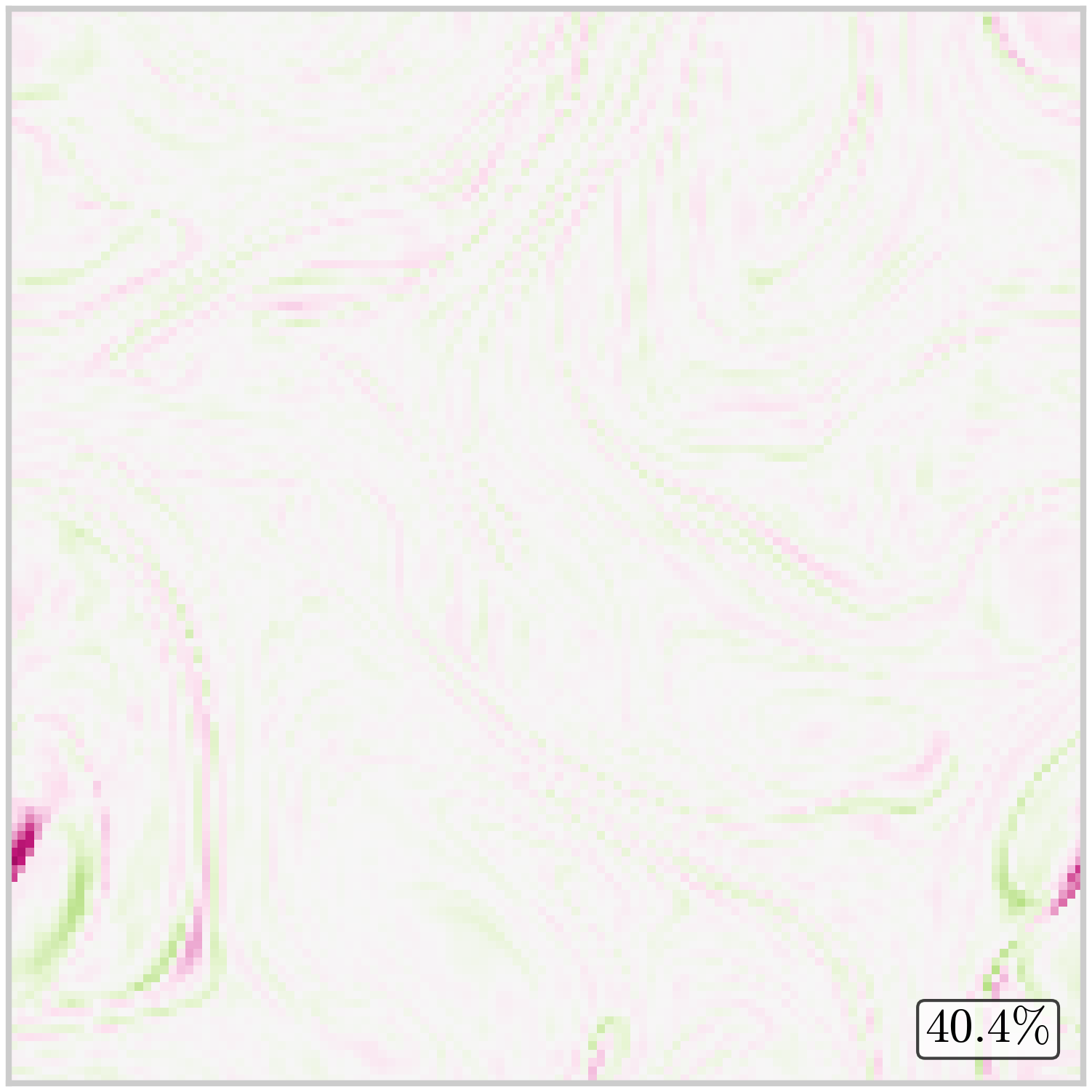}
    \panel
    {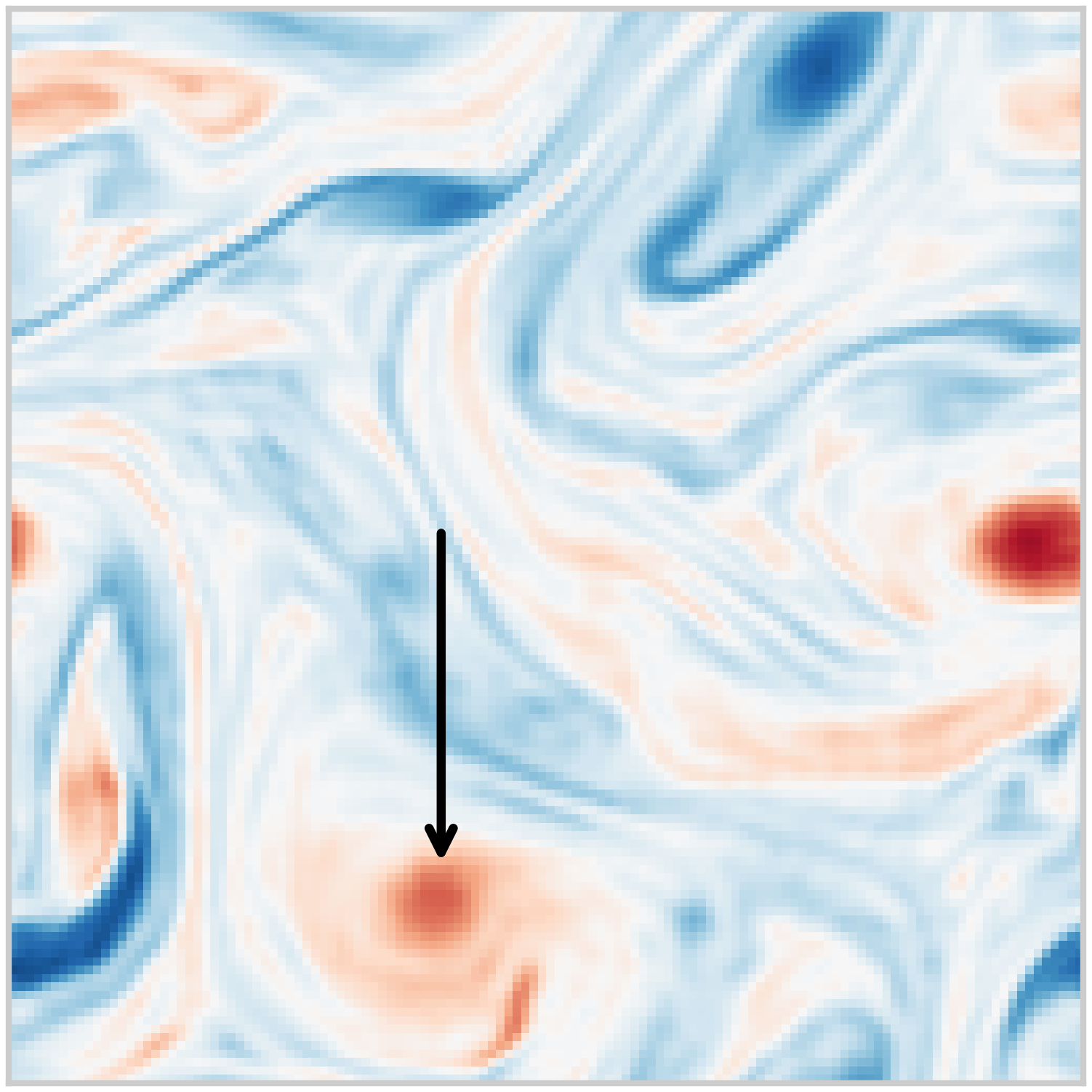}
    {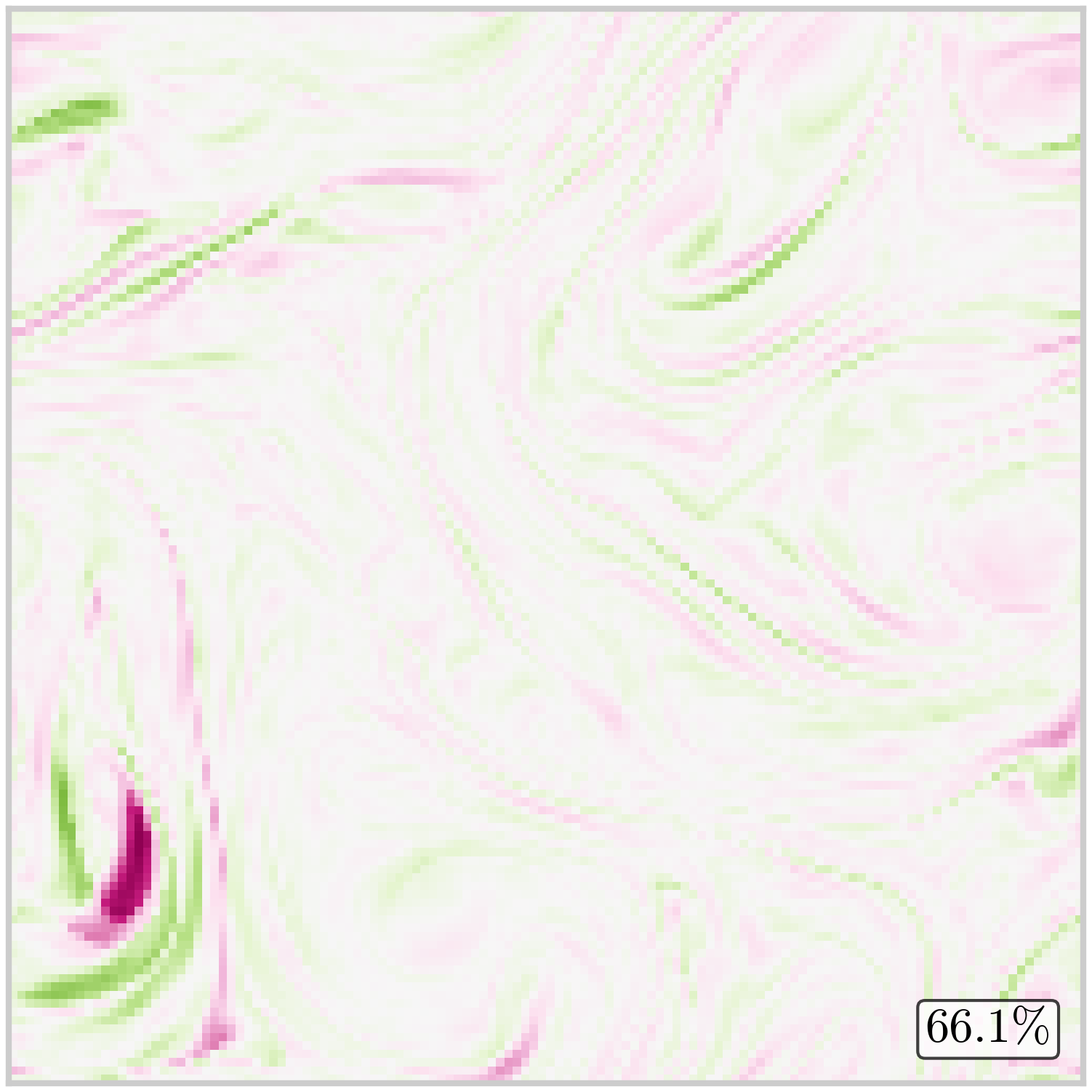}
    \panel
    {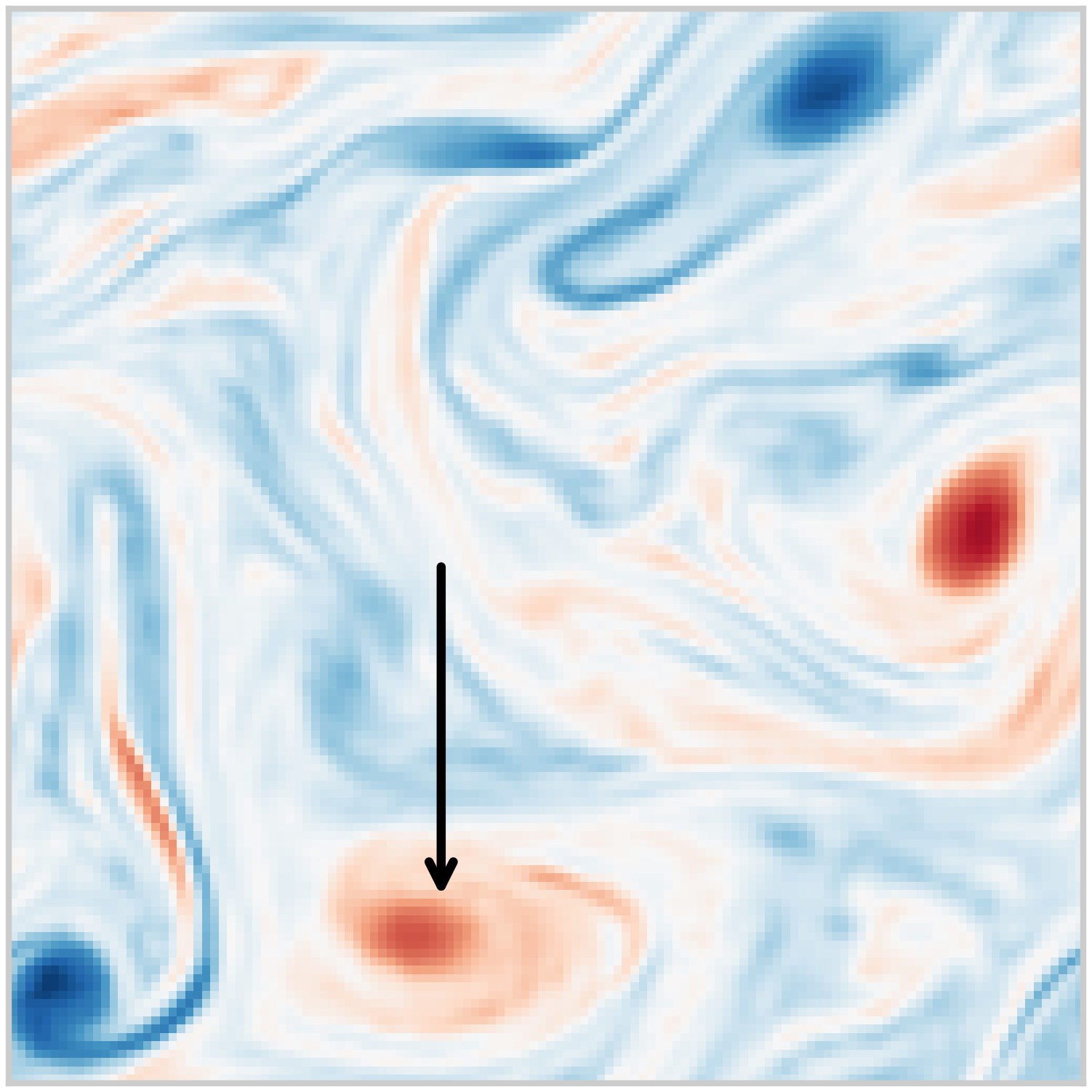}
    {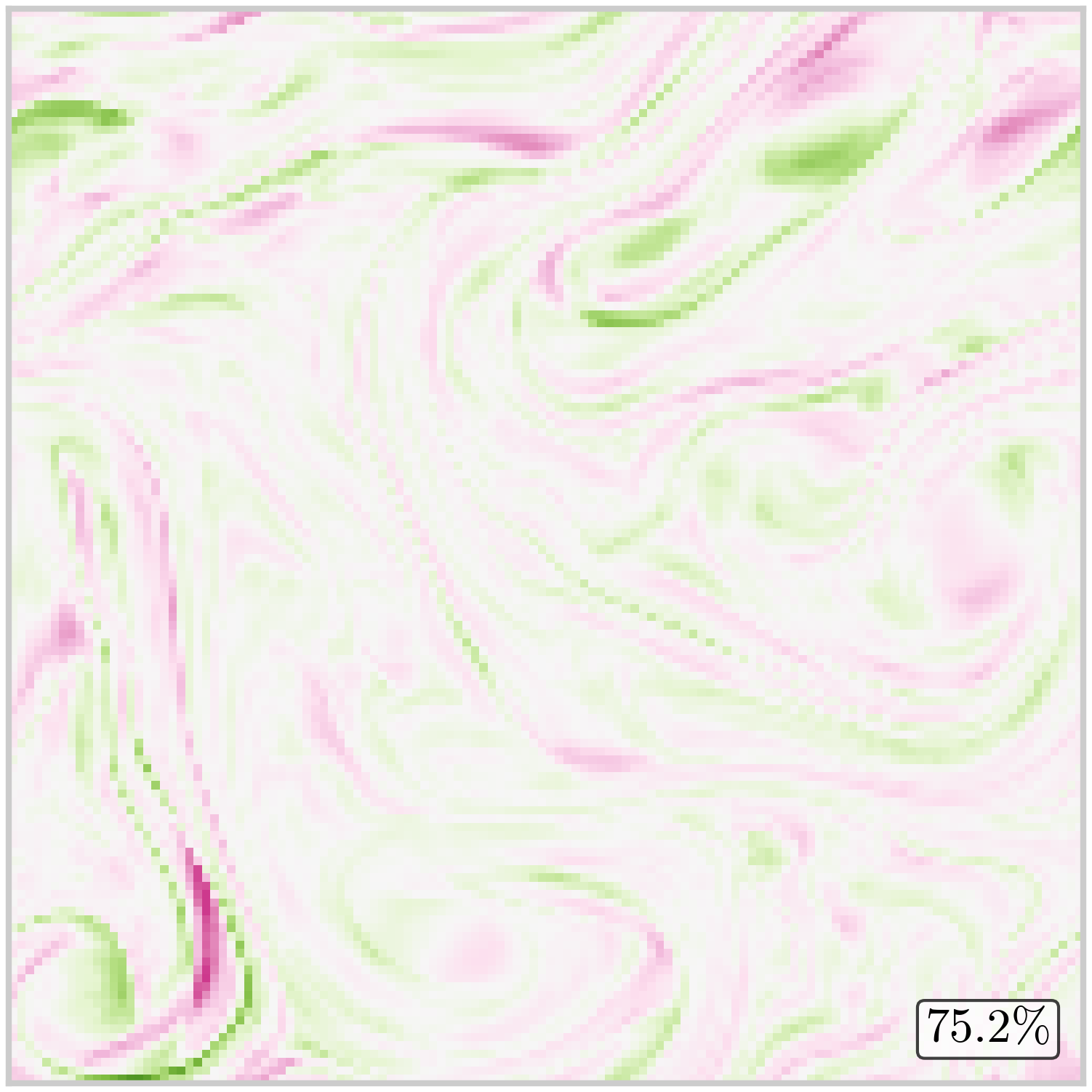}
    \panel
    {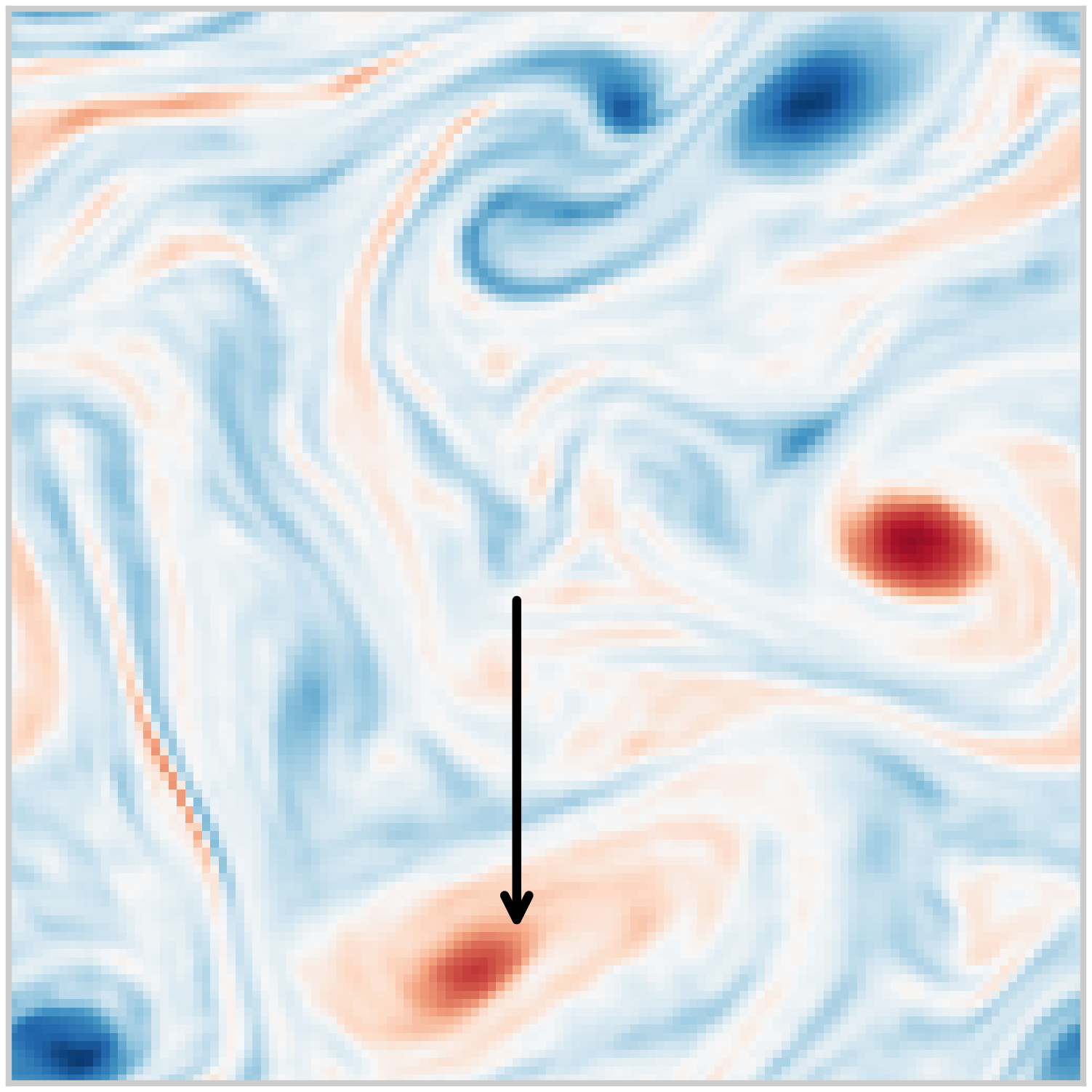}
    {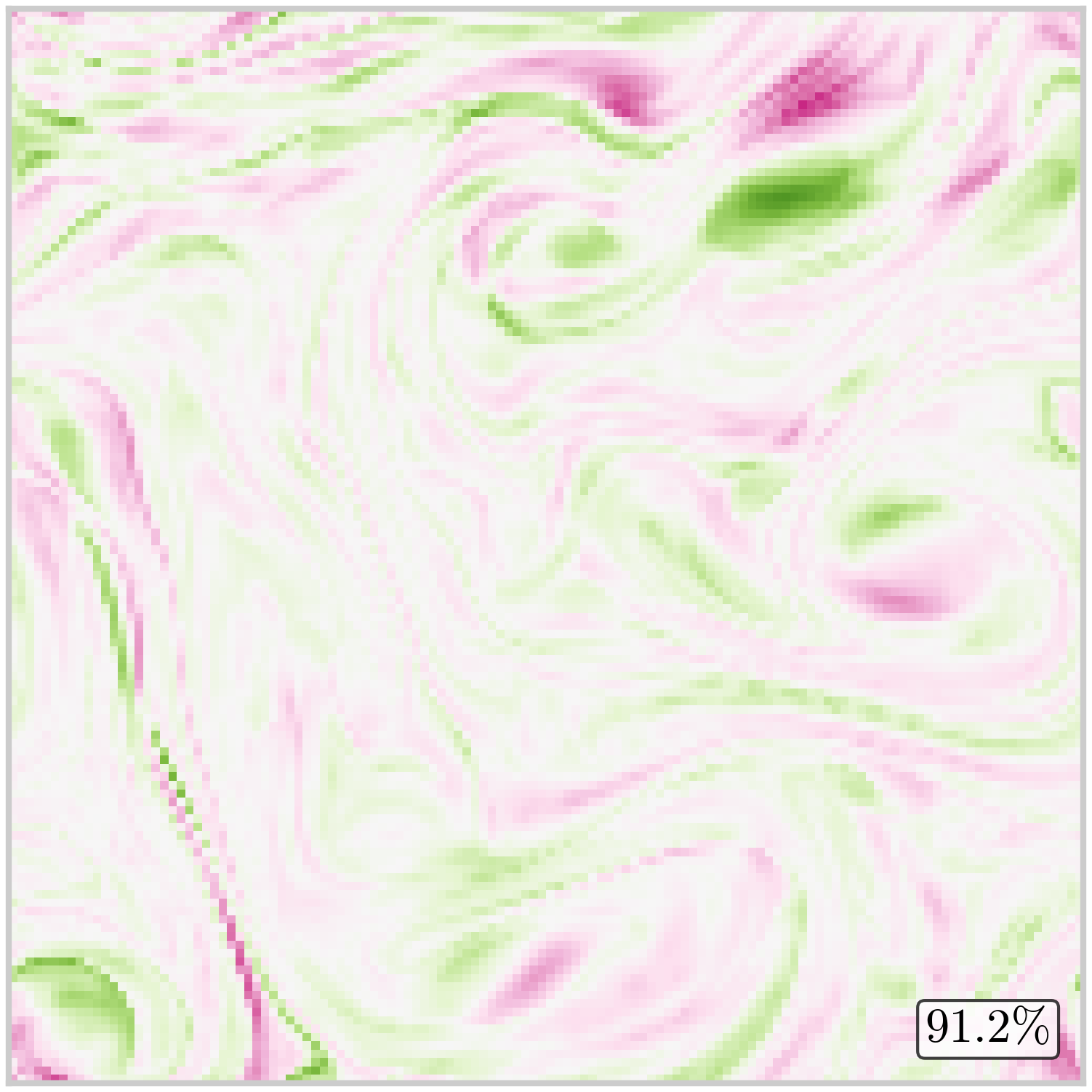}

    \vspace{1mm}

    \begin{minipage}{\labw}
        \centering
        \rotatebox{90}{\scalebox{0.8}{\textbf{UNet}}}
    \end{minipage}%
    \panel
    {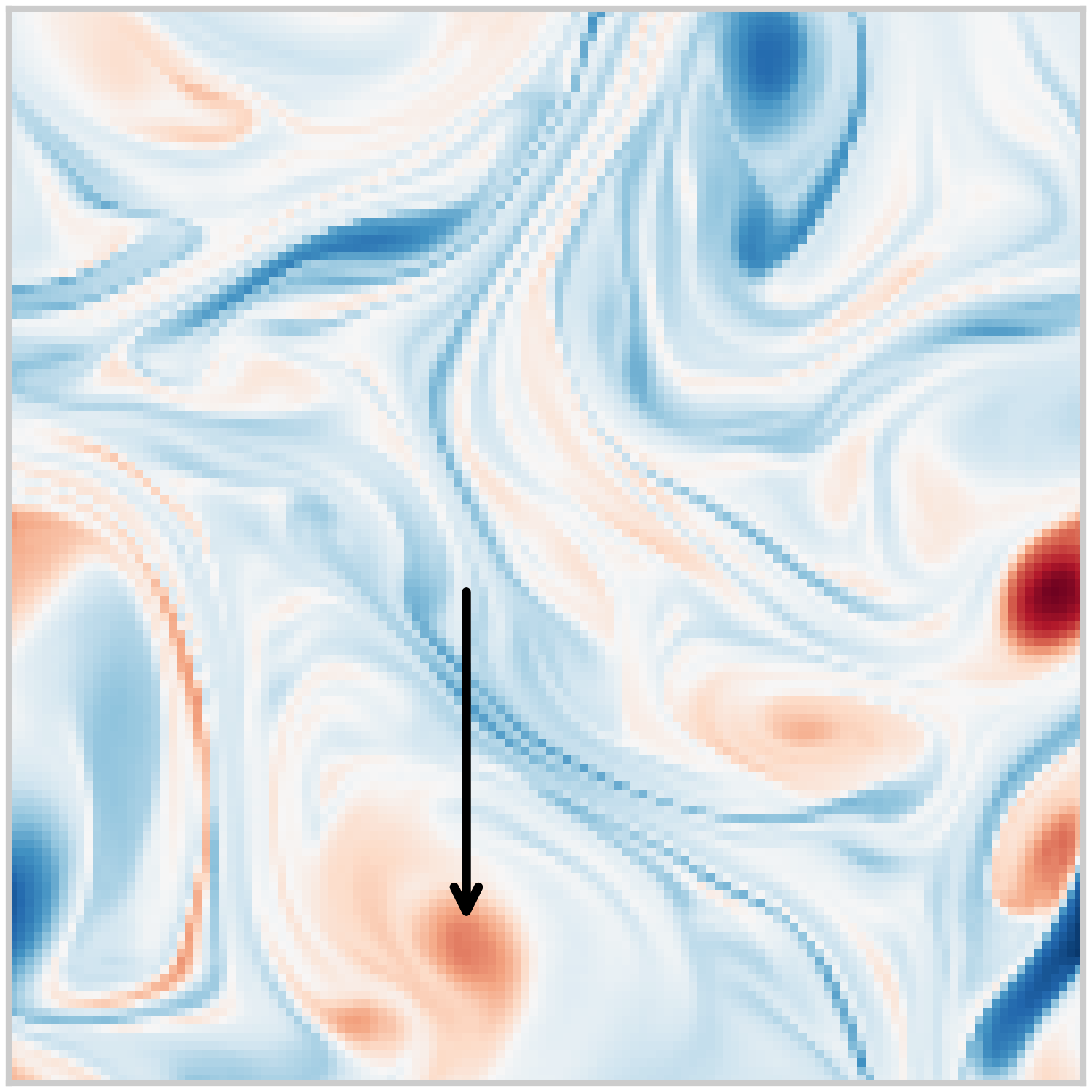}
    {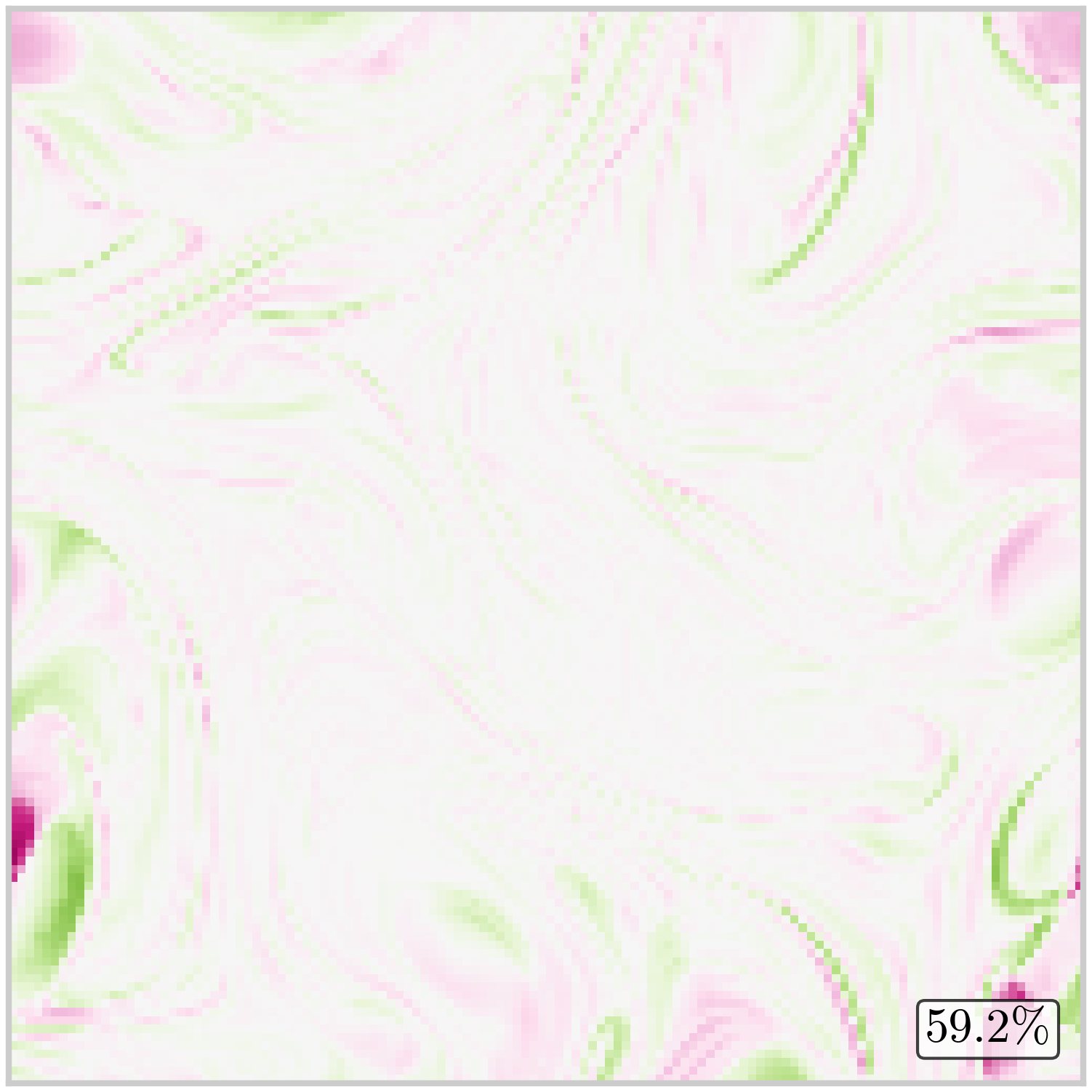}
    \panel
    {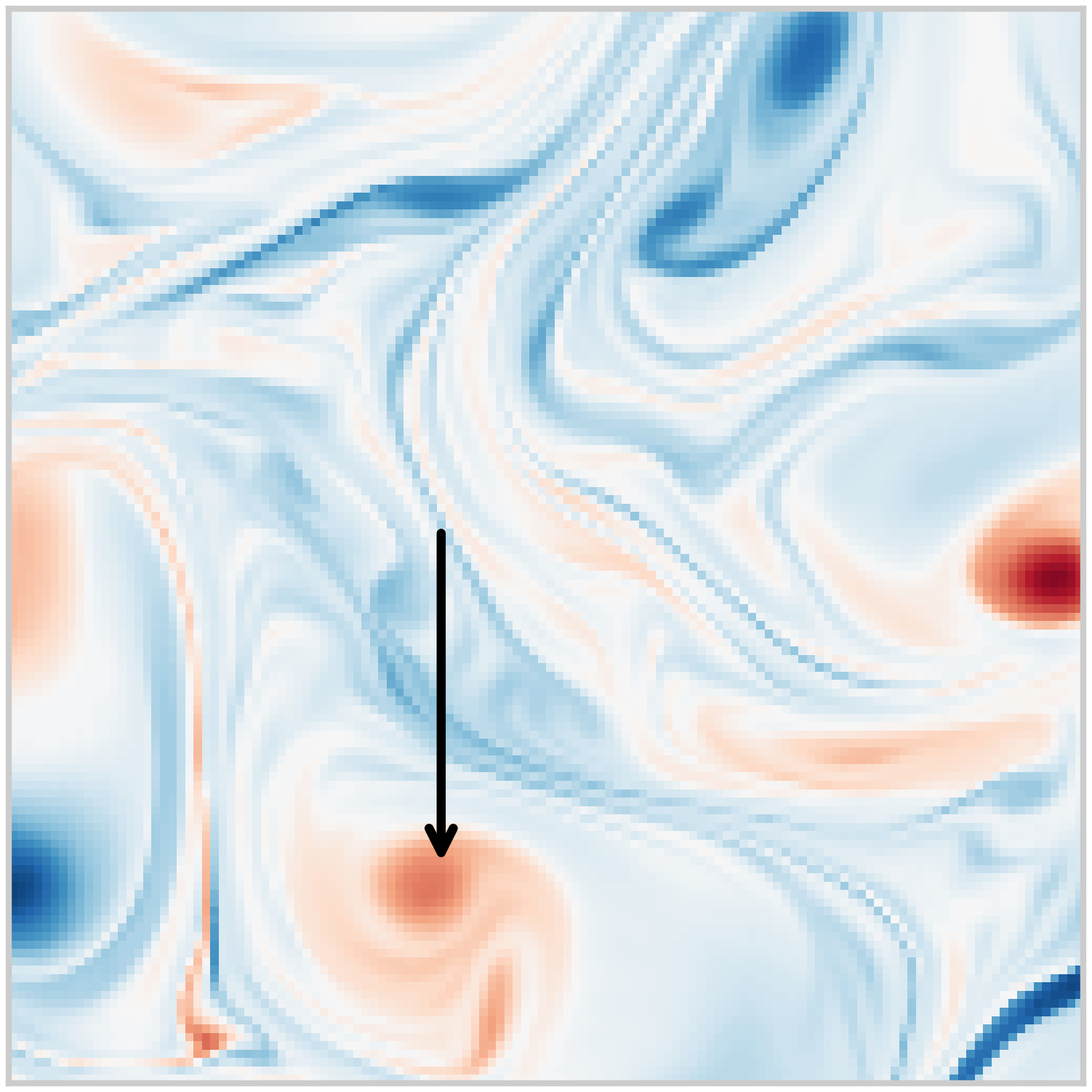}
    {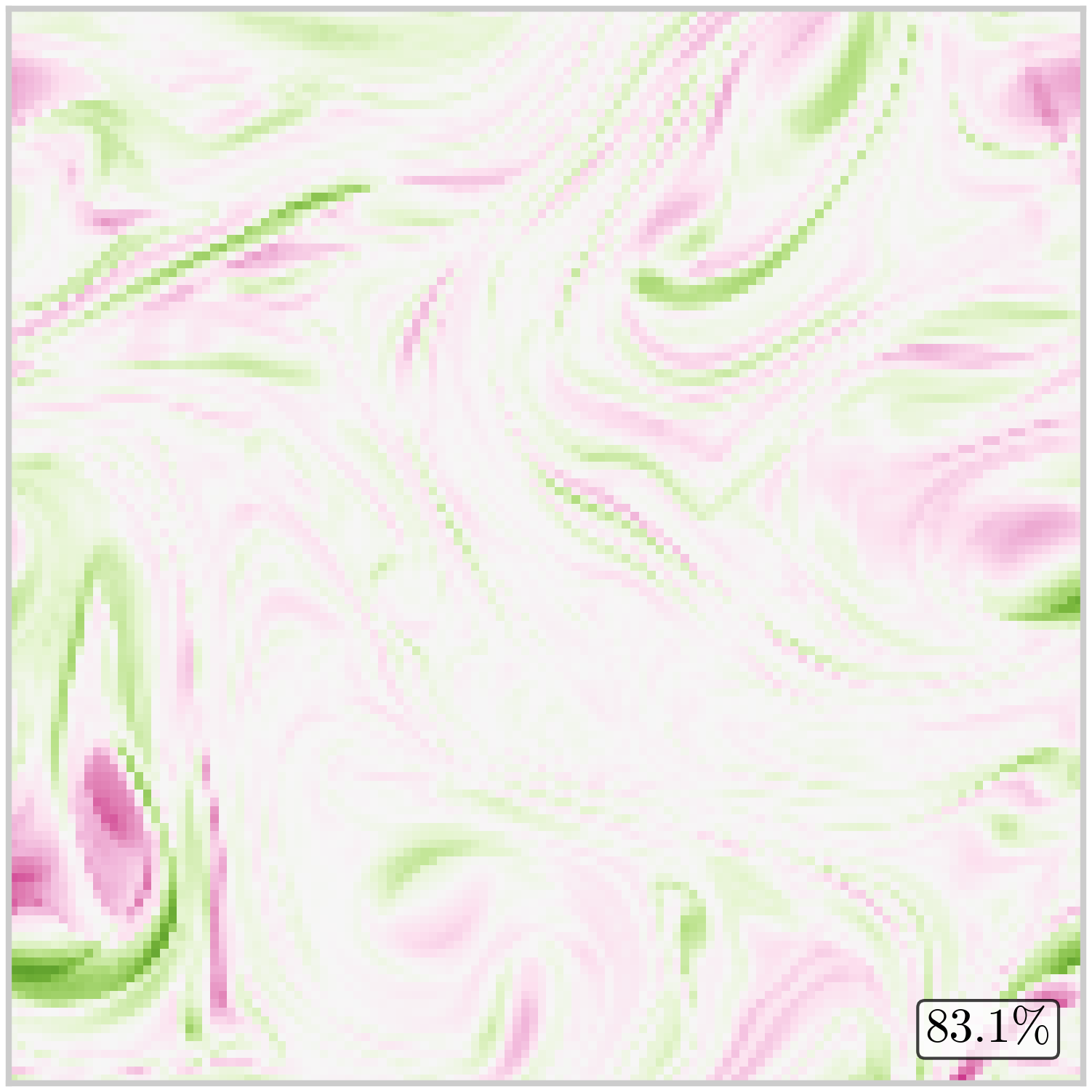}
    \panel
    {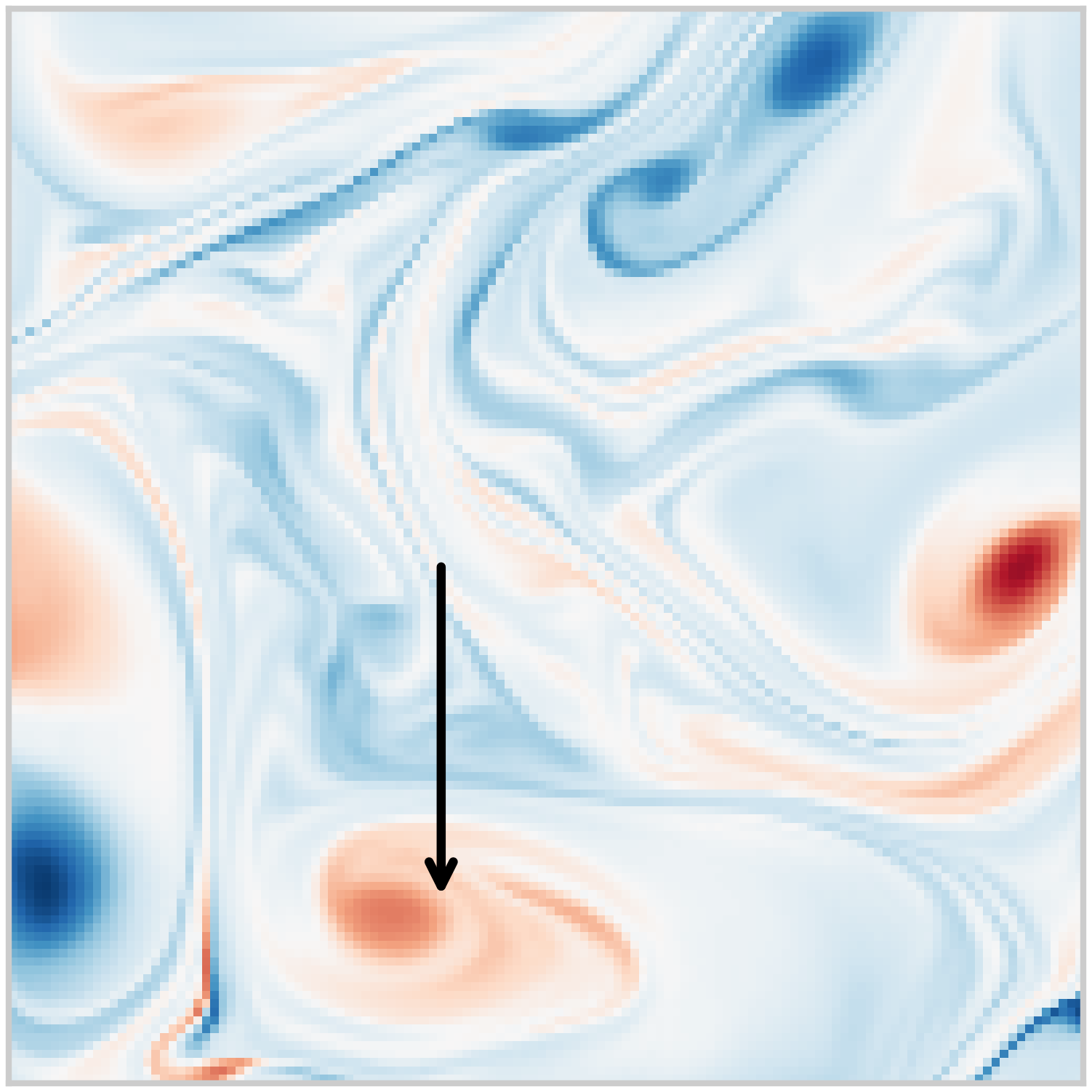}
    {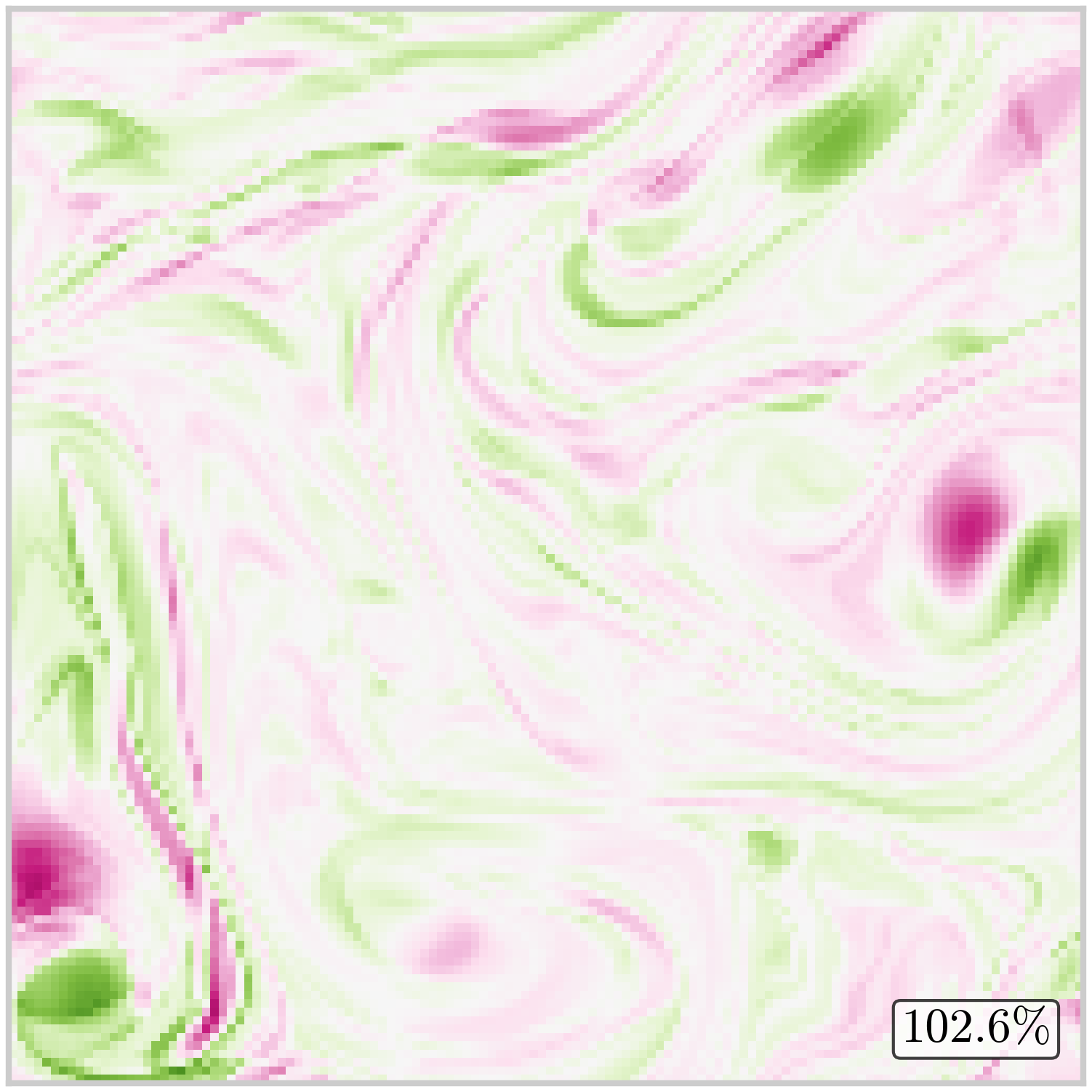}
    \panel
    {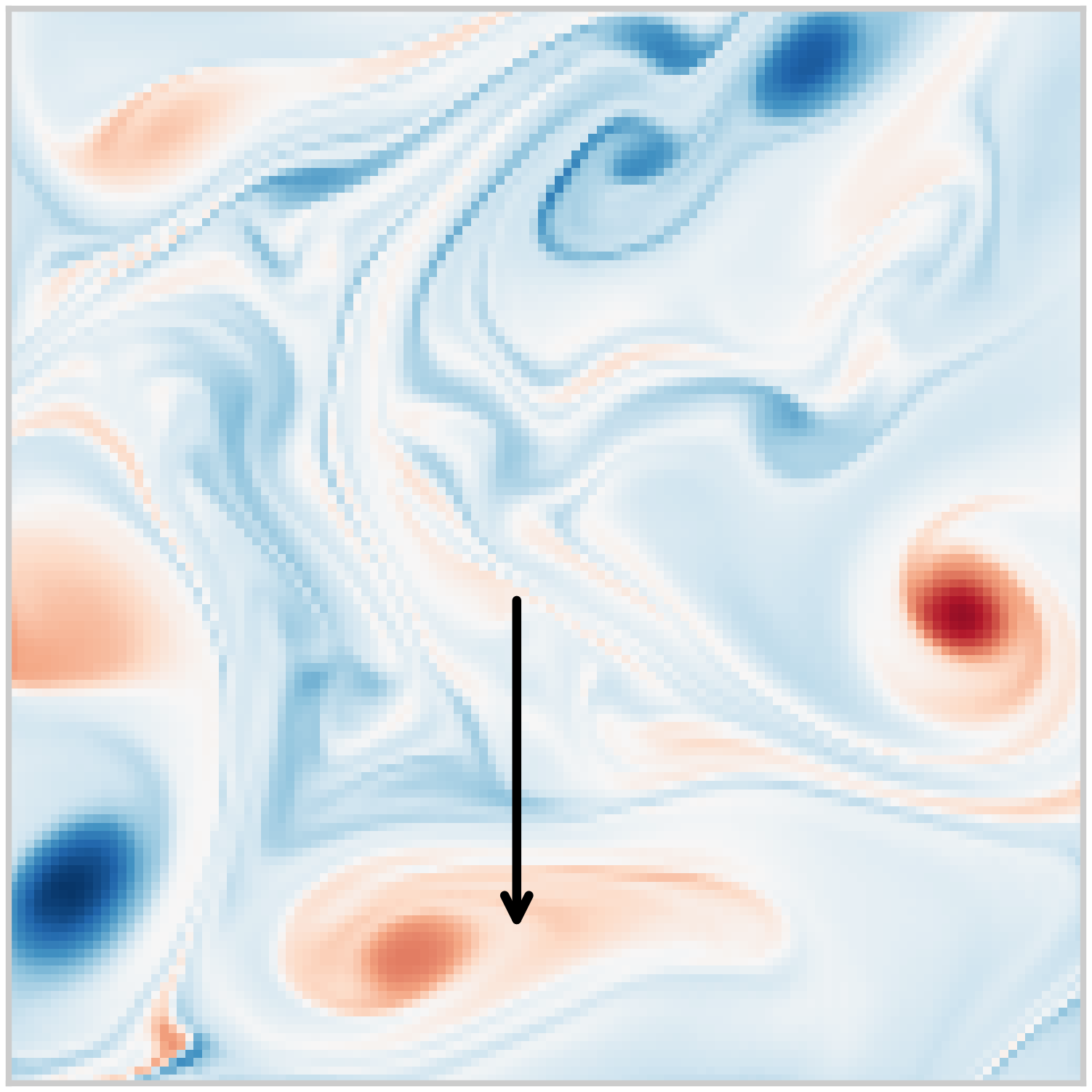}
    {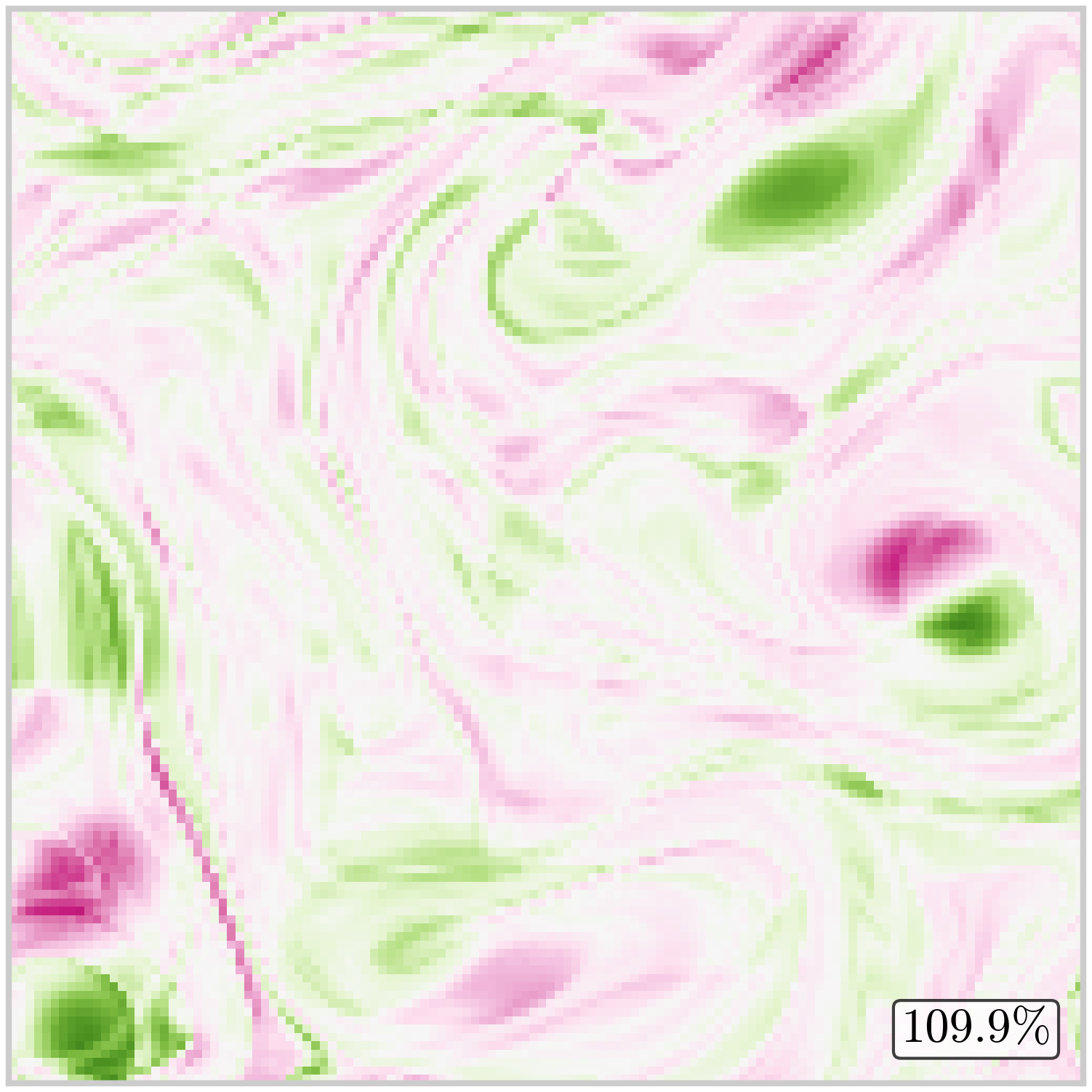}

    \vspace{1mm}

    \hspace*{\labw}%
    \includegraphics[width=0.32\textwidth]{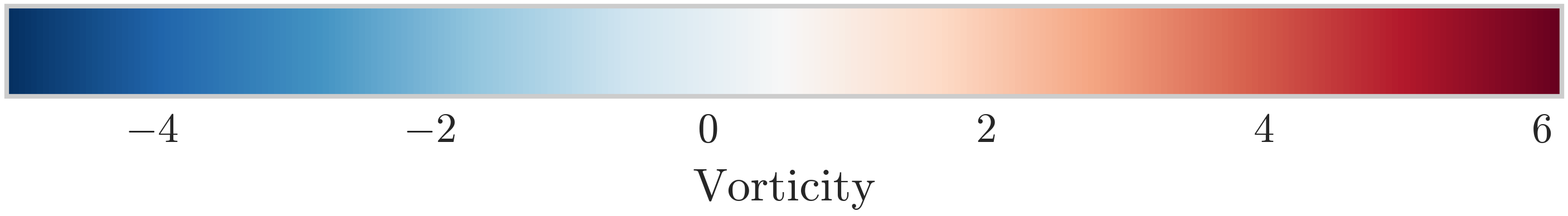}
    \hspace{4mm}
    \includegraphics[width=0.32\textwidth]{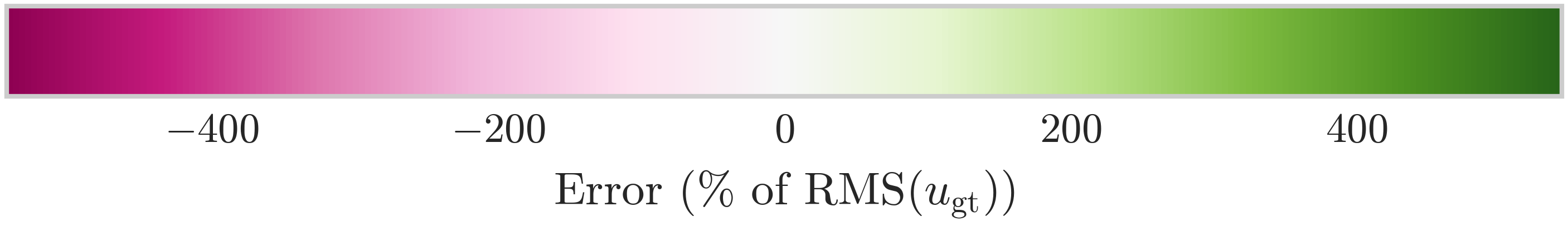}

    \caption{
    Qualitative autoregressive rollout comparison on the 2D Navier--Stokes benchmark.
    Columns show prediction horizons $\Delta t \in \{3,6,9,12\}$.
    For each model, the upper image in each cell is the predicted vorticity field and the lower image is the corresponding error map.
    MFC-UNet remains closer to the ground truth, while FNO and UNet accumulate larger errors over longer rollouts.
    }
    \label{fig:ns_appendix_rollout_panels}
\end{figure*}

\subsection{WeatherBench2}
\label{app:wb2_qualitative}

We provide additional WeatherBench2 rollout examples at step 50 for \(U10\), \(V10\), and \(T850\).
Each figure compares the ground truth, MFC-parabolic, and FNO at the training resolution \(64\times32\) and under zero-shot transfer to \(240\times121\).
These examples complement the cross-resolution metrics in Table~\ref{tab:wb2_crossres}.

We choose the displayed initial conditions using a median-improvement criterion, rather than selecting examples by visual inspection.
For each variable, we evaluate FNO and MFC-parabolic on test-set initial conditions at rollout step 50, using a stride of 28 timesteps.
For each initial condition, we compute
\[
\Delta
=
\mathrm{mean}_{\mathrm{lat}}\!\left(|u_{\mathrm{FNO}} - u_{\mathrm{true}}|\right)
-
\mathrm{mean}_{\mathrm{lat}}\!\left(|u_{\mathrm{MFC}} - u_{\mathrm{true}}|\right),
\]
where the mean is latitude-weighted.
Positive values indicate lower average absolute error for MFC-parabolic than for FNO.
For each variable, we show the initial condition whose score is closest to the median of the test-set distribution.
This makes the snapshots representative of typical behavior, rather than best-case or worst-case examples.

\begin{figure}[!t]
    \centering

    {\small \textbf{\(U10\), training resolution \(64\times32\)}}\\[0.25em]
    \begin{minipage}[t]{0.2\textwidth}
        \centering
        \includegraphics[width=\textwidth]{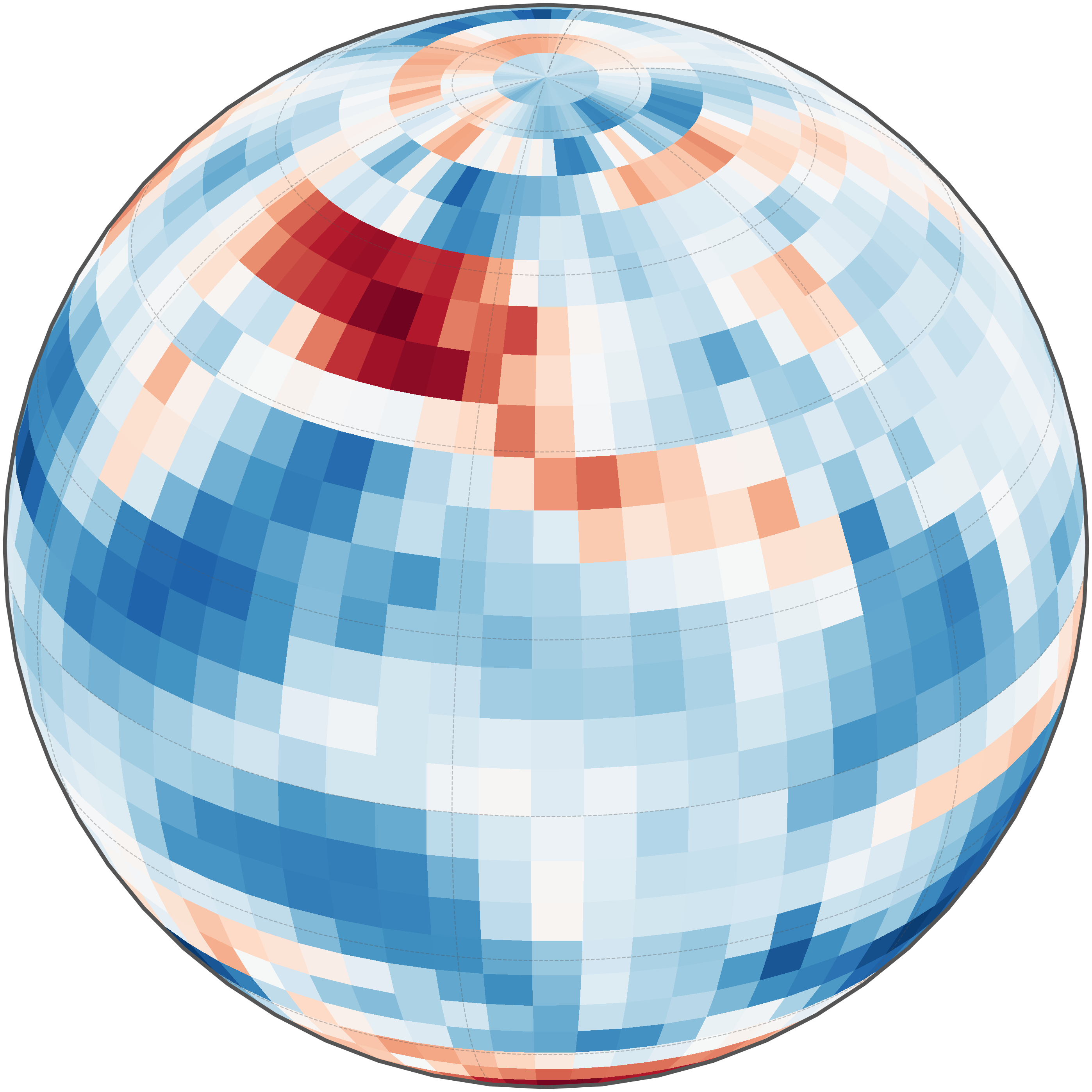}
        
        {\scriptsize Ground truth}
    \end{minipage}
    \hspace{0.008\textwidth}
    \begin{minipage}[t]{0.2\textwidth}
        \centering
        \includegraphics[width=\textwidth]{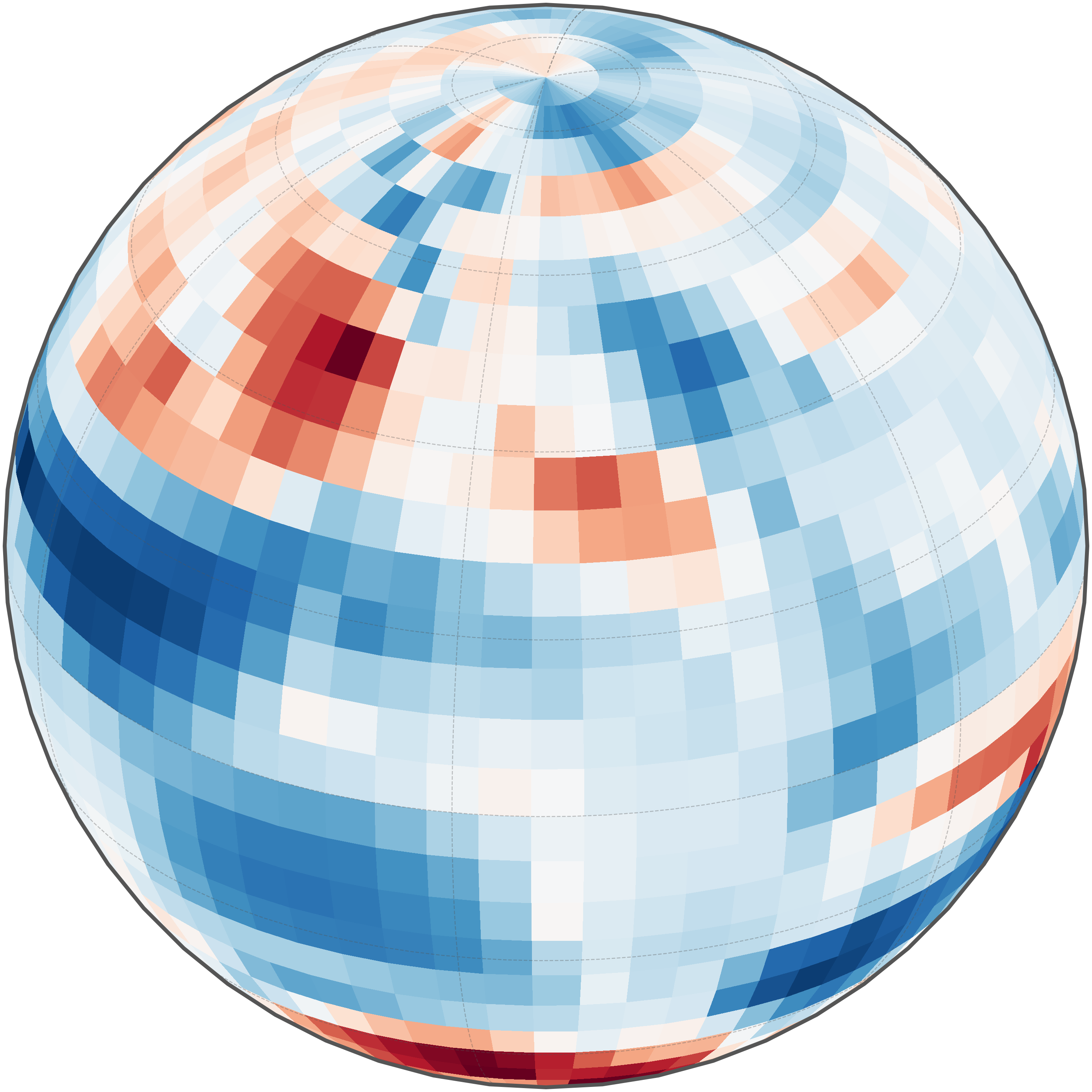}
        
        {\scriptsize MFC-resnet}
    \end{minipage}
    \hspace{0.008\textwidth}
    \begin{minipage}[t]{0.2\textwidth}
        \centering
        \includegraphics[width=\textwidth]{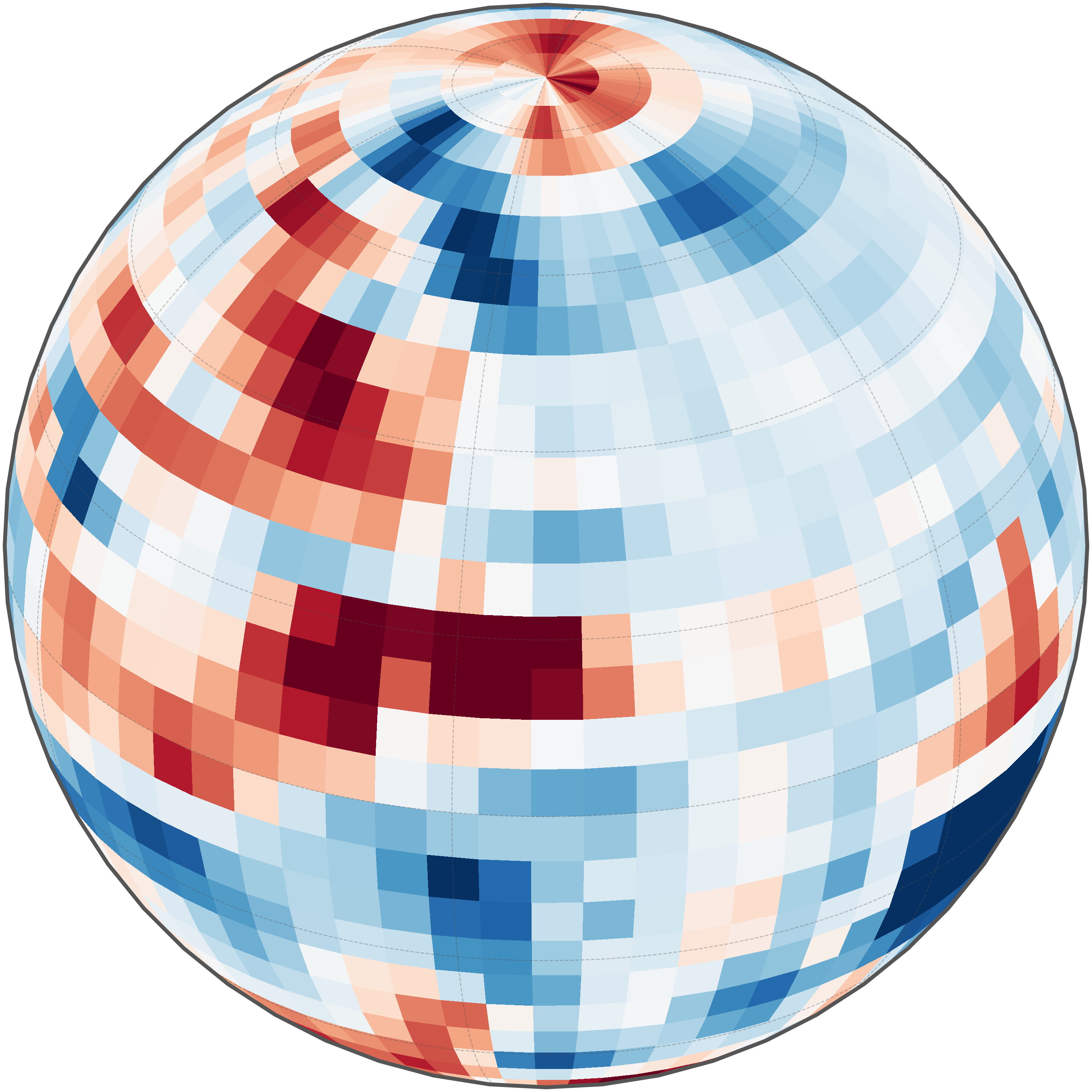}
        
        {\scriptsize FNO}
    \end{minipage}

    \vspace{0.2em}
    \includegraphics[width=0.30\textwidth]{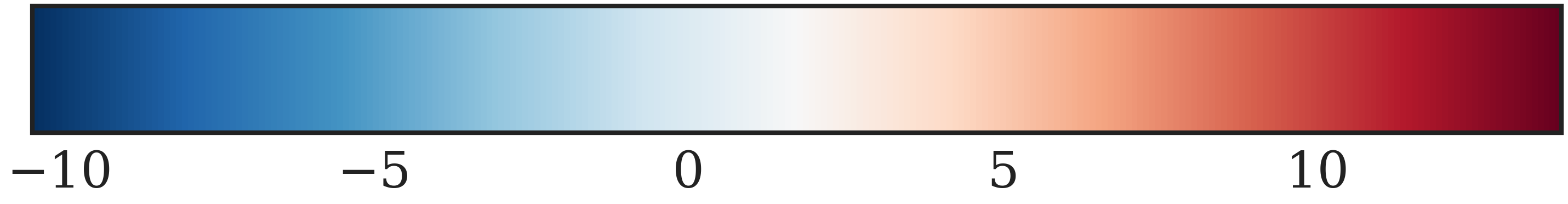}
    \vspace{0.7em}

    {\small \textbf{\(U10\), zero-shot transfer to \(240\times121\)}}\\[0.25em]
    \begin{minipage}[t]{0.2\textwidth}
        \centering
        \includegraphics[width=\textwidth]{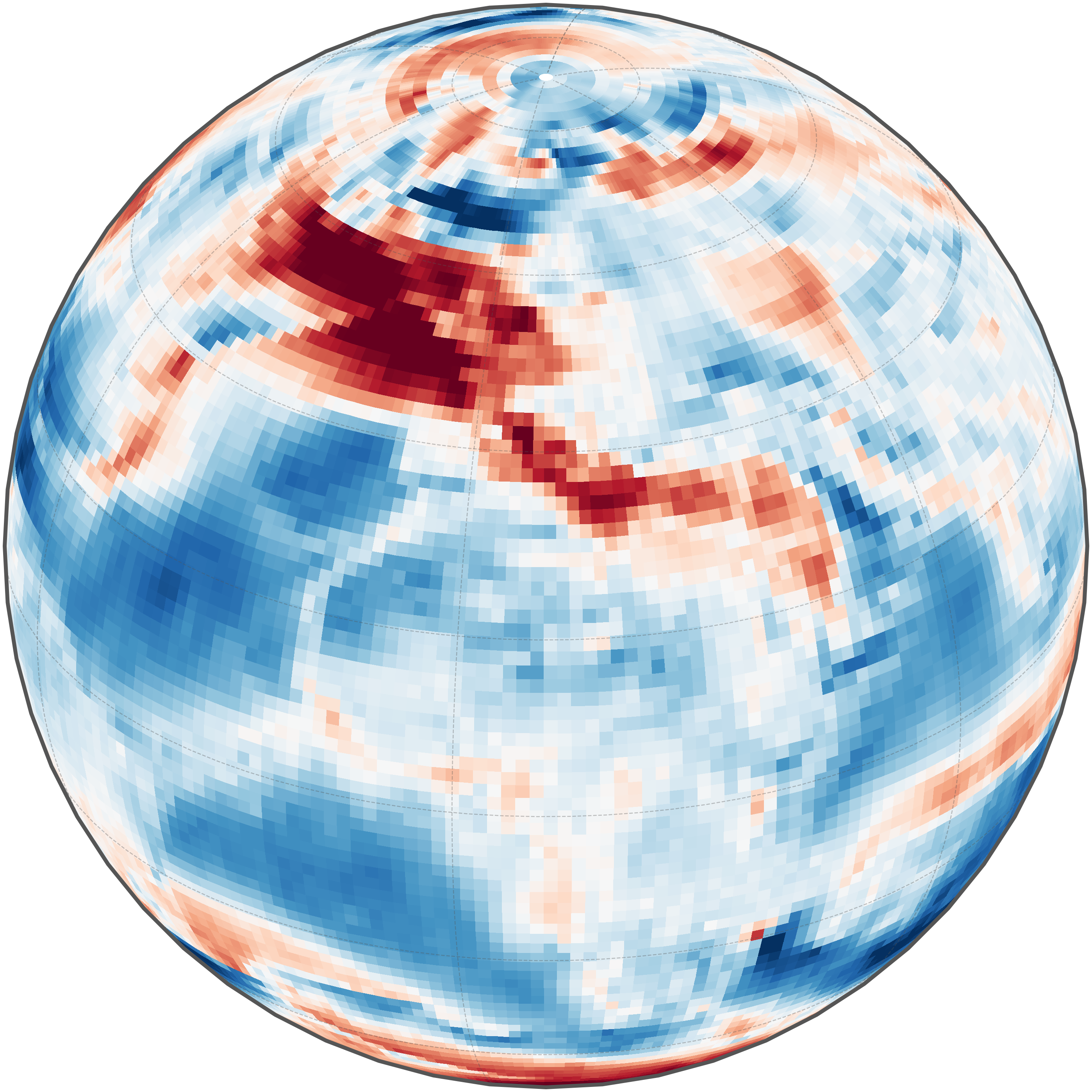}
        
        {\scriptsize Ground truth}
    \end{minipage}
    \hspace{0.008\textwidth}
    \begin{minipage}[t]{0.2\textwidth}
        \centering
        \includegraphics[width=\textwidth]{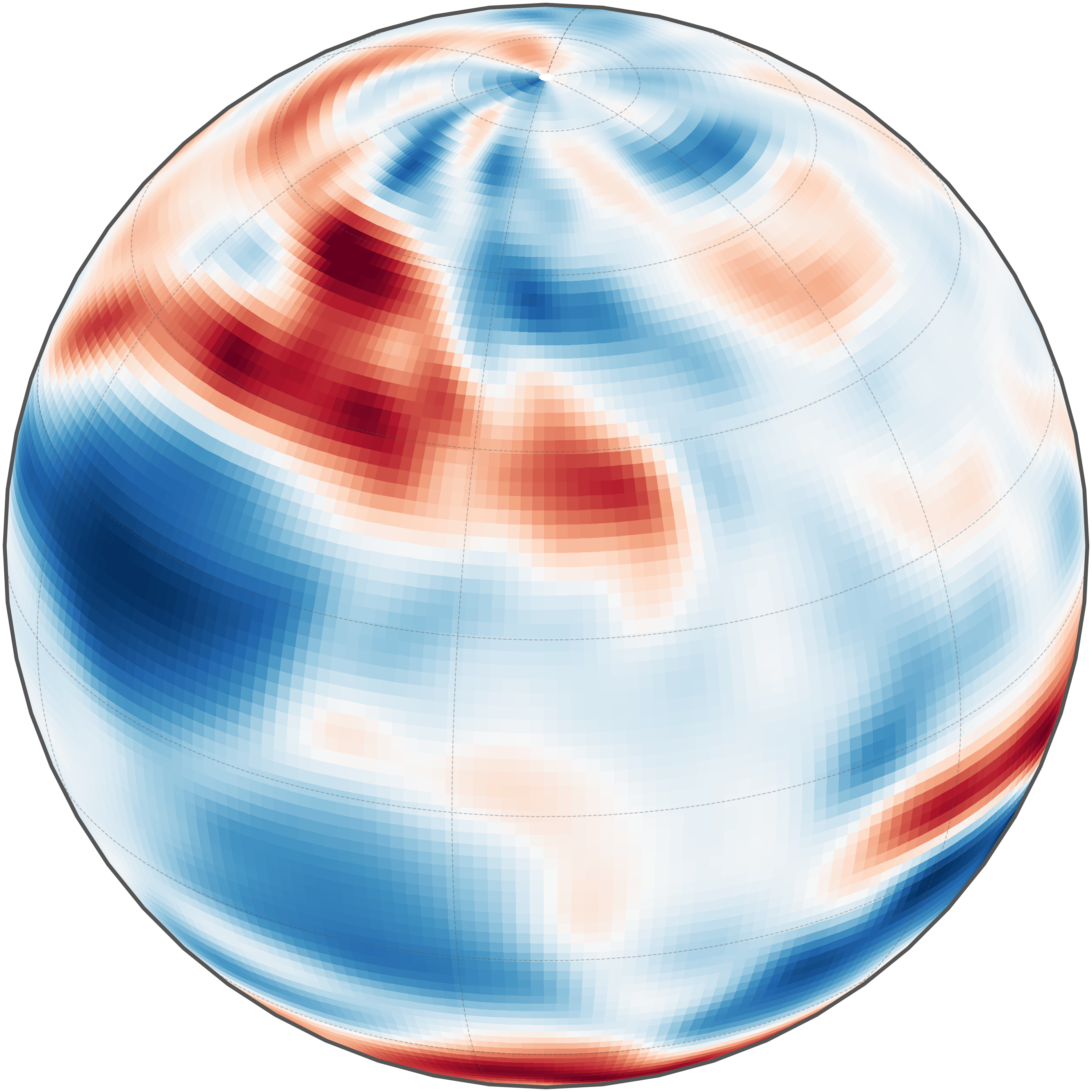}
        
        {\scriptsize MFC-resnet}
    \end{minipage}
    \hspace{0.008\textwidth}
    \begin{minipage}[t]{0.2\textwidth}
        \centering
        \includegraphics[width=\textwidth]{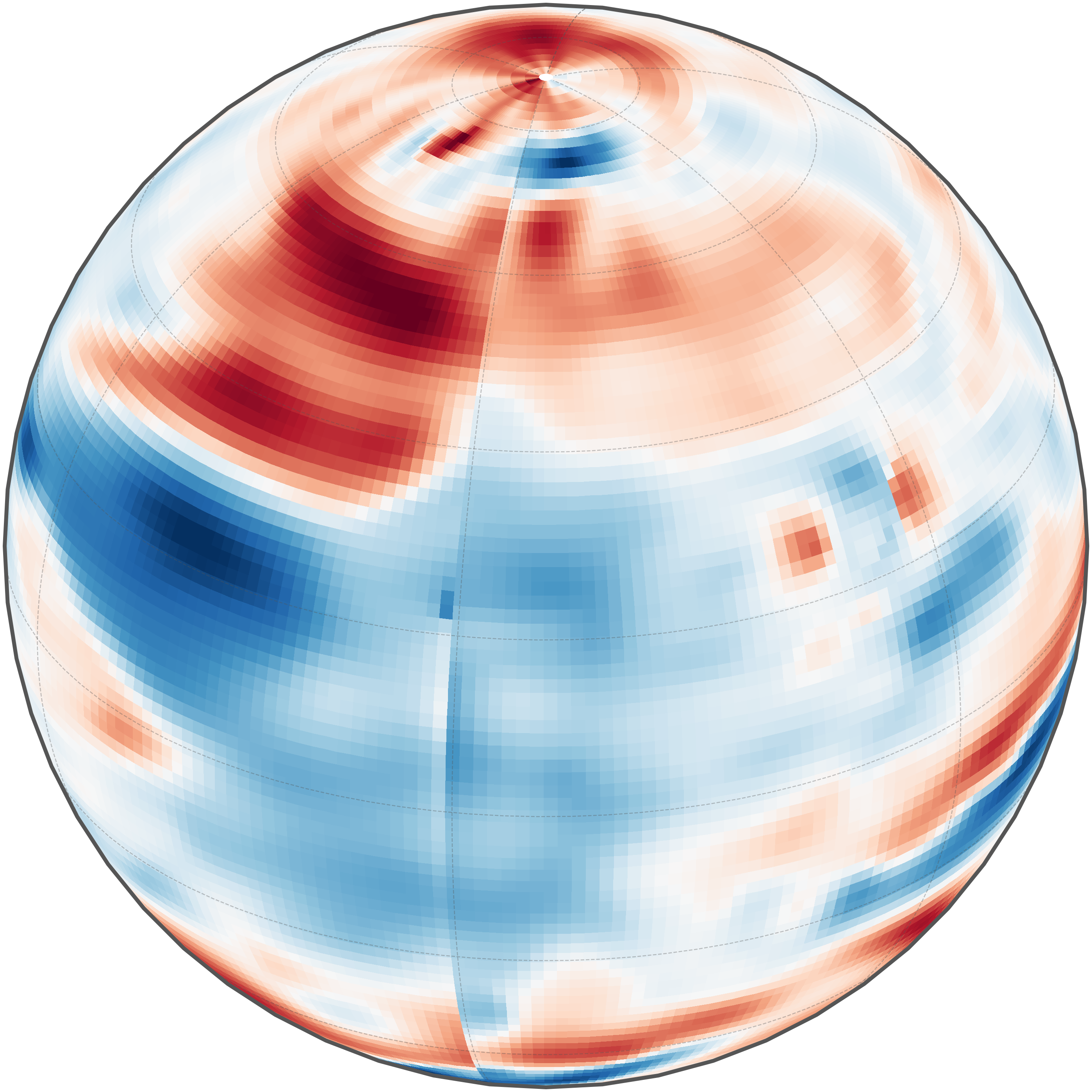}
        
        {\scriptsize FNO}
    \end{minipage}

    \vspace{0.2em}
    \includegraphics[width=0.30\textwidth]{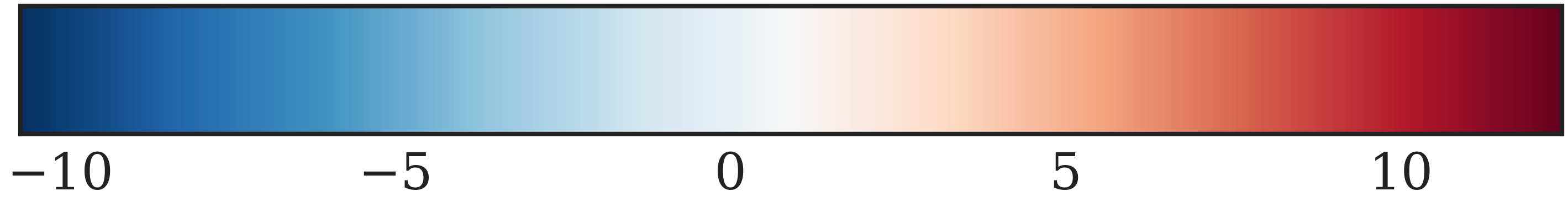}

    \caption{
    Representative \(U10\) rollout at step 50.
    The initial condition is selected by the median-improvement criterion described above.
    MFC-parabolic remains closer to the ground truth at both resolutions, while FNO shows stronger distortion after rollout.
    }
    \label{fig:appendix_wb2_u10}
\end{figure}

\begin{figure}[h]
    \centering

    {\small \textbf{\(V10\), training resolution \(64\times32\)}}\\[0.25em]
    \begin{minipage}[t]{0.2\textwidth}
        \centering
        \includegraphics[width=\textwidth]{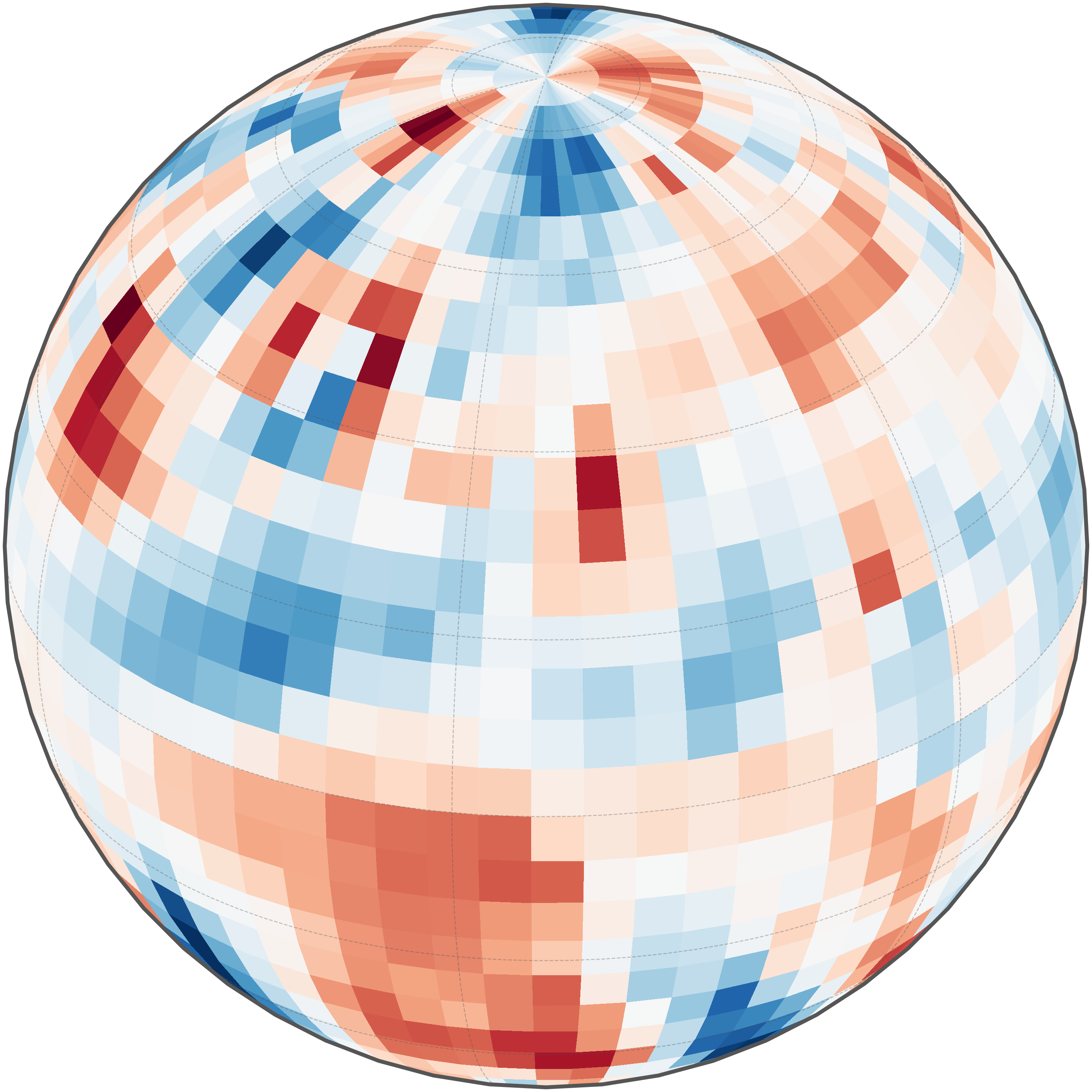}
        
        {\scriptsize Ground truth}
    \end{minipage}
    \hspace{0.008\textwidth}
    \begin{minipage}[t]{0.2\textwidth}
        \centering
        \includegraphics[width=\textwidth]{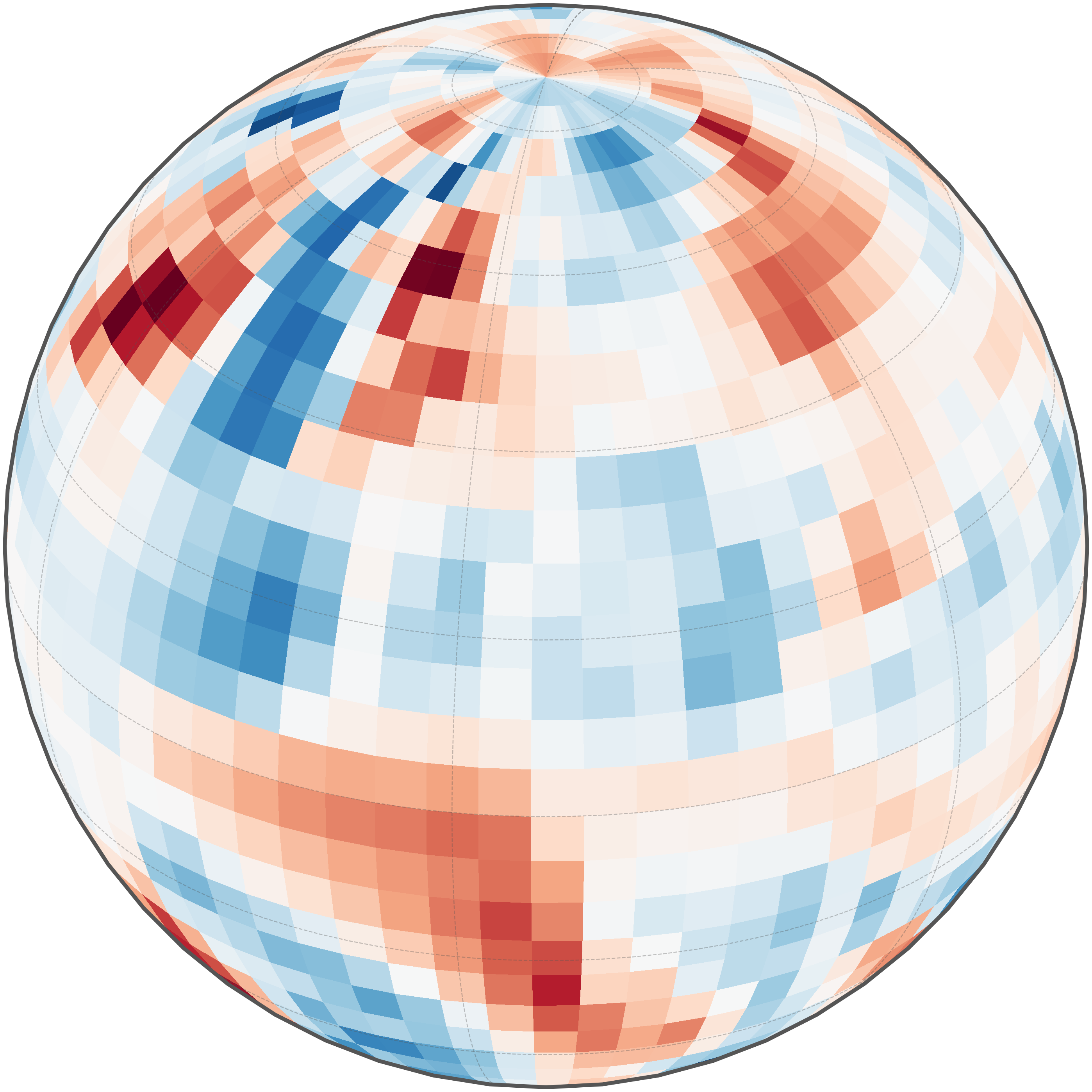}
        
        {\scriptsize MFC-resnet}
    \end{minipage}
    \hspace{0.008\textwidth}
    \begin{minipage}[t]{0.2\textwidth}
        \centering
        \includegraphics[width=\textwidth]{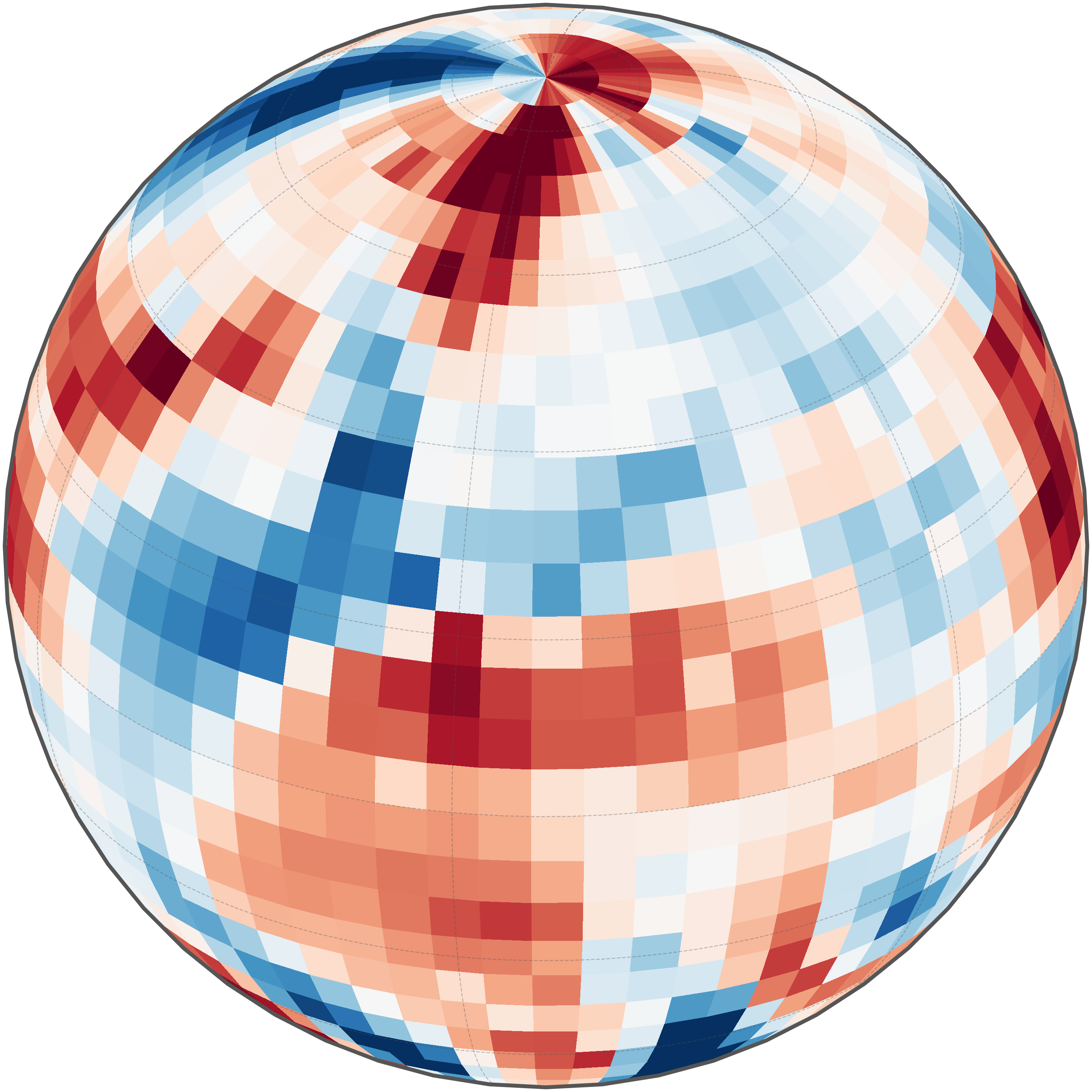}
        
        {\scriptsize FNO}
    \end{minipage}

    \vspace{0.2em}
    \includegraphics[width=0.30\textwidth]{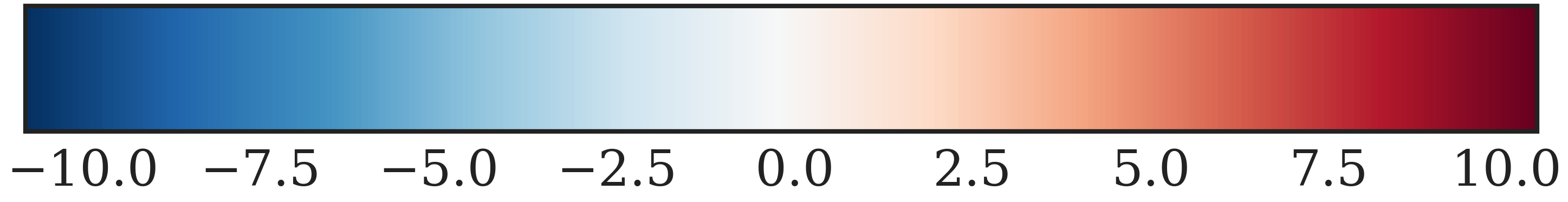}
    \vspace{0.7em}

    {\small \textbf{\(V10\), zero-shot transfer to \(240\times121\)}}\\[0.25em]
    \begin{minipage}[t]{0.2\textwidth}
        \centering
        \includegraphics[width=\textwidth]{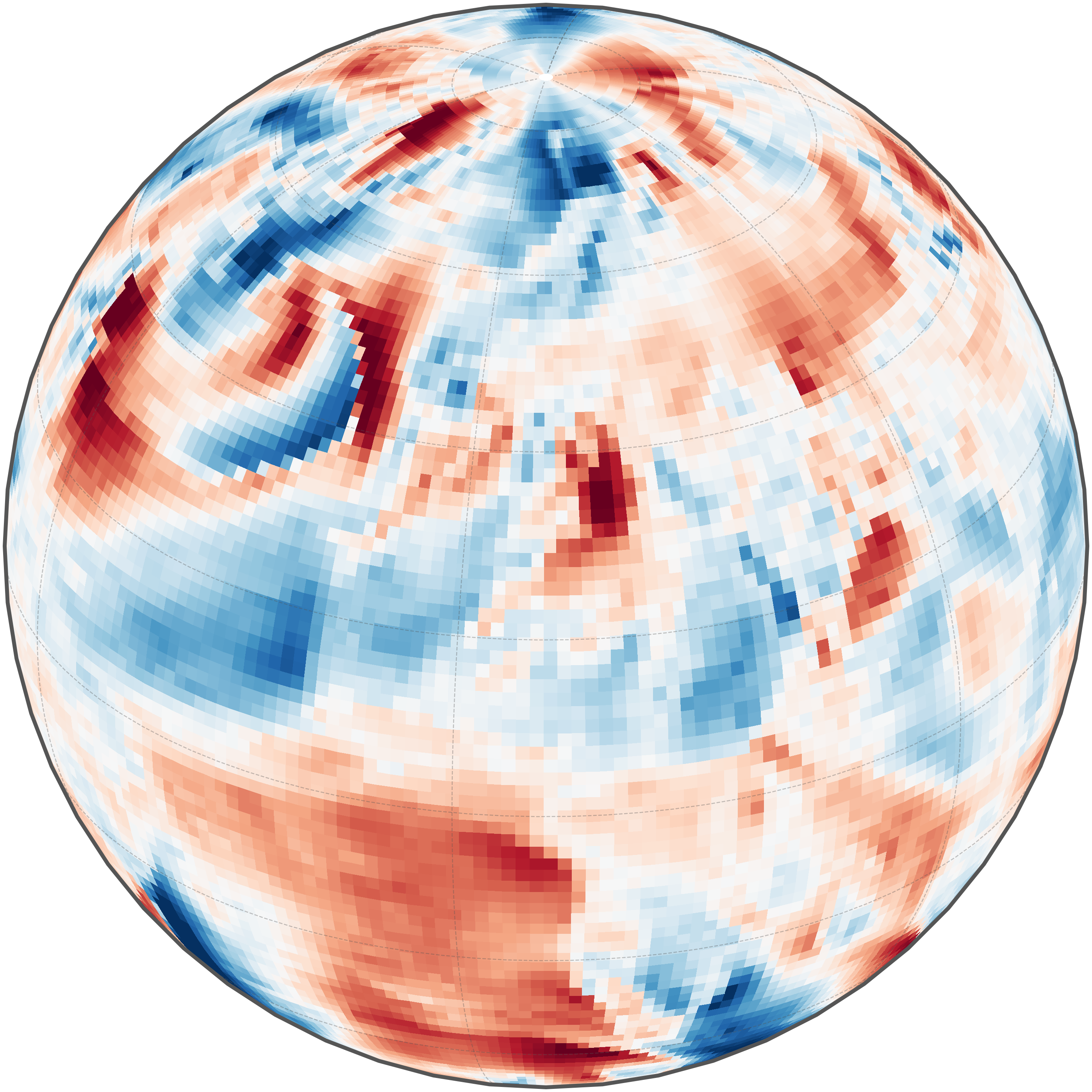}
        
        {\scriptsize Ground truth}
    \end{minipage}
    \hspace{0.008\textwidth}
    \begin{minipage}[t]{0.2\textwidth}
        \centering
        \includegraphics[width=\textwidth]{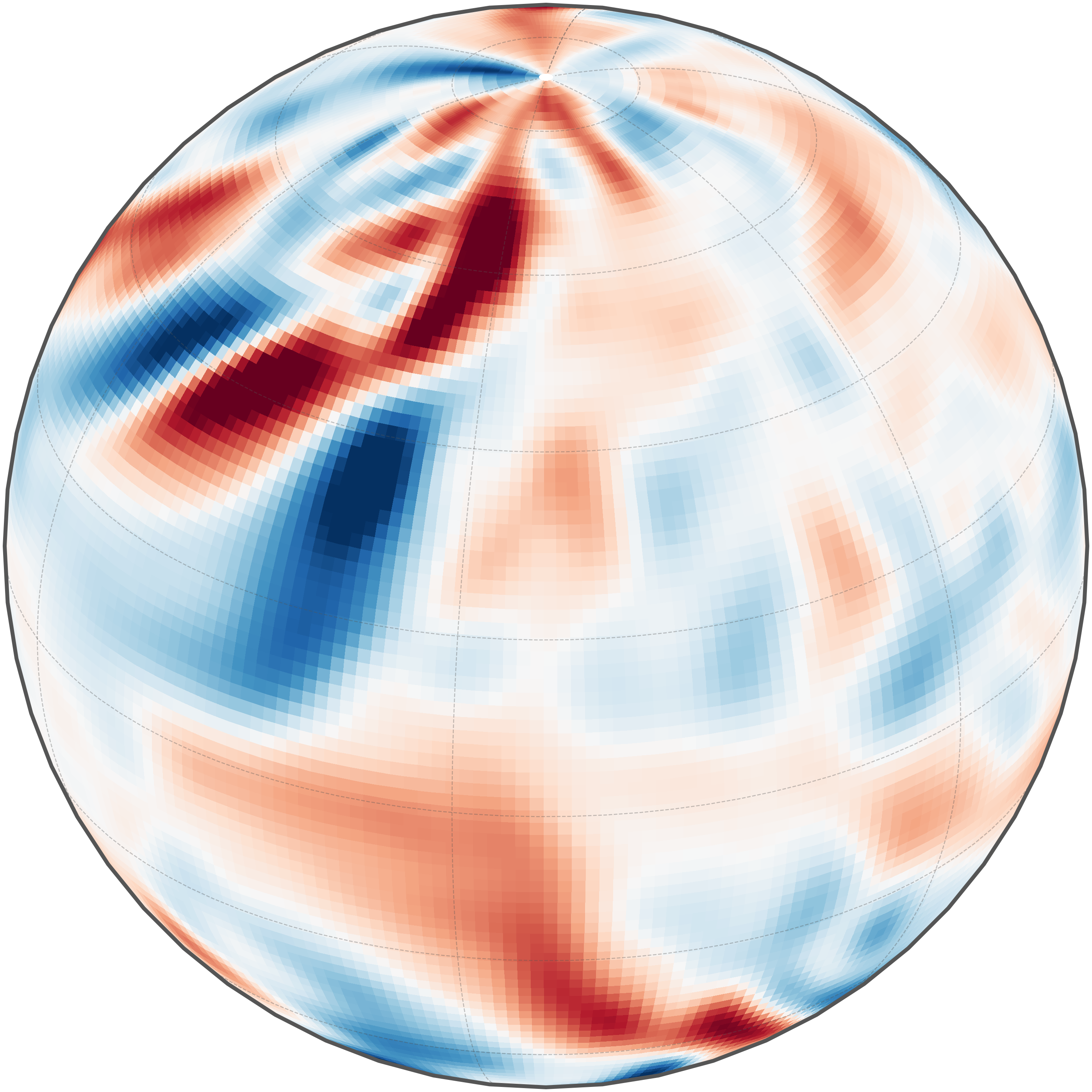}
        
        {\scriptsize MFC-resnet}
    \end{minipage}
    \hspace{0.008\textwidth}
    \begin{minipage}[t]{0.2\textwidth}
        \centering
        \includegraphics[width=\textwidth]{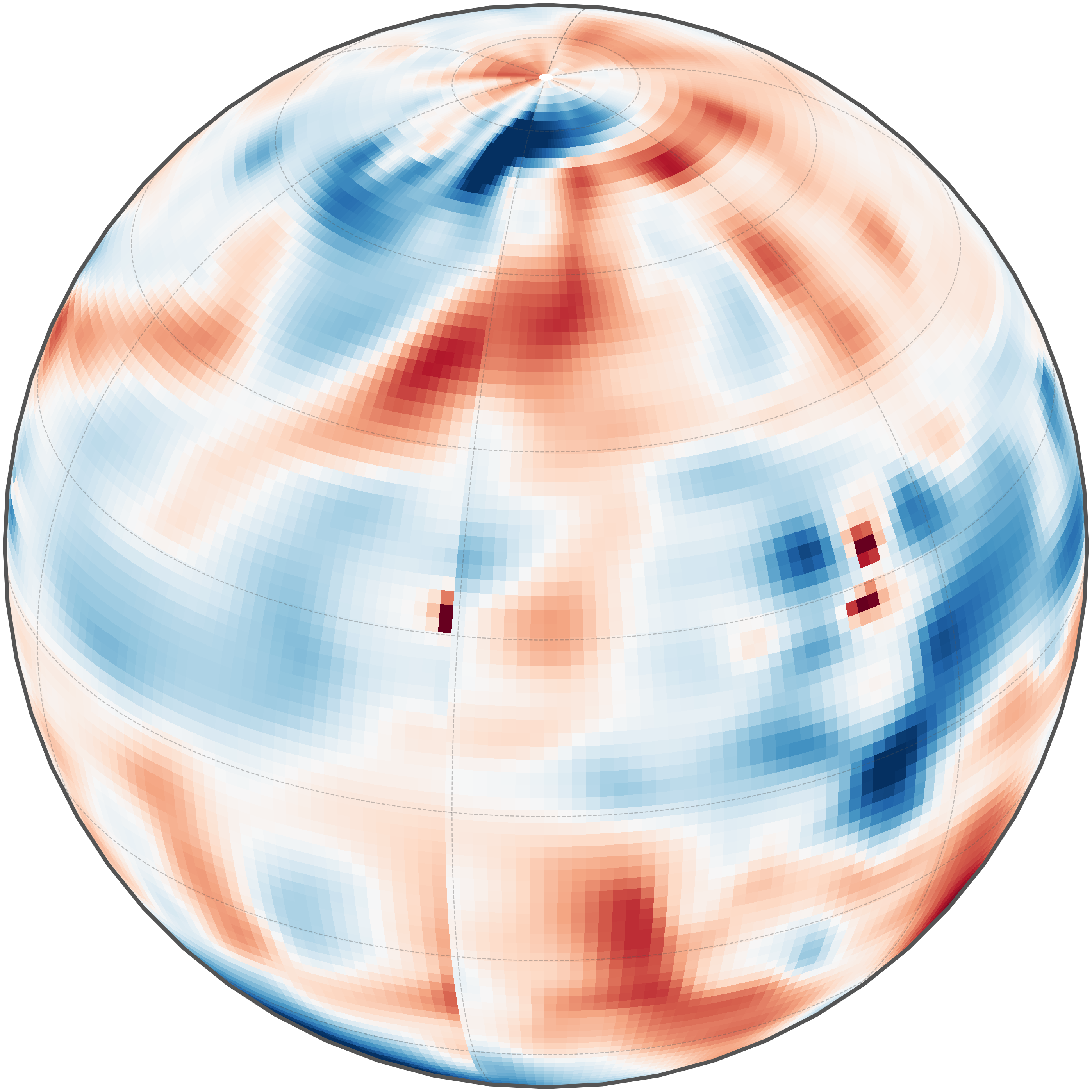}
        
        {\scriptsize FNO}
    \end{minipage}

    \vspace{0.2em}
    \includegraphics[width=0.30\textwidth]{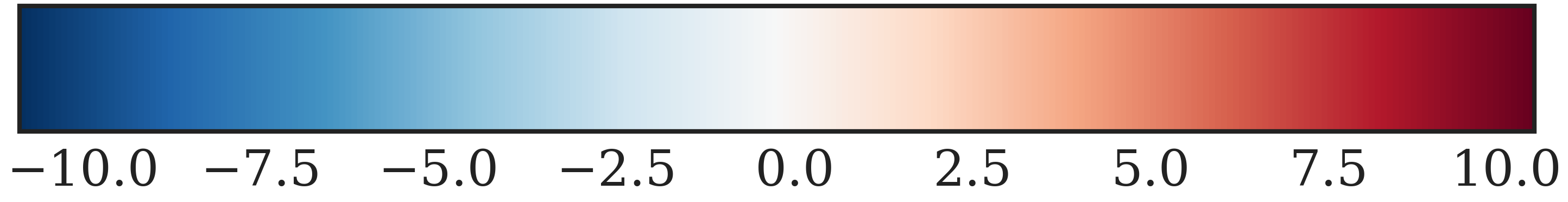}

    \caption{
    Representative \(V10\) rollout at step 50.
    The initial condition is selected by the median-improvement criterion described above.
    MFC-parabolic better preserves the spatial structure across both the training and zero-shot resolutions.
    }
    \label{fig:appendix_wb2_v10}
\end{figure}

\begin{figure}[h]
    \centering

    {\small \textbf{\(T850\), training resolution \(64\times32\)}}\\[0.25em]
    \begin{minipage}[t]{0.2\textwidth}
        \centering
        \includegraphics[width=\textwidth]{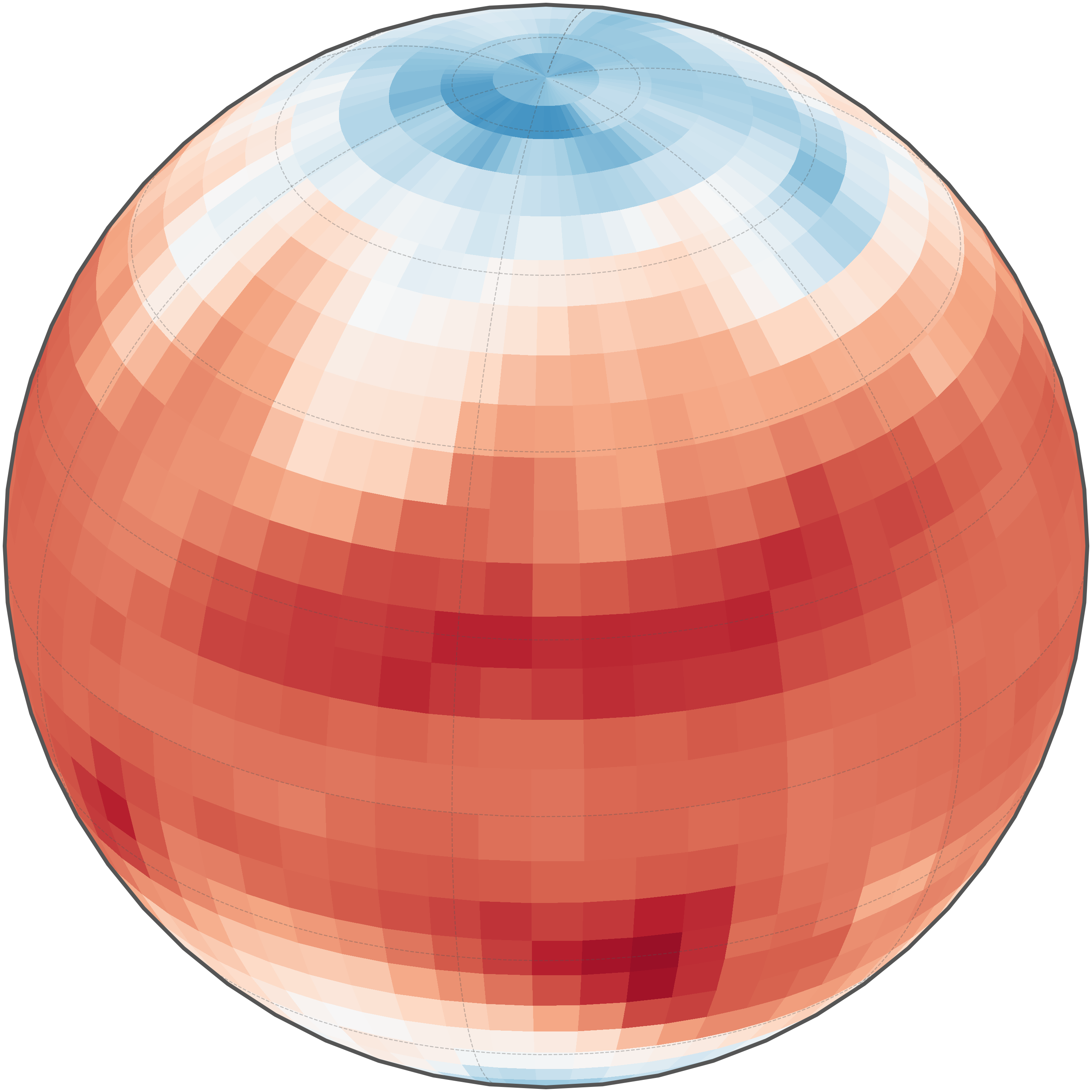}
        
        {\scriptsize Ground truth}
    \end{minipage}
    \hspace{0.008\textwidth}
    \begin{minipage}[t]{0.2\textwidth}
        \centering
        \includegraphics[width=\textwidth]{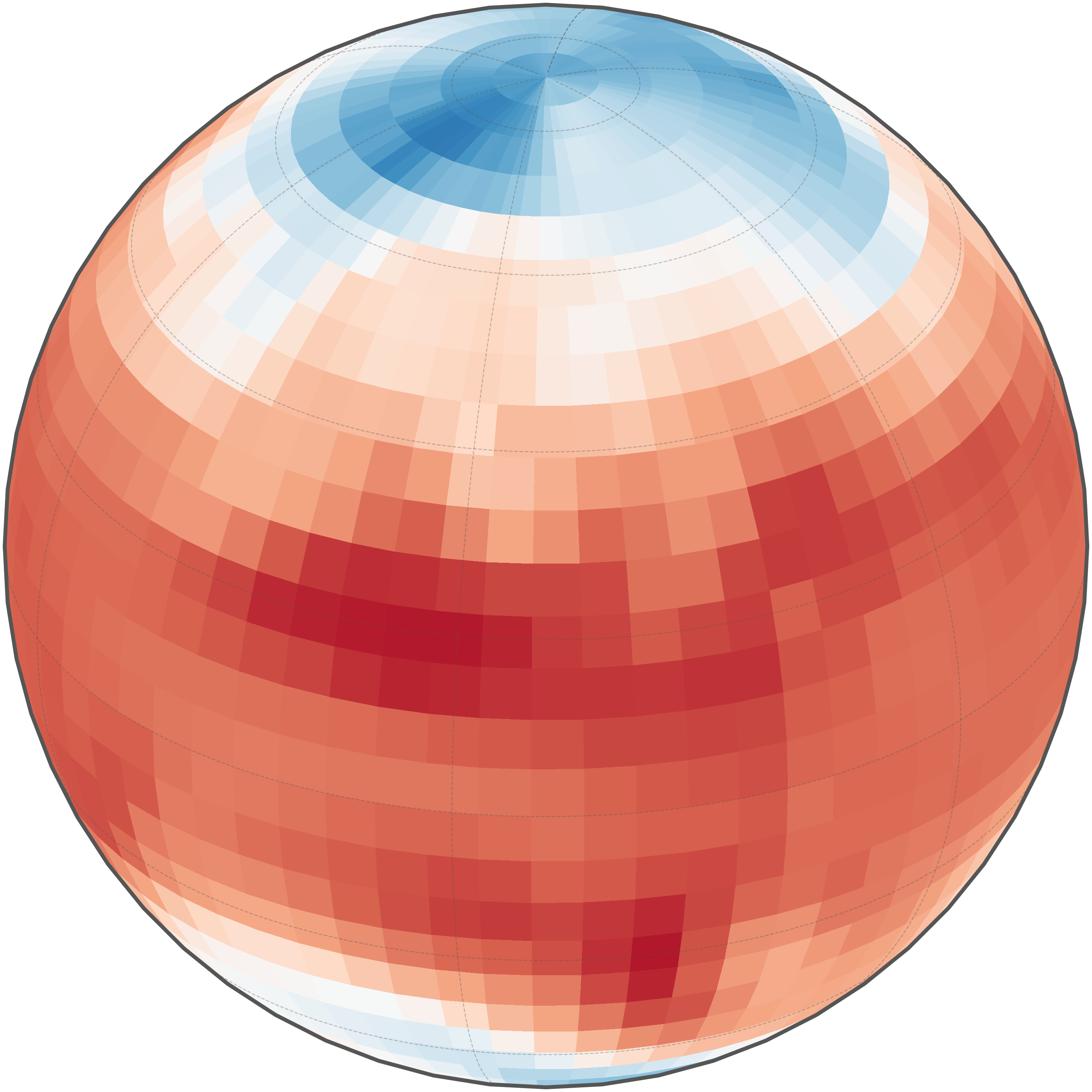}
        
        {\scriptsize MFC-resnet}
    \end{minipage}
    \hspace{0.008\textwidth}
    \begin{minipage}[t]{0.2\textwidth}
        \centering
        \includegraphics[width=\textwidth]{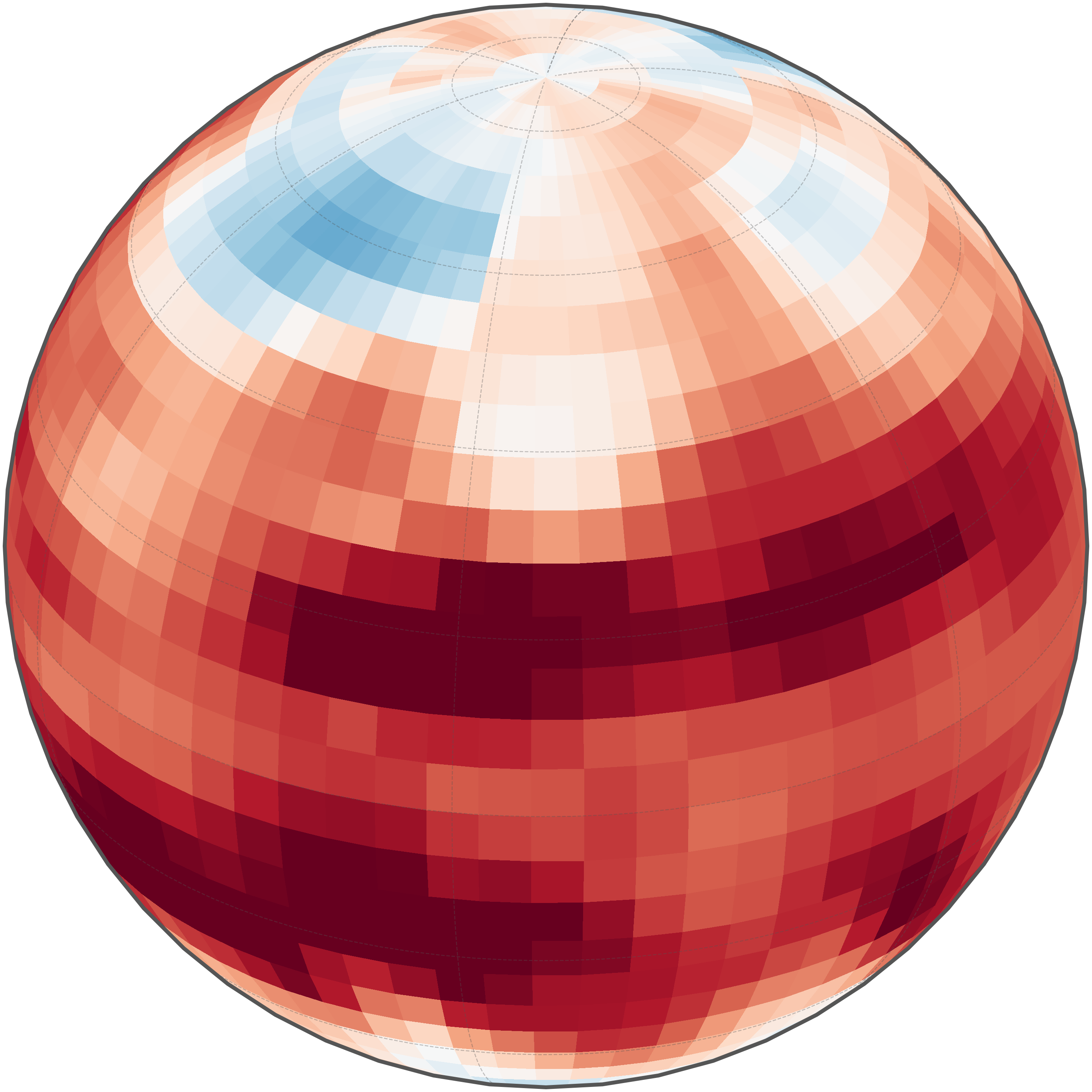}
        
        {\scriptsize FNO}
    \end{minipage}

    \vspace{0.2em}
    \includegraphics[width=0.30\textwidth]{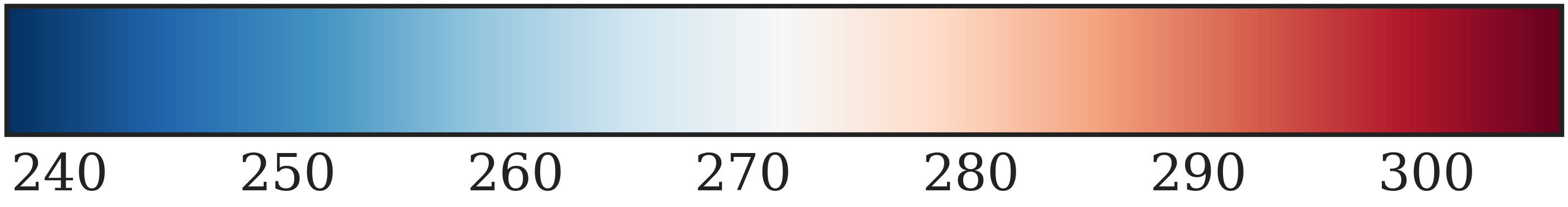}
    \vspace{0.7em}

    {\small \textbf{\(T850\), zero-shot transfer to \(240\times121\)}}\\[0.25em]
    \begin{minipage}[t]{0.2\textwidth}
        \centering
        \includegraphics[width=\textwidth]{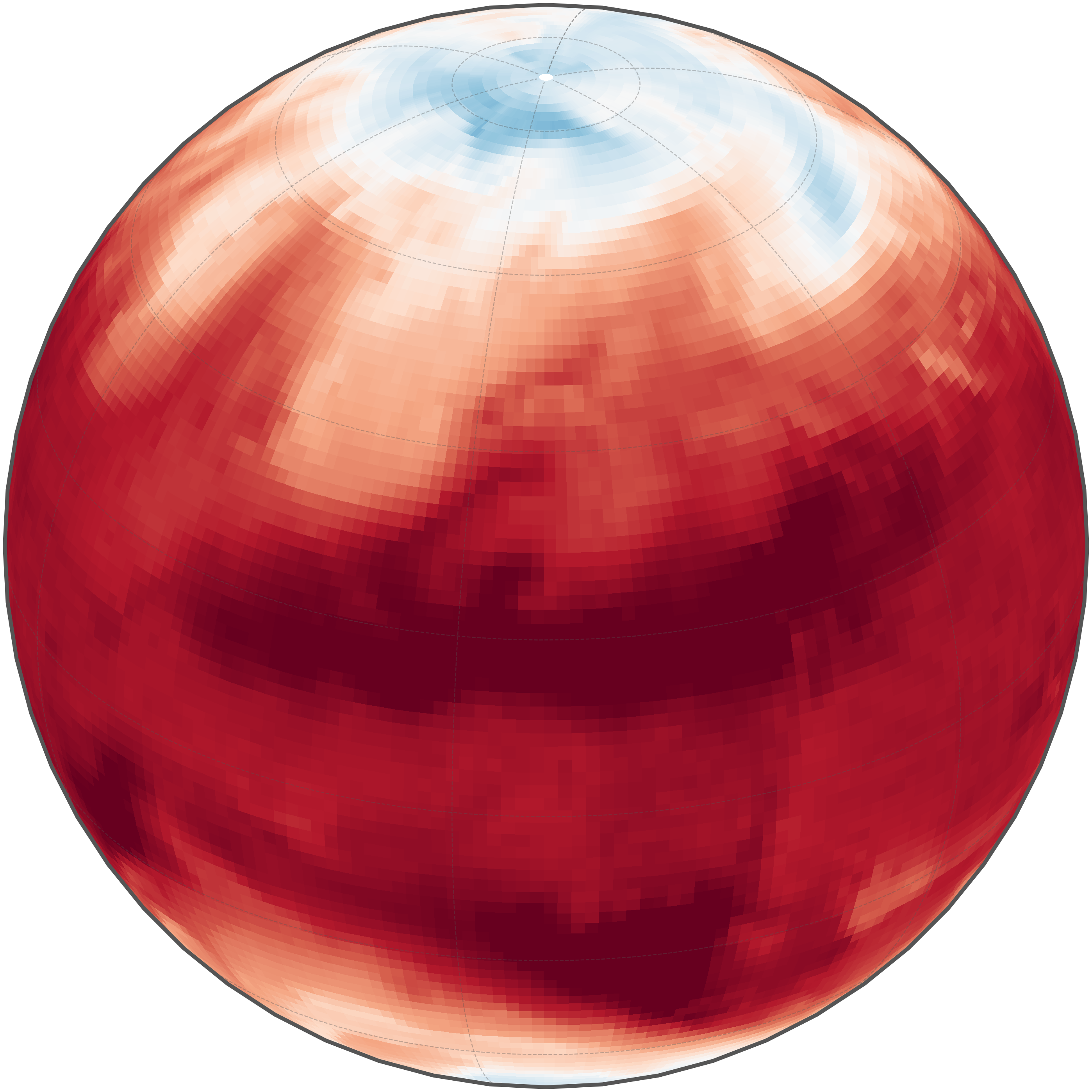}
        
        {\scriptsize Ground truth}
    \end{minipage}
    \hspace{0.008\textwidth}
    \begin{minipage}[t]{0.2\textwidth}
        \centering
        \includegraphics[width=\textwidth]{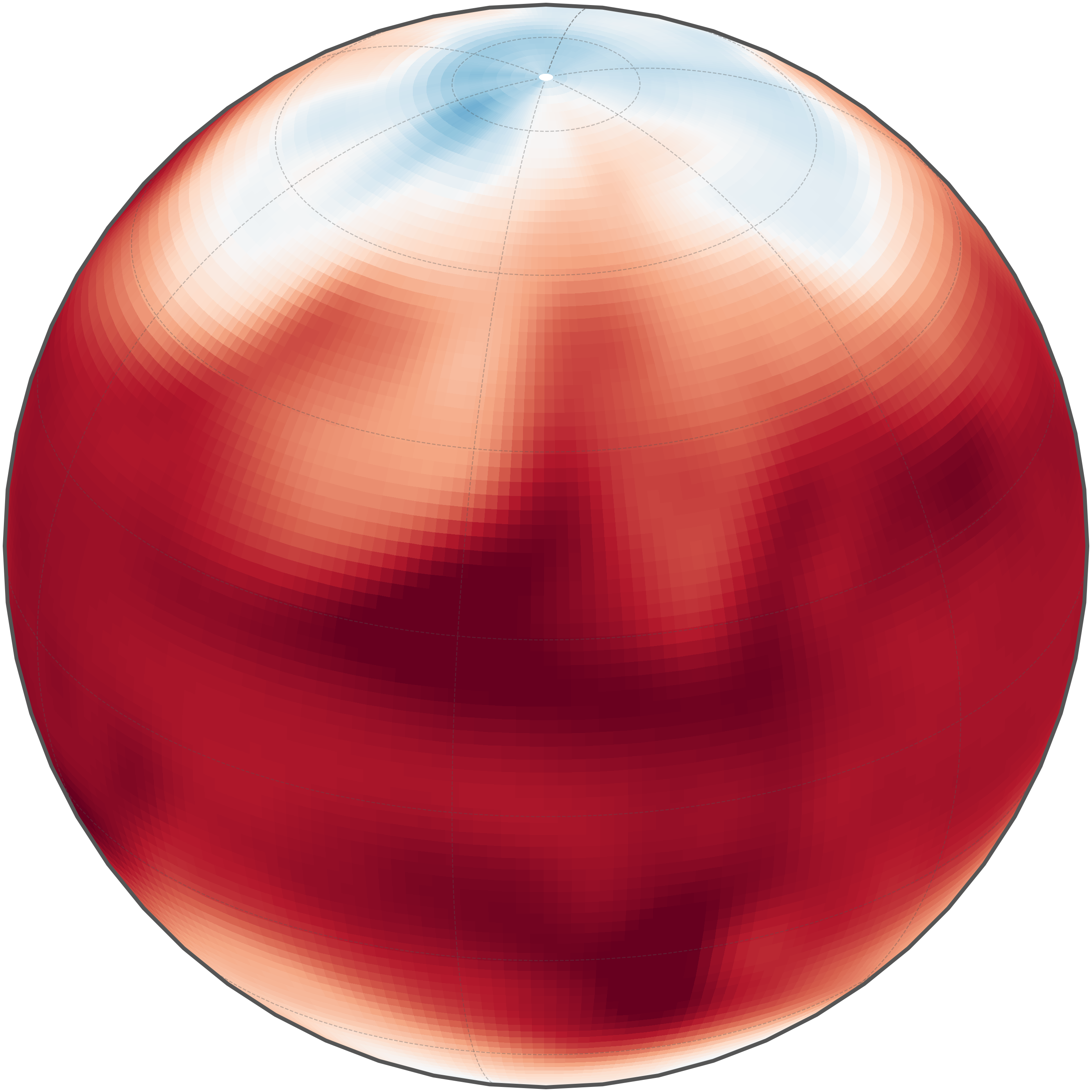}
        
        {\scriptsize MFC-resnet}
    \end{minipage}
    \hspace{0.008\textwidth}
    \begin{minipage}[t]{0.2\textwidth}
        \centering
        \includegraphics[width=\textwidth]{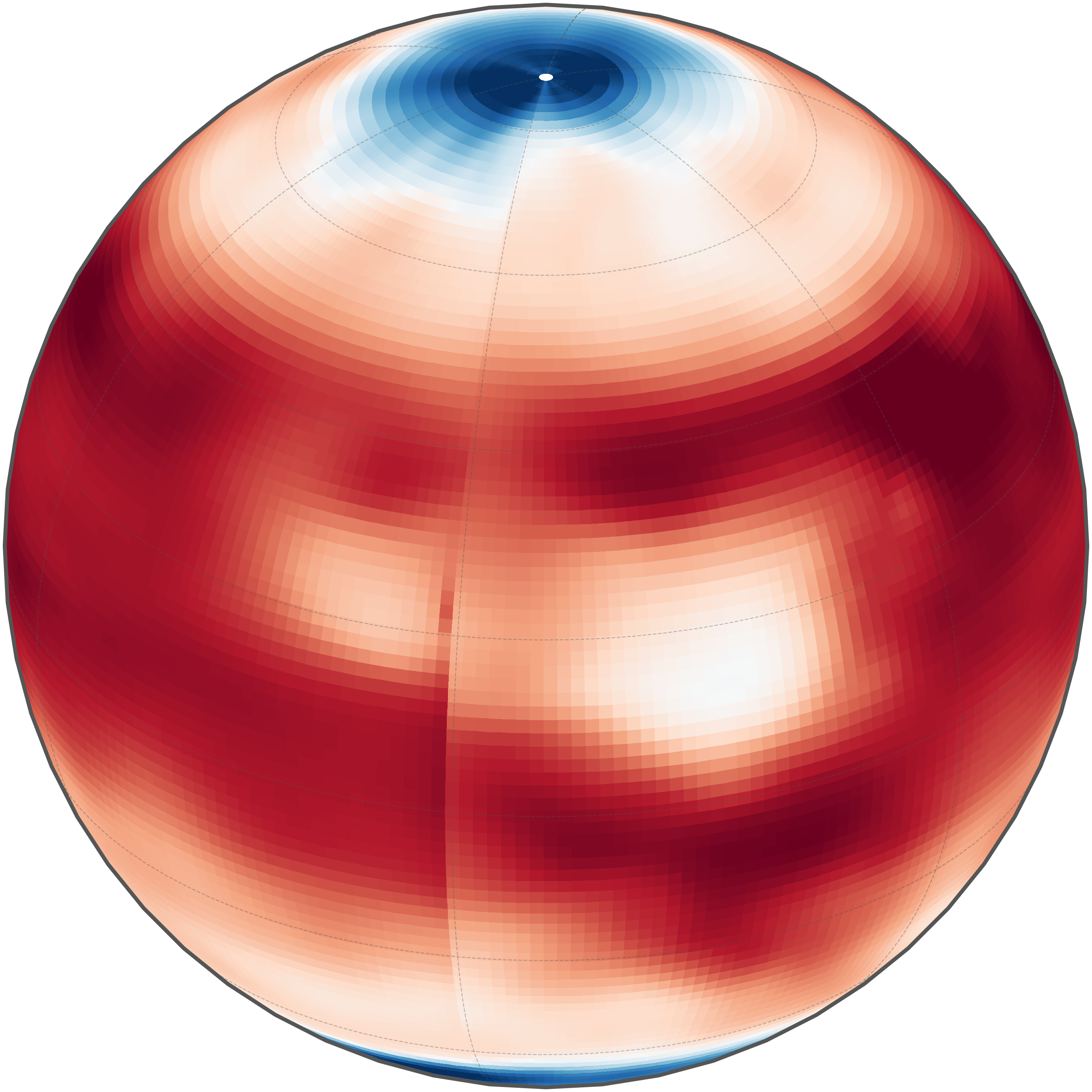}
        
        {\scriptsize FNO}
    \end{minipage}

    \vspace{0.2em}
    \includegraphics[width=0.30\textwidth]{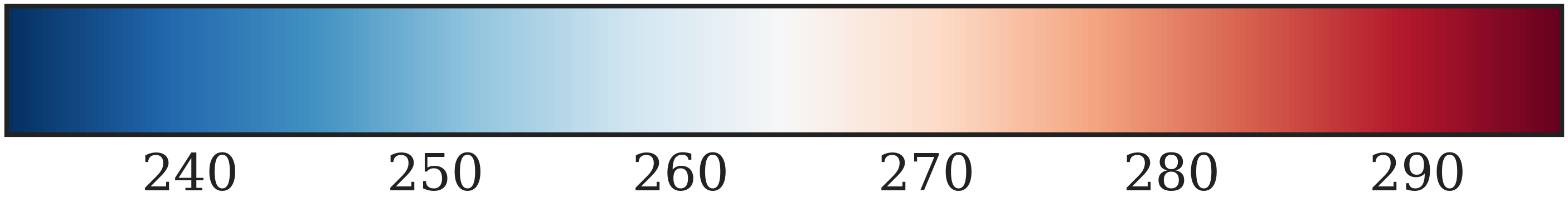}

    \caption{
    Representative \(T850\) rollout at step 50.
    The initial condition is selected by the median-improvement criterion described above.
    The gap between MFC-parabolic and FNO becomes more visible under zero-shot transfer.
    }
    \label{fig:appendix_wb2_t850}
\end{figure}

\newpage
\subsection{Training and Architecture Details}
\label{app:training-details}

\paragraph{Navier--Stokes dataset.}
The Navier--Stokes experiment uses the \(Re=5000\) Kolmogorov-flow split from~\cite{kovachki2023data}. 
The data consist of 100 trajectories with 500 time steps on a \(128\times128\) grid. 
The dynamics are governed by the incompressible Navier--Stokes equations on the periodic domain \(D=(0,2\pi)^2\),
\[
\partial_t \mathbf{u}+\mathbf{u}\cdot\nabla \mathbf{u}
-\nu\Delta\mathbf{u}
=
-\nabla p+\sin(my)\hat{\mathbf{x}},
\qquad
\nabla\cdot\mathbf{u}=0 .
\]
Following common practice for this benchmark, we predict scalar vorticity rather than velocity,
\[
\omega=(\nabla\times\mathbf{u})\cdot\hat{\mathbf{z}}
=
\partial_x u_y-\partial_y u_x .
\]
All models are trained with teacher forcing to predict the next vorticity field from a context window of 10 previous frames, and are evaluated autoregressively.

\begin{table}[h]
\centering
\caption{Training configuration for the Navier--Stokes Kolmogorov-flow experiment.}
\label{tab:ns-training}
\begin{tabular}{ll}
\toprule
\textbf{Setting} & \textbf{Value} \\
\midrule
Resolution & $128 \times 128$ \\
Context frames & 10 \\
Prediction & Single-step teacher-forced \\
Optimizer & Adam \\
Learning rate & $5 \times 10^{-4}$ \\
Scheduler & StepLR, step size \(50\), \(\gamma=0.5\) \\
Loss & MSE \\
Epochs & 200 \\
Batch size & 16 \\
Gradient clipping & 1.0 \\
Seed & 1 \\
\midrule
\multicolumn{2}{l}{\textbf{Model architectures (\({\sim}1.1\)M params each)}}\\
\midrule
FNO & 1.18 M\\ 
F-FNO & 1.12 M \\ 
UNet & 1.18 M\\ 
MFC-UNet & 1.15 M \\ 
\bottomrule
\end{tabular}
\end{table}

\begin{table}[h]
\centering
\caption{Training configuration for the WeatherBench2 experiment (ERA5; U10, V10, T850).}
\label{tab:wb2-training}
\begin{tabular}{ll}
\toprule
\textbf{Setting} & \textbf{Value} \\
\midrule
Training resolution & $64 \times 32$ \\
Zero-shot test resolution & $240 \times 121$ \\
History frames & 4 \\
Prediction & Single-step \\
Optimizer & Adam \\
Learning rate & $10^{-3}$ \\
Scheduler & Cosine annealing \\
Loss & Latitude-weighted MSE \\
Epochs & ${\sim}400$ \\
Batch size & 32 \\
Seed & 42 \\
\midrule
\multicolumn{2}{l}{\textbf{Model architectures (${\sim}500$K params each)}} \\
\midrule
FNO & 513K params \\
AFNO & 513K params \\
MFC-ResNet & 498K params \\
\bottomrule
\end{tabular}
\end{table}

\paragraph{Burgers dataset.}
The Burgers experiment uses the 2D viscous Burgers equation
\[
\partial_t u + (u\cdot\nabla)u = \nu \Delta u
\]
on the periodic domain \([0,2\pi]^2\). 
The training set contains 200 trajectories at \(64\times64\) resolution with
\(\nu\in[10^{-2},2{\cdot}10^{-2}]\). 
We evaluate on four test sets of 50 trajectories each: in-distribution, OOD viscosity with
\(\nu\in[10^{-3},7{\cdot}10^{-3}]\), OOD resolution at \(128\times128\), and the combined OOD viscosity/resolution setting. 
All models use 5 context frames and are trained with a 5-step unrolled loss, then evaluated with 100-step autoregressive rollouts.

\begin{table}[h]
\centering
\caption{Training configuration for the 2D Burgers experiment.}
\label{tab:burgers-training}
\begin{tabular}{ll}
\toprule
\textbf{Setting} & \textbf{Value} \\
\midrule
Training resolution & $64 \times 64$ \\
Cross-resolution test & $128 \times 128$ \\
Context frames & 5 \\
Unroll steps (training) & 5 \\
Optimizer & Adam \\
Learning rate & $5 \times 10^{-4}$ \\
Scheduler & Cosine annealing ($\mathrm{lr}_{\min}=10^{-6}$) \\
Loss & Relative \(L^2\) (5-step unrolled) \\
Epochs & 300 \\
Batch size & 16 \\
Gradient clipping & 0.5 \\
Weight decay & 0 \\
Seed & 1 \\
\midrule
\multicolumn{2}{l}{\textbf{Model architectures (\({\sim}1\)M params each); values from stored checkpoint args}} \\
\midrule
FNO & modes 18, width 20, 1.04M params \\
F-FNO &  988K params \\
MFC-UNet & 996K params \\ 
\bottomrule
\end{tabular}
\end{table}

\paragraph{Compute resources.}
All experiments were run on NVIDIA GeForce RTX 4090 GPUs with 24\,GB of VRAM, using one GPU per training run. The originally reported experiments used approximately 140 GPU-hours. Camera-ready additions (ablations, seed study, FNO modes-32 control) add approximately 110 GPU-hours. Including development runs, the project used approximately 460 GPU-hours.

\newpage
\FloatBarrier
\section*{NeurIPS Paper Checklist}

\begin{enumerate}

\item {\bf Claims}
    \item[] Question: Do the main claims made in the abstract and introduction accurately reflect the paper's contributions and scope?
    \item[] Answer: \answerYes{} 
    \item[] Justification: The abstract and Introduction claim three core contributions (smooth spectral attenuation, directional transport terms, and cross-resolution transfer), and these are directly evaluated in Experiments via long-horizon rollouts (Navier--Stokes, WeatherBench2, Burgers Equation).

\item {\bf Limitations}
    \item[] Question: Does the paper discuss the limitations of the work performed by the authors?
    \item[] Answer: \answerYes{} 
    \item[] Justification: A dedicated \textbf{Limitations} paragraph in the Discussion section identifies the key scope boundaries of the work. It states that MFC's FFT-based evaluation restricts the method to periodic boundary conditions on rectangular domains, and that extending to irregular geometries or non-periodic boundaries is left to future work. It also notes that computational cost (runtime, memory, FLOPs) is not characterized, parameter counts are matched across models but no efficiency claim is made. 

\item {\bf Theory assumptions and proofs}
    \item[] Question: For each theoretical result, does the paper provide the full set of assumptions and a complete (and correct) proof?
    \item[] Answer: \answerYes{}{} 
    \item[] Justification: The paper provides a concrete proof for the proposed kernel in the appendix, along with an additional stability lemma with a complete proof.

    \item {\bf Experimental result reproducibility}
    \item[] Question: Does the paper fully disclose all the information needed to reproduce the main experimental results of the paper to the extent that it affects the main claims and/or conclusions of the paper (regardless of whether the code and data are provided or not)?
    \item[] Answer: \answerYes{}{} 
    \item[] Justification:  All datasets used are publicly available (WeatherBench2/ERA5, Navier--Stokes, Burgers). The proposed method is fully explained in the paper, and architecture details are provided in the appendix.

\item {\bf Open access to data and code}
    \item[] Question: Does the paper provide open access to the data and code, with sufficient instructions to faithfully reproduce the main experimental results, as described in supplemental material?
    \item[] Answer: \answerNo{} 
    \item[] Justification: Code will be released publicly upon acceptance.

\item {\bf Experimental setting/details}
    \item[] Question: Does the paper specify all the training and test details (e.g., data splits, hyperparameters, how they were chosen, type of optimizer) necessary to understand the results?
    \item[] Answer: \answerYes{} 
    \item[] Justification: There is a dedicated section in the Appendix with full experimental details (data splits, hyperparameters, optimizer, architecture configurations, and evaluation protocols) for all three benchmarks.

\item {\bf Experiment statistical significance}
    \item[] Question: Does the paper report error bars suitably and correctly defined or other appropriate information about the statistical significance of the experiments?
    \item[] Answer: \answerYes{}{} 
    \item[] Justification: All experiments report results over multiple initial conditions or trajectories. WeatherBench2 long-horizon evaluation uses $n{=}35$ ICs with a divergence criterion ($10\times$ GT max; 11/35 FNO diverge, 0/35 for MFC and AFNO), and climatology statistics are reported over the shared 24/35 non-divergent subset. Cross-resolution results report median RMSE over 35 ICs together with degradation ratios. Navier--Stokes metrics are computed over 10 test trajectories, with qualitative panels selected by a deterministic median-energy criterion. Burgers evaluates 50 trajectories per test set and reports geometric-mean nRMSE over 100 rollout steps with per-step curves for all models.
    

\item {\bf Experiments compute resources}
    \item[] Question: For each experiment, does the paper provide sufficient information on the computer resources (type of compute workers, memory, time of execution) needed to reproduce the experiments?
    \item[] Answer: \answerYes{}{} 
    \item[] Justification: The manuscript reports the hardware and approximate compute used for the experiments. Each run used a single NVIDIA GeForce RTX 4090 GPU with 24\,GB VRAM. The reported experiments required approximately 140 GPU-hours in total, and the full project including development runs required approximately 350 GPU-hours.
    
\item {\bf Code of ethics}
    \item[] Question: Does the research conducted in the paper conform, in every respect, with the NeurIPS Code of Ethics \url{https://neurips.cc/public/EthicsGuidelines}?
    \item[] Answer: \answerYes{}{} 
    \item[] Justification: This is a methodological study on scientific forecasting surrogates using benchmark datasets and standard baselines, with no human-subject data collection, intervention, or deployment claims. We have reviewed the NeurIPS Code of Ethics and are not aware of deviations.

\item {\bf Broader impacts}
    \item[] Question: Does the paper discuss both potential positive societal impacts and negative societal impacts of the work performed?
    \item[] Answer: \answerNo{}{} 
    \item[] Justification: The manuscript currently focuses on technical fidelity/stability and does not include a dedicated broader-impacts section discussing positive impacts (e.g., faster scientific simulation) and potential risks (e.g., misuse or over-trust in operational forecasting contexts).
    
\item {\bf Safeguards}
    \item[] Question: Does the paper describe safeguards that have been put in place for responsible release of data or models that have a high risk for misuse (e.g., pre-trained language models, image generators, or scraped datasets)?
    \item[] Answer: \answerNo{}{} 
    \item[] Justification: The work does not release high-risk generative assets or scraped web datasets; it studies operator-learning surrogates on established scientific benchmark data. Under this scope, additional controlled-release safeguards are not required beyond standard responsible reporting.

\item {\bf Licenses for existing assets}
    \item[] Question: Are the creators or original owners of assets (e.g., code, data, models), used in the paper, properly credited and are the license and terms of use explicitly mentioned and properly respected?
    \item[] Answer: \answerNo{}{} 
    \item[] Justification: Existing assets and data sources used in the paper are credited. However, explicit license names, terms of use, and version information are not yet fully listed.

\item {\bf New assets}
    \item[] Question: Are new assets introduced in the paper well documented and is the documentation provided alongside the assets?
    \item[] Answer: \answerNA{}{} 
    \item[] Justification: The paper's primary contribution is the method and evaluation, not a newly released standalone dataset or benchmark. No new external asset package is currently being released with separate documentation requirements.

\item {\bf Crowdsourcing and research with human subjects}
    \item[] Question: For crowdsourcing experiments and research with human subjects, does the paper include the full text of instructions given to participants and screenshots, if applicable, as well as details about compensation (if any)? 
    \item[] Answer: \answerNA{} 
    \item[] Justification: The work does not include crowdsourcing experiments or research with human subjects.

\item {\bf Institutional review board (IRB) approvals or equivalent for research with human subjects}
    \item[] Question: Does the paper describe potential risks incurred by study participants, whether such risks were disclosed to the subjects, and whether Institutional Review Board (IRB) approvals (or an equivalent approval/review based on the requirements of your country or institution) were obtained?
    \item[] Answer: \answerNA{} 
    \item[] Justification: The work does not involve human-subject participation, so IRB or equivalent review is not applicable.

\item {\bf Declaration of LLM usage}
    \item[] Question: Does the paper describe the usage of LLMs if it is an important, original, or non-standard component of the core methods in this research? Note that if the LLM is used only for writing, editing, or formatting purposes and does \emph{not} impact the core methodology, scientific rigor, or originality of the research, declaration is not required.
    \item[] Answer: \answerNA{} 
    \item[] Justification: LLMs are not an important or non-standard component of the core modeling method or experimental pipeline presented in this work.

\end{enumerate}

\end{document}